\documentclass[a4paper, 11pt]{article}
\pdfoutput=1  
\usepackage{etoolbox,verbatim}

\newtoggle{arxiv}
\toggletrue{arxiv} 

\newtoggle{customthms}
\toggletrue{customthms}

\newtoggle{colt}
\togglefalse{colt} 

\usepackage[utf8]{inputenc} %
\usepackage[T1]{fontenc}    %
\usepackage{url}            %
\usepackage{booktabs}       %
\usepackage{amsfonts}       %
\usepackage{nicefrac}       %
\usepackage{microtype}      %
\usepackage[table]{xcolor}         %
\usepackage{subcaption}
\usepackage{algorithm}
\usepackage{algorithmic}
\usepackage{wrapfig}

\usepackage{xcolor}
\usepackage{tcolorbox}
\tcbuselibrary{skins,breakable}

\tcbset{
  aibox/.style={
      width=\linewidth,
      top=6pt,
      bottom=0pt,
      left=1.5pt,
      right=1.5pt,
      colback=blue!60!gray!5,
      colframe=black,
      colbacktitle=black,
      enhanced,
      center,
      attach boxed title to top left={yshift=-0.1in,xshift=0.15in},
      boxed title style={boxrule=0pt,colframe=white,},
    }
}
\newtcolorbox{AIbox}[2][]{aibox,title=#2,#1}

\newtcolorbox{bluequote}{
  colframe=black,        %
  arc=2pt,               %
  left=8pt,
  right=8pt,
  top=6pt,
  bottom=6pt
}

\definecolor{primalcolor}{HTML}{A60000}
\definecolor{contrarycolor}{HTML}{00A6A6}
\definecolor{darkcontrarycolor}{HTML}{004C4C}
\definecolor{lightblue}{HTML}{2970CC}
\definecolor{lightpurple}{HTML}{673147}
\definecolor{ForestGreen}{HTML}{FF5733}
\definecolor{myred}{HTML}{AA4A44}
\definecolor{hyppurple}{HTML}{800080}
\newcommand{\niceblue}{blue!55!black}
\newcommand{\nicepurple}{violet!70!black}

\newcommand{\nicered}{red!50!black}
\newcommand{\linkcolor}{darkcontrarycolor}
\newcommand{\urlcolor}{darkcontrarycolor}
\newcommand{\citecolor}{darkcontrarycolor}

\newcommand{\thmcolordark}{red!30!black}

\newcommand{\takeawaycolor}{red!40!black}
\newcommand{\takeawaybold}[1]{{\color{\takeawaycolor} {\textbf{#1}}}}

\usepackage[T1]{fontenc}
\usepackage{capt-of}
\usepackage[font=small,labelfont=bf]{caption}
\usepackage{pgf}

\usepackage{natbib}
\usepackage{scalefnt}
\usepackage{amssymb,upgreek,bm}
\usepackage{mathtools}
\usepackage[english]{babel}  
\usepackage{multirow,tcolorbox,tikz-cd,turnstile,xspace,xparse}
\usepackage{relsize,breakcites,appendix}
\usepackage{longtable,tablefootnote,booktabs,float,makecell}
\usepackage[normalem]{ulem}
\usepackage{tabu}
\usepackage[shortlabels]{enumitem} 
\usepackage{footnote}
\usepackage{boxedminipage}
\usepackage{multirow,nicefrac}
\usepackage{verbatim} %

\usepackage[colorlinks=true,linkcolor=\linkcolor,urlcolor=\urlcolor,citecolor=\citecolor,breaklinks]{hyperref}

\usepackage{prettyref}

\usepackage[capitalize,noabbrev]{cleveref} %

\iftoggle{colt}{
    \DeclareRobustCommand{\qed}{
        \usepackage{thmtools}
          \ifmmode \mathqed
          \else
            \leavevmode\unskip\penalty9999 \hbox{}\nobreak\hfill
            \quad\hbox{\qedsymbol}%
          \fi
    }
}
{
    \usepackage{amsthm}
     
}
\usepackage{thm-restate}

\iftoggle{arxiv}
{
\usepackage{mathrsfs}

    \usepackage[
  letterpaper,
  left=1in,
  right=1in,
  top=0.9in,
  bottom=0.9in,
  headheight=14pt,
  headsep=14pt,
  footskip=20pt
]{geometry}

\usepackage{fancyhdr}
\usepackage{lipsum} %

\pagestyle{fancy}
\fancyhf{} 
\fancyfoot[C]{\small \thepage}

\renewcommand{\headrulewidth}{0.4pt}
\renewcommand{\footrulewidth}{0pt}

\usepackage[bitstream-charter]{mathdesign}
 
\usepackage[scaled=0.92]{PTSans}

}
{
    \usepackage{mathrsfs}
}

\DeclareMathAlphabet{\mathbfsf}{\encodingdefault}{\sfdefault}{bx}{n}

\numberwithin{equation}{section}

\iftoggle{arxiv}
{
    }
{
    
}
\iftoggle{arxiv}
{
    
}
{
    
}

\usepackage{thmtools}

\Crefname{equation}{Eq.}{Eqs.}
\Crefname{assumption}{Assumption}{Assumptions}
\Crefname{condition}{Condition}{Conditions}
\Crefname{claim}{Claim}{Claims}
\Crefname{property}{Property}{Properties}
\Crefname{construction}{Construction}{Constructions}

\declaretheoremstyle[
    headformat=\normalfont\textcolor{\thmcolordark}{\bfseries\NAME\,\NUMBER}\NOTE,%
    notefont={\normalfont\textcolor{\thmcolordark}{\bfseries}}, 
    notebraces={}{},
    bodyfont=\normalfont\itshape,
    spaceabove = 6pt,
    spacebelow = 6pt,
    ]{coloredthmversion}

\declaretheoremstyle[
    headformat=\normalfont\textcolor{\thmcolordark}{\bfseries\NAME\,\NUMBER}\NOTE,%
    bodyfont=\normalfont\itshape,
    spaceabove = 6pt,
    spacebelow = 6pt,
    ]{coloredthm}

\declaretheoremstyle[
    headformat=\normalfont\textcolor{\thmcolordark}{\bfseries\NAME\,\NUMBER}\NOTE,%
    bodyfont=\normalfont,
    spaceabove = 6pt,
    spacebelow = 6pt,
    ]{coloreddef}

\iftoggle{customthms}
{
    \theoremstyle{coloredthmversion}
}
{}

\iftoggle{customthms}{
  \theoremstyle{coloredthm}
  \newtheorem{theorem}{Theorem}
  \newtheorem{lemma}{Lemma}[section]
  \newtheorem{corollary}{Corollary}[section]
  \newtheorem{proposition}[lemma]{Proposition}
}
{}

\newtheorem*{thminformal*}{Informal Theorem}

\iftoggle{customthms}{
    \theoremstyle{coloreddef}
    \newtheorem{definition}{Definition}[section]
    
    \newtheorem{remark}{Remark}[section]
    \newtheorem{property}{Property}[section]
}
{
  
}

\newtheorem{observation}[lemma]{Observation}

\newtheorem{assumption}{Assumption}[section]
\newtheorem{condition}{Condition}[section]

\makeatletter
\newcommand{\neutralize}[1]{\expandafter\let\csname c@#1\endcsname\count@}
\makeatother

\iftoggle{customthms}{
    \newtheoremstyle{named}{}{}{\itshape}{}{\bfseries}{}{.5em}{\Cref{#3} {\normalfont (informal)} }{}
    \theoremstyle{named}
    
    \theoremstyle{plain}
}
{}

\newtheorem*{theorem*}{Theorem}
\newtheorem*{lemma*}{Lemma}
\newtheorem*{corollary*}{Corollary}
\newtheorem*{proposition*}{Proposition}
\newtheorem*{claim*}{Claim}
\newtheorem*{fact*}{Fact}
\newtheorem*{observation*}{Observation}
\newtheorem*{definition*}{Definition}
\newtheorem*{remark*}{Remark}
\newtheorem*{example*}{Example}

\newcommand{\bzero}{\ensuremath{\mathbf 0}}
\def\bb{\mathbf{b}}
\def\bB{\mathbf{B}}

\def\by{\mathbf{y}}
\def\bz{\mathbf{z}}

\def\bv{\mathbf{v}}

\def\bw{\mathbf{w}}
\def\bA{\mathbf{A}}
\def\bS{\mathbf{S}}
\def\bI{\mathbf{I}}

\def\bQ{\mathbf{Q}}
\def\bV{\mathbf{V}}

\def\ddefloop#1{\ifx\ddefloop#1\else\ddef{#1}\expandafter\ddefloop\fi}
\def\ddef#1{\expandafter\def\csname bb#1\endcsname{\ensuremath{\mathbb{#1}}}}
\ddefloop ABCDEFGHIJKLMNOPQRSTUVWXYZ\ddefloop

\def\ddefloop#1{\ifx\ddefloop#1\else\ddef{#1}\expandafter\ddefloop\fi}
\def\ddef#1{\expandafter\def\csname frak#1\endcsname{\ensuremath{\mathfrak{#1}}}}
\ddefloop ABCDEFGHIJKLMNOPQRSTUVWXYZ\ddefloop

\def\ddefloop#1{\ifx\ddefloop#1\else\ddef{#1}\expandafter\ddefloop\fi}
\def\ddef#1{\expandafter\def\csname fr#1\endcsname{\ensuremath{\mathfrak{#1}}}}
\ddefloop ABCDEFGHIJKLMNOPQRSTUVWXYZ\ddefloop

\def\ddefloop#1{\ifx\ddefloop#1\else\ddef{#1}\expandafter\ddefloop\fi}
\def\ddef#1{\expandafter\def\csname eul#1\endcsname{\ensuremath{\EuScript{#1}}}}
\ddefloop ABCDEFGHIJKLMNOPQRSTUVWXYZ\ddefloop

\def\ddefloop#1{\ifx\ddefloop#1\else\ddef{#1}\expandafter\ddefloop\fi}
\def\ddef#1{\expandafter\def\csname scr#1\endcsname{\ensuremath{\mathscr{#1}}}}
\ddefloop ABCDEFGHIJKLMNOPQRSTUVWXYZ\ddefloop

\def\ddefloop#1{\ifx\ddefloop#1\else\ddef{#1}\expandafter\ddefloop\fi}
\def\ddef#1{\expandafter\def\csname b#1\endcsname{\ensuremath{\mathbf{#1}}}}
\ddefloop ABCDEFGHIJKLMNOPQRSTUVWXYZ\ddefloop

\def\ddefloop#1{\ifx\ddefloop#1\else\ddef{#1}\expandafter\ddefloop\fi}
\def\ddef#1{\expandafter\def\csname bhat#1\endcsname{\ensuremath{\hat{\mathbf{#1}}}}}
\ddefloop ABCDEFGHIJKLMNOPQRSTUVWXYZ\ddefloop

\def\ddefloop#1{\ifx\ddefloop#1\else\ddef{#1}\expandafter\ddefloop\fi}
\def\ddef#1{\expandafter\def\csname btil#1\endcsname{\ensuremath{\tilde{\mathbf{#1}}}}}
\ddefloop ABCDEFGHIJKLMNOPQRSTUVWXYZ\ddefloop

\def\ddefloop#1{\ifx\ddefloop#1\else\ddef{#1}\expandafter\ddefloop\fi}
\def\ddef#1{\expandafter\def\csname bst#1\endcsname{\ensuremath{\mathbf{#1}^\star}}}
\ddefloop ABCDEFGHIJKLMNOPQRSTUVWXYZ\ddefloop

\def\ddefloop#1{\ifx\ddefloop#1\else\ddef{#1}\expandafter\ddefloop\fi}
\def\ddef#1{\expandafter\def\csname bst#1\endcsname{\ensuremath{\mathbf{#1}^\star}}}
\ddefloop abcdeghijklmnopqrstuvwxyz\ddefloop

\def\ddefloop#1{\ifx\ddefloop#1\else\ddef{#1}\expandafter\ddefloop\fi}
\def\ddef#1{\expandafter\def\csname bhat#1\endcsname{\ensuremath{\hat{\mathbf{#1}}}}}
\ddefloop abcdefghijklmnopqrstuvwxyz\ddefloop

\def\ddefloop#1{\ifx\ddefloop#1\else\ddef{#1}\expandafter\ddefloop\fi}
\def\ddef#1{\expandafter\def\csname b#1\endcsname{\ensuremath{\mathbf{#1}}}}
\ddefloop abcdeghijklnopqrstuvwxyz\ddefloop

\def\ddefloop#1{\ifx\ddefloop#1\else\ddef{#1}\expandafter\ddefloop\fi}
\def\ddef#1{\expandafter\def\csname barb#1\endcsname{\ensuremath{\bar{\mathbf{#1}}}}}
\ddefloop abcdefghijklmnopqrstuvwxyz\ddefloop

\def\ddef#1{\expandafter\def\csname c#1\endcsname{\ensuremath{\mathcal{#1}}}}
\ddefloop ABCDEFGHIJKLMNOPQRSTUVWXYZ\ddefloop
\def\ddef#1{\expandafter\def\csname h#1\endcsname{\ensuremath{\widehat{#1}}}}
\ddefloop ABCDEFGHIJKLMNOPQRSTUVWXYZ\ddefloop
\def\ddef#1{\expandafter\def\csname hc#1\endcsname{\ensuremath{\widehat{\mathcal{#1}}}}}
\ddefloop ABCDEFGHIJKLMNOPQRSTUVWXYZ\ddefloop
\def\ddef#1{\expandafter\def\csname t#1\endcsname{\ensuremath{\widetilde{#1}}}}
\ddefloop ABCDEFGHIJKLMNOPQRSTUVWXYZ\ddefloop
\def\ddef#1{\expandafter\def\csname tc#1\endcsname{\ensuremath{\widetilde{\mathcal{#1}}}}}
\ddefloop ABCDEFGHIJKLMNOPQRSTUVWXYZ\ddefloop

\usepackage{tcolorbox}
\tcbuselibrary{skins,breakable}

\tcbset{
  aibox/.style={
      width=\linewidth,
      top=6pt,
      bottom=0pt,
      left=1.5pt,
      right=1.5pt,
colback=blue!60!gray!5,
colframe=black,
      colbacktitle=black,
      enhanced,
      center,
      attach boxed title to top left={yshift=-0.1in,xshift=0.15in},
      boxed title style={boxrule=0pt,colframe=white,},
    }
}

\DeclarePairedDelimiter{\abs}{\lvert}{\rvert} %

\DeclareMathOperator*{\argmax}{arg\,max}

\let\Pr\relax
\DeclareMathOperator{\Pr}{\mathbb{P}}

\newcommand{\E}{\mathbb{E}}
\newcommand{\Exp}{\mathbb{E}}

\newcommand{\var}{\mathrm{Var}}

\newcommand{\Normal}{\mathrm{N}}

\newcommand{\eye}{\mathbf{I}}

\newcommand{\trace}{\mathrm{tr}}

\newcommand{\ballkr}[1][r]{\cB_{k}(r)}

\newcommand{\fro}{\mathrm{F}}

\newcommand{\rmd}{\mathrm{d}}

\newcommand{\nicehalf}{\nicefrac{1}{2}}

\DeclareMathSymbol{\shortminus}{\mathbin}{AMSa}{"39}

\newcommand{\R}{\mathbb{R}}

\newcommand{\dist}{\mathrm{dist}}

\newcommand{\dx}{\mathrm{d}_{\scriptscriptstyle\mathsf{X}}}

\newcommand{\dy}{\mathrm{d}_{\scriptscriptstyle\mathsf{Y}}}

\newcommand{\bSigma}{\bm{\Sigma}}

\newcommand{\hatExp}{\widehat{\Exp}}

\newcommand{\mem}{\hspace{-1.5pt}}

\newcommand{\Gperp}{\mathcal{P}_{\star}^\perp}
\newcommand{\bveps}{\bm{\varepsilon}}

\newcommand{\rowsp}{\mathrm{rowsp}}

\newcommand{\btheta}{\boldsymbol{\theta}}

\newcommand{\cLhat}{\widehat{\cL}}

\newcommand{\VEC}{\mathrm{vec}}

\newcommand{\RMS}{\mathsf{RMS}}

\newcommand{\KFAC}{\texttt{KFAC}\xspace}
\newcommand{\SGD}{\texttt{SGD}\xspace}

\newcommand{\GD}{\texttt{GD}\xspace}
\newcommand{\RMSprop}{\texttt{RMSprop}\xspace}
\newcommand{\Adam}{\texttt{Adam}\xspace}
\newcommand{\DoPe}{\texttt{DoPr}\xspace}
\newcommand{\DoPr}{\texttt{DoPr}\xspace}
\newcommand{\AP}{\texttt{AP}\xspace}
\newcommand{\GP}{\texttt{GP}\xspace}
\newcommand{\AdamW}{\texttt{AdamW}\xspace}
\newcommand{\Shampoo}{\texttt{Shampoo}\xspace}
\newcommand{\Muon}{\texttt{Muon}\xspace}
\newcommand{\AdaGrad}{\texttt{AdaGrad}\xspace}

\newcommand{\Signum}{\texttt{Signum}\xspace}
\newcommand{\AdaMuon}{\texttt{AdaMuon}\xspace}

\newcommand{\bnorm}[1]{\ensuremath{\left\| #1 \right\|}}
\newcommand{\opnorm}[1]{\bnorm{#1}_{\rm{op}}}
\DeclareMathOperator{\lmin}{\lambda_\textup{min}}
\DeclareMathOperator{\lmax}{\lambda_\textup{max}}

\newcommand{\bmat}[1]{\begin{bmatrix} #1\end{bmatrix} }
\newcommand{\paren}[1]{\left(#1\right)}

\newcommand{\scurly}[1]{\{#1\}}
\newcommand{\brac}[1]{\left[#1\right]}

\newcommand{\ip}[1]{\ensuremath{\left\langle #1 \right\rangle}}
\DeclarePairedDelimiter\norm{\lVert}{\rVert}

\iftoggle{arxiv}
{
  \newcommand{\icmlpar}[1]{\paragraph{#1}}
}
{
  \newcommand{\icmlpar}[1]{\textbf{#1}}
}

\newcommand{\blue}[1]{{\color{\niceblue}#1}}
\newcommand{\red}[1]{{\color{\nicered}#1}}
\newcommand{\purple}[1]{{\color{\nicepurple}#1}}

\newcommand{\xblue}[1][t]{{\color{\niceblue} \bx_{#1}}}
\newcommand{\xred}[1][t]{{\color{\nicered} \widehat{\bx}_{#1}}}
\newcommand{\yblue}[1][t]{{\color{\niceblue} \by_{#1}}}

\newcommand{\zred}[1][t]{{\color{\nicered} \widehat{\bz}_{#1}}}
\newcommand{\sblue}[1][t]{{\color{\niceblue} \bs}_{#1}}
\newcommand{\sred}[1][t]{{\color{\nicered} \hat{\bs}}_{#1}}
\newcommand{\ablue}[1][t]{{\color{\niceblue} \ba}_{#1}}
\newcommand{\ared}[1][t]{{\color{\nicered} \hat{\ba}}_{#1}}

\newcommand{\bxbar}{\overline{\bx}}
\newcommand{\bsbar}{{\color{\niceblue} \bar\bs}}
\newcommand{\bthetabar}{\bar{\btheta}}
\newcommand{\bzbar}[1][\ell]{\overline{\bz}_{#1}}
\newcommand{\bWbar}[1][\ell]{\overline{\bW}_{#1}}
\newcommand{\bZbar}[1][\ell]{\overline{\bZ}_{#1}}

\newcommand{\bGbar}{\overline{\bG}}
\newcommand{\Sigmabar}{\overline{\Sigma}}

\newcommand{\vred}{{\color{\nicered} \widehat{\bv}}}

\newcommand{\Ktheta}{{\color{\nicered} \bK_{\theta}}}
\newcommand{\Ftheta}{{\color{\nicered} \mathbf{F}_{\theta}}}
\newcommand{\Gtheta}{{\color{\nicered} \mathbf{G}_{\theta}}}
\newcommand{\GthetaGD}{{\color{\nicered} \mathbf{G}^{\normalfont\mathrm{next}, \GD}_{\theta}}}
\newcommand{\GthetaAP}{{\color{\nicered} \mathbf{G}^{\normalfont\mathrm{next}, \AP}_{\theta}}}
\newcommand{\Gproj}{\mathcal{P}_{\mathbf{G}}}

\newcommand{\Fstar}{{\color{\niceblue} \mathbf{F}_{\star}}}
\newcommand{\Gstar}{{\color{\niceblue} \mathbf{G}_{\star}}}

\newcommand{\Kdemo}{{\color{\niceblue} \bK_{\star}}}

\newcommand{\Sigmas}{\bSigma_{\bs}}

\newcommand{\rmnext}{\mathrm{next}}

\newcommand{\cLval}{\cL_{\mathrm{val}}}
\newcommand{\cLtrain}{\cL_{\mathrm{train}}}

\newcommand{\cRtest}{\cR_{\mathrm{test}}}

\newcommand{\Lval}{L_{\mathrm{val}}}
\newcommand{\Ltrain}{L_{\mathrm{train}}}

\newcommand{\cLideal}{\cL_{\mathrm{ideal}}}

\newcommand{\pidemo}{{\color{\niceblue}\pi_{\mathrm{demo}}}}
\newcommand{\pitheta}{{\color{\nicered}\pi_{\btheta}}}
\newcommand{\mutheta}{{\color{\nicered}\bm{\mu}_{\theta}}}

\newcommand{\Diverg}{\mathrm{D}}

\iftoggle{arxiv}{
    \newcommand{\nvsp}[1][0]{}
}{
    \newcommand{\nvsp}[1][0.3]{\vspace{-#1cm}}
}

\def\ep{{\varepsilon}}

\newcommand{\normal}{{\sf N}}

\newcommand{\out}{\textit{out}}
\newcommand{\inp}{\textit{in}}

\usepackage{preamble/color-edits}

\addauthor{ms}{magenta}

\newcommand{\ignore}[1]{}

\iftoggle{arxiv}
{
\input{preamble/arxiv_title}
}
{}

\addauthor{tz}{blue}
\addauthor{as}{magenta}
\addauthor{vz}{orange}
\addauthor{yz}{green}

\title{Double Preconditioning (\DoPe): Optimization  for \\ Test-Time Performance, not Validation Loss}
\author{
  \small Thomas T.\ Zhang\footnote{\texttt{\{ttz2, alokshah, vsyzhang, nmatni\}@seas.upenn.edu} \label{foot:penn}}\textsuperscript{,$\dagger$}  ~ Alok Shah\textsuperscript{\ref{foot:penn},$\dagger$} ~ Yifei Zhang\footnote{\texttt{\{yifeizh3,msimchow\}@andrew.cmu.edu}\label{foot:cmu}}\textsuperscript{,$\dagger$} ~ Vincent Zhang\textsuperscript{\ref{foot:penn},$\dagger$}\\
  \small Nikolai Matni\textsuperscript{\ref{foot:penn},$\ddagger$} ~ Max Simchowitz\textsuperscript{\ref{foot:cmu},c,$\ddagger$}\\
  \vspace{-.3em}
 \rule{.38\textwidth}{.7pt}\\
   \footnotesize $^a$University of Pennsylvania~~$^b$Carnegie Mellon University~~$^c$Amazon FAR \\
   \footnotesize $^\dagger$Equal contribution \quad $^\ddagger$Equal advising\\
}
\papercode{https://tinyurl.com/3kfuhmpf}
\paperdate{\today}

\date{\vspace{-.5em}}

\begin{document}

\begin{tcolorbox}[
colback=blue!60!gray!5, colframe=gray!50, 
boxrule=0pt, 
    arc=2mm%
]
\maketitle
\vspace{-0.7cm}
\tcbline
\begin{minipage}[t]{1.0\linewidth}
    \centering
    \includegraphics[width=0.85\columnwidth]{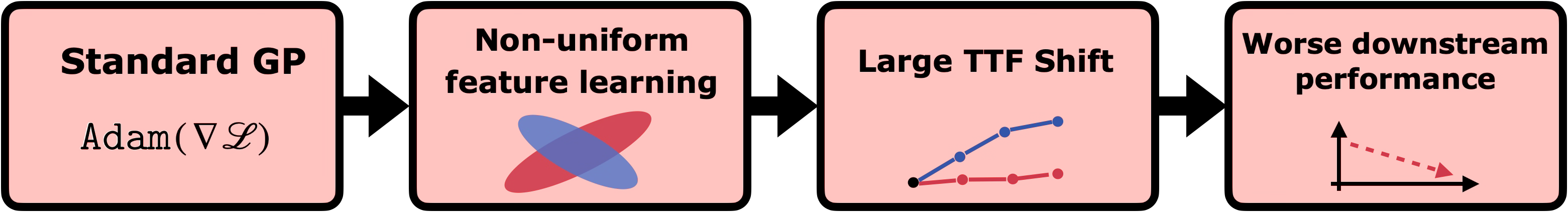}
    
    \vspace{0.1cm}
    \includegraphics[width=0.85\columnwidth]{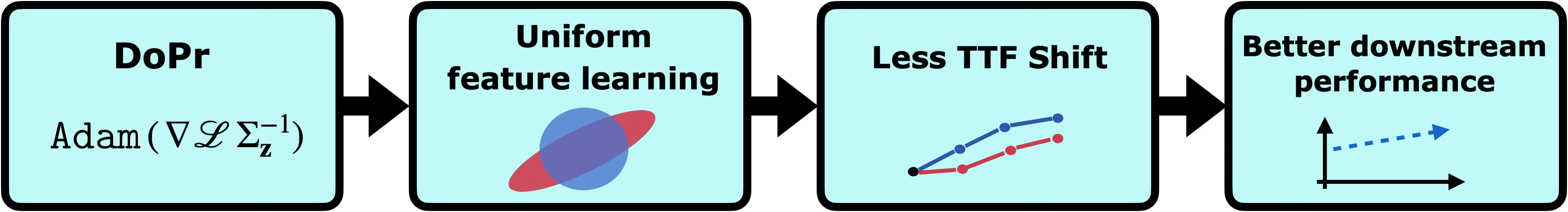}

    \captionof{figure}{Standard optimizers, while effective at accelerating validation loss convergence, may induce poor feature learning. This can exacerbate distribution shift due to \textbf{test-time feedback} (TTF), the growing compounding errors as the model is deployed along its own predictions, ultimately leading to degraded downstream performance. We propose \textbf{Double Preconditioning} (\DoPe) as a plug-in approach, where we apply a particular preconditioner to the layerwise gradient to encourage more ``uniform'' feature learning before passing it to a more standard optimizer. This reduces susceptibility to TTF, thereby improving downstream performance.
    }
    \label{fig:flowchart}
\end{minipage}

  \vskip 0.5cm
  \makeatletter
\vspace{-0.7cm}
  \tcbline \ifdefempty{\metadatalist}{}{\metadatalist\par}
  \makeatother
  
\end{tcolorbox}
\begin{abstract}

Many modern applications of deep learning involve training a neural network via a one-step prediction loss (e.g., $L^2$ regression, cross-entropy), but deploy the network by rolling out along its own predictions. Key examples include autoregressive language modeling, flow-based generative modeling, and robot policy learning. It is well-documented that these settings induce a phenomenon we call \textbf{test-time feedback} (TTF): the mismatch between the training/validation loss and downstream metrics of interest, such as task success rate and generation quality, which grows with task length. While data curation, architecture, and objective design have been proposed to combat train-test shift in TTF settings, this paper proposes \emph{optimization} as a new design axis to mitigate error accumulation. Specifically, we introduce a new optimization paradigm called \textbf{double-preconditioning} (\DoPe) uniquely tailored to the challenges of TTF. \DoPe combines gradient-wise preconditioning, as in \Adam and \Muon, with activation-wise preconditioning (\AP), such as in \KFAC. We show that the addition of \AP yields a \emph{drop-in intervention} for increasing downstream model performance across a range of TTF settings.
Interestingly, these gains in test-time performance \emph{do not consistently accompany improvements in validation loss}, opening new questions about how to properly evaluate models trained with one-step supervised objectives.

\iftoggle{arxiv}{}{\textbf{Anonymized Code:} [\href{https://anonymous.4open.science/r/fm_precond-2601/}{link}]}
\vspace{-0.3cm}

\end{abstract}

\iftoggle{arxiv}{
\clearpage
}{}

\sloppy
\allowdisplaybreaks

\nvsp
\section{Introduction}
\nvsp[0.3]

Recent years have seen a resurgence in the design of novel optimizers for deep learning,  particularly via the use of ``matrix-shaped'' normalization or \textbf{preconditioning}. These optimizers adjust gradient updates by a ``preconditioning'' matrix to reweight the descent direction. The successes highlight improved \emph{convergence time} on training or validation loss, such as
\Shampoo \citep{gupta2018shampoo, anil2020scalable, shi2023distributed} winning the AlgoPerf optimizer competition \citep{mlcommons2024algoperf} as measured by an overall ``holdout error per unit compute'' metric, and \Muon \citep{jordan2024muon} coming to prominence in greatly accelerating low ``time to target validation loss'' on NanoGPT pretraining and seeing adoption in training production-scale language models \citep{team2025kimi}.

These successes implicitly assume that low generalization error on the training objective implies good downstream performance, e.g. pass@$k$ in the context of language reasoning. %
However, it has been widely noted that validation loss may not track natural performance metrics. 
The two diverge under a regime we call \textbf{test-time feedback} (TTF), ubiquitous in
contemporary deep learning, in which one \emph{trains on per-step supervised objectives}, while deployment involves \emph{rolling out multiple steps according to the model's own predictions}. 
Key examples include: (1) \textbf{robot policy learning} where per-step (or per-chunk) imitation of expert actions along expert trajectories, and deployment involves rolling out using the learned policy's actions and (2) autoregressive \textbf{LLM training} in which models are trained via next-token prediction, yet deployed via autoregressive sampling. In these settings, sequential deployment can cause the distribution of inputs (e.g., context tokens in language modeling) to diverge from that of the training distribution, leading to distribution shift driven solely by the model's own prediction errors. This ``TTF shift'' is widely attested in both robotics \citep{ross2011reduction,simchowitz2025pitfalls} and in long-horizon language generation (e.g., \citep{bengio2015scheduled, ranzato2015sequence, song2023consistency}).
\iftoggle{arxiv}{
}
{\begin{figure}[t]
  \vskip 0.2in
  \begin{center}
    \includegraphics[width=0.85\columnwidth]{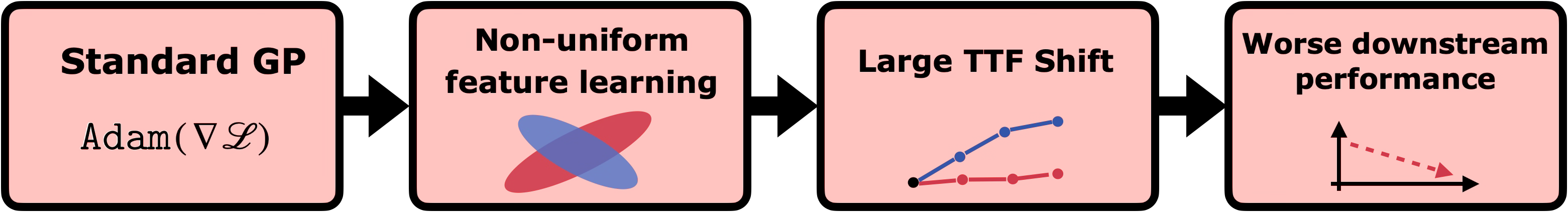}
    
    \vspace{0.1cm}
    \includegraphics[width=0.85\columnwidth]{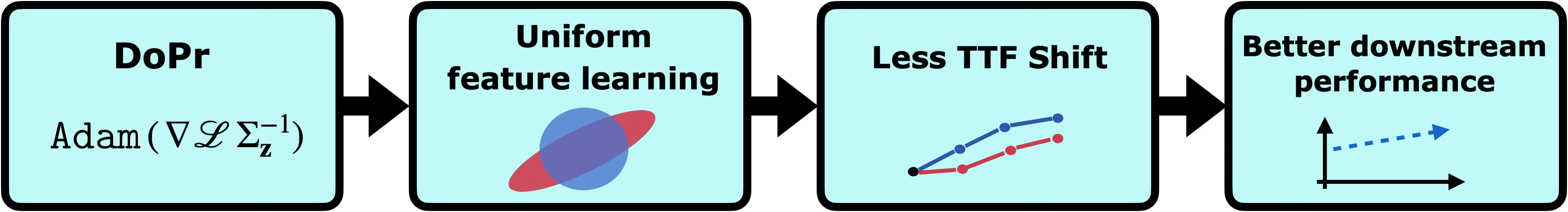}

    \caption{
      Standard optimizers can induce non-uniform feature learning (see \Cref{sec:connection TTF feature learning}), where small errors measured by validation loss $\cLval$ can cause outsized errors during deployment due to \textbf{test-time feedback} (TTF), inducing train-test mismatch leading to worsened downstream performance. We propose \textbf{Double Preconditioning} (\DoPe), see \Cref{sec:derivation}, which targets more ``uniform'' feature learning for reducing TTF.
    }
    \label{fig:flowchart}
  \end{center}
  \vspace{-0.7cm}
\end{figure}}
Motivated by these findings and recent achievements in deep-learning optimizer design, we ask:
\begin{quote}
  Is there a design space for lightweight, generally applicable interventions on deep learning optimizers that target improved \textbf{downstream performance?}
\end{quote}

\takeawaybold{Contributions.}  \textbf{First}, we establish a rigorous connection between \emph{distribution shift} induced by TTF and \emph{imperfect feature learning} under today's popular optimizers that precondition updates solely using gradient statistics (e.g. \citet{kingma2014adam,kellerjordan/cifar10-airbench}). Specifically, we observe that when activations/inputs at each layer are non-isotropic, feature learning suffers, which induces errors that can disproportionately drive distribution shift under TTF. Moreover, these \emph{cannot} be remedied by updates at later layers (\Cref{prop:TTF_feature_learning}).

\textbf{Second}, based on the above insights, we propose the \takeawaybold{Double Preconditioning} (\DoPe) framework, which couples an \textbf{activation-based preconditioner} (\AP) using layer activations to compute a preconditioning matrix (e.g. \cite{martens2018kronecker}) with more standard \textbf{gradient-based preconditioners} (\GP), such as in \Adam \citep{kingma2014adam, loshchilov2017decoupled} or \Muon \citep{jordan2024muon}. \AP encourages more uniformly accurate feature learning, mitigating the TTF distribution shift  described above, whilst benefiting from the training speed and stability afforded by \GP.

Whereas past variants of \AP in the literature were introduced as standalone optimizers, \DoPe 
shows how to incorporate \AP 
as an \emph{drop-in} modification to any  \GP optimizer, such as \Adam and \Muon, providing a stable training recipe and produces models with strong downstream performance. We develop an \emph{invariance principle} under which \DoPe updates can be systematically derived for any \{architecture, \GP\} pairing, e.g., convolution and self-attention layers. Further, optimal hyperparameters can be reliably predicted by popular \GP scaling heuristics \citep{yang2023tensor}.

Finally, we carefully evaluate the \textbf{capabilities} of \DoPe across a range of continuous-control, robotics, and language generation tasks. We find that \DoPe consistently improves downstream performance---measured via natural, task-specific metrics across numerous applications--- by intervening \emph{purely on the optimizer}, without additional modifications to the data, training objective, or architecture.

\takeawaybold{Notation.} We denote vector quantities by \textbf{bold} lower-case, and matrix/tensor quantities by \textbf{bold} upper-case. We use $\odot$ to denote element-wise (Hadamard) product, $\otimes$ for Kronecker product, and $\VEC(\cdot)$ the \emph{column-major} vectorization operator. We reserve $\bW$ to denote weights in their tensor shape, and $\btheta = \VEC(\bW)$ be the vectorized version. We use $\Exp[f(\bx)]$ to denote the expectation of $f(\bx)$. We reserve $\bSigma$ for \emph{uncentered} covariance matrices, e.g.\ $\Sigma_\bx = \Exp[\bx \bx^\top]$. We denote layer-indexing with subscripts and (optimizer) iterates with superscripts, e.g., $\bW_{\ell}^{(k)}$. Lastly, we use the index shorthand $[L] = \scurly{1,\dots,L}$.

\nvsp
\section{Learning under Test-Time Feedback}\label{sec:TTF}
\nvsp

We study the problem of learning under \textbf{test-time feedback} (TTF) --- where our goal is to learn a model that is deployed in feedback with its own generations (as in language modeling) or with the environment (as in robotics). 
We formalize TTF
as a problem of behavior cloning \citep{pomerleau1988alvinn} in a Markov Decision Process (MDP). Consider a finite-horizon  MDP with states $\bs_t \in S$, actions $\ba_t \in A$, transitions $\bs_{t+1} \sim P(\bs_t,\ba_t)$, and total horizon $T$.
 \emph{Models} or \emph{policies} $\pi$ are maps from $S \to \Delta(A)$, and we let $\Exp^{\pi}[\cdot]$ denote the expectation over actions $\ba_t \sim \pi(\bs_t)$ and state transitions $\bs_{t+1} \sim P(\bs_t,\ba_t)$.  In behavior cloning, we learn a policy $ \pi_\theta: S \to \Delta(A)$ via demonstration sequences $(\bs_1,\ba_1,\dots,\bs_T,\ba_T)$.  Given a per-example training loss $\Ltrain(\pi;\bs,\ba)$ measuring the distance between $\pi(\bs)$ and $\ba$, and pairs of states and corresponding actions $(\sblue[],\ablue[])$ (blue denoting demonstrator/data distribution), we train a learned policy $\pitheta$ by minimizing
\begin{align}
\cLtrain(\pitheta) = \textstyle\sum_{(\sblue[],\ablue[]) \in \mathrm{data}} \Ltrain(\pitheta;\sblue[],\ablue[]). \label{eq:Ltrain}
\end{align}
In the infinite data regime, and assuming $(\sblue[1],\ablue[1],\dots)$ are collected by a demonstrator $\pidemo$ deployed in the MDP, minimizing \Cref{eq:Ltrain} approximately minimizes \emph{validation loss}
\begin{align}
\textstyle\cLval(\pitheta) = \Exp^{\pidemo} \Big[\frac{1}{T}\sum_{t = 1}^T \Ltrain(\pi_\theta;\sblue,\ablue)\Big].
\end{align}
When there is no risk of confusion, we will abbreviate $\Ltrain = L$ and $\cLval = \cL$. Subsequently, $\pi_\theta$ is \emph{deployed} in the MDP, and we evaluate its expected reward under the distribution of states $(\sred[1],\ared[1],\dots,\sred[T],\ared[T])$ induced by $\pitheta$ under a  trajectory-wise reward function $R(\cdot)$:
\begin{align}
\textstyle\cRtest(\pitheta) = \Exp^{\pitheta}\Big[  R(\sred[1], \ared[1],\dots,\sred[T],\ared[T])\Big].
\end{align}
\nvsp[0.2]

\icmlpar{Language modeling.} Language modeling  with tokens $x_{t}$ can be cast as an instance of TTF \citep{ouyang2022training, foster2024behavior}, where the states $\bs_t = x_{1:t}$ denote the current context, actions $\ba_t = x_{t+1}$ are the next token, and the dynamics are induced by concatenation: $\bs_{t+1} = x_{1:t+1}$. The language model then produces tokens $x_{t+1} \sim \pitheta(\cdot \mid x_{1:t})$. Next-token prediction uses the  cross entropy loss $\Ltrain(\pitheta ; \sblue = x_{1:t},\ablue =x_{t+1}) = -\log \pitheta(x_{t+1} \mid x_{1:t})$. Typical choices of reward include success rate or pass@$k$ \citep{chen2021evaluating}. 
\icmlpar{Robotic behavior cloning.}
    In robotics, $\bs$ corresponds to the robot and environment state, and $\ba$ to the robot actions. In practice, $\bs$ is replaced by robot observations (e.g., pixels, tactile), and actions $\ba$ may be short sequences or ``action-chunks''. Earlier works parameterize Gaussian policies $\pitheta(\cdot \mid \bs) = \Normal(\mutheta(\bs),\sigma^2 \bI)$, motivating an $L_2$ training loss \iftoggle{arxiv}{$\Ltrain(\pitheta;\sblue[],\ablue[]) = \|\mutheta(\sblue[]) - \ablue[]\|^2$}{} whereas modern works use generative model parameterizations (e.g. flows or diffusion \citep{chi2023diffusion, pan2025much}).
    \iftoggle{arxiv}{
    Typical rewards can include task success (e.g., an object was successfully moved to a desired location) or dense locomotive rewards (e.g., how far the robot has traversed).
    }{
    Rewards can include task success or dense locomotive rewards.
    }
\newcommand{\Sigmaw}{\Sigma_w}

\newcommand{\thetared}[1][t] {{\color{\nicered}\bm{\theta}}^{(k)}}
\newcommand{\Alg}{\texttt{alg}}
\icmlpar{Optimizer Preconditioning for Minimizing $\cLtrain$ in TTF.} 
In our TTF formulation, $\cLtrain$ is a standard supervised-learning loss. We consider an optimization algorithm $\Alg$ that produces iterates $\{\thetared\}_{k \ge 1}$ in order to minimize $\cLtrain$. Examples include mini-batch stochastic gradient descent (\SGD), \Adam(W), \Shampoo \citep{gupta2018shampoo, shi2023distributed}, and the recently popularized Muon optimizer \citep{jordan2024muon}. 
Most popular deep learning optimizers can be viewed as \emph{gradient-based preconditioners} (\GP), where the update direction is shaped purely using gradient information. Many prominent members therein, such as \Adam, \Shampoo, \Muon etc., are specifically structured approximations (e.g., diagonal, layer-wise) of the \AdaGrad gradient-covariance $\sum_{k\geq 1} \nabla\cLtrain(\thetared) \nabla\cLtrain(\thetared)^\top$ \citep{duchi2011adaptive}, and admit specific interpretations as gradient \emph{whiteners} or \emph{normalizers}. We provide full discussion in \Cref{appdx:extended related}.
\textbf{Past work} has extensively studied how the choice of optimization improves performance only on the \emph{validation loss} $\cLval$, both in terms of final iterate $\lim_{k \to \infty} \cLval(\thetared)$, or the convergence rate, i.e., how quickly $\Lval(\thetared)$ approaches its limit. In this work, we consider a different problem statement:
\begin{AIbox}{Problem Statement 1}
    How do we design an optimizer to minimize $\cLtrain$, so as to maximize \emph{downstream performance} of the resulting model $\pitheta$ when evaluated at test-time, via $\cRtest(\pitheta)$? 
\end{AIbox}

\nvsp
\section{Distribution Shift under Test-Time Feedback}
\label{sec:ttf}
\nvsp

Whereas the validation loss $\cLval$ evaluates training loss under the distribution of states induced by $\pidemo$,  test-time reward $\cRtest$ considers the distribution of states under $\pitheta$. Due to sequential deployment in \emph{feedback} with the MDP (e.g. through autoregressive generation in language modeling or with the environment in robotics), these two distributions do not agree (\Cref{fig:TTF rollout error}). We call their consequent difference test-time feedback (TTF) shift.
\begin{figure}[t]
  \vskip 0.2in
  \begin{center}
    \includegraphics[width=0.3\columnwidth]{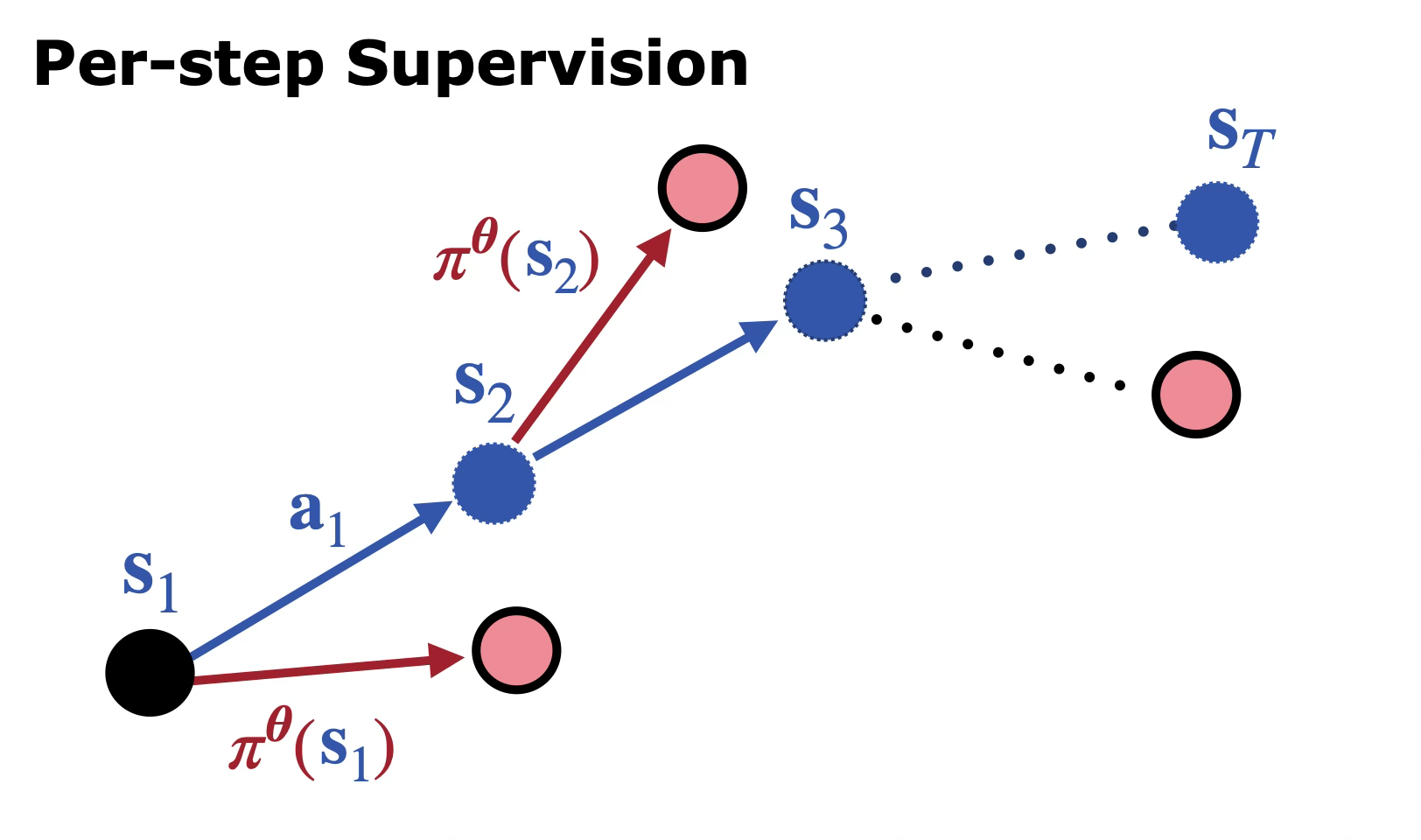}
    \includegraphics[width=0.3\columnwidth]{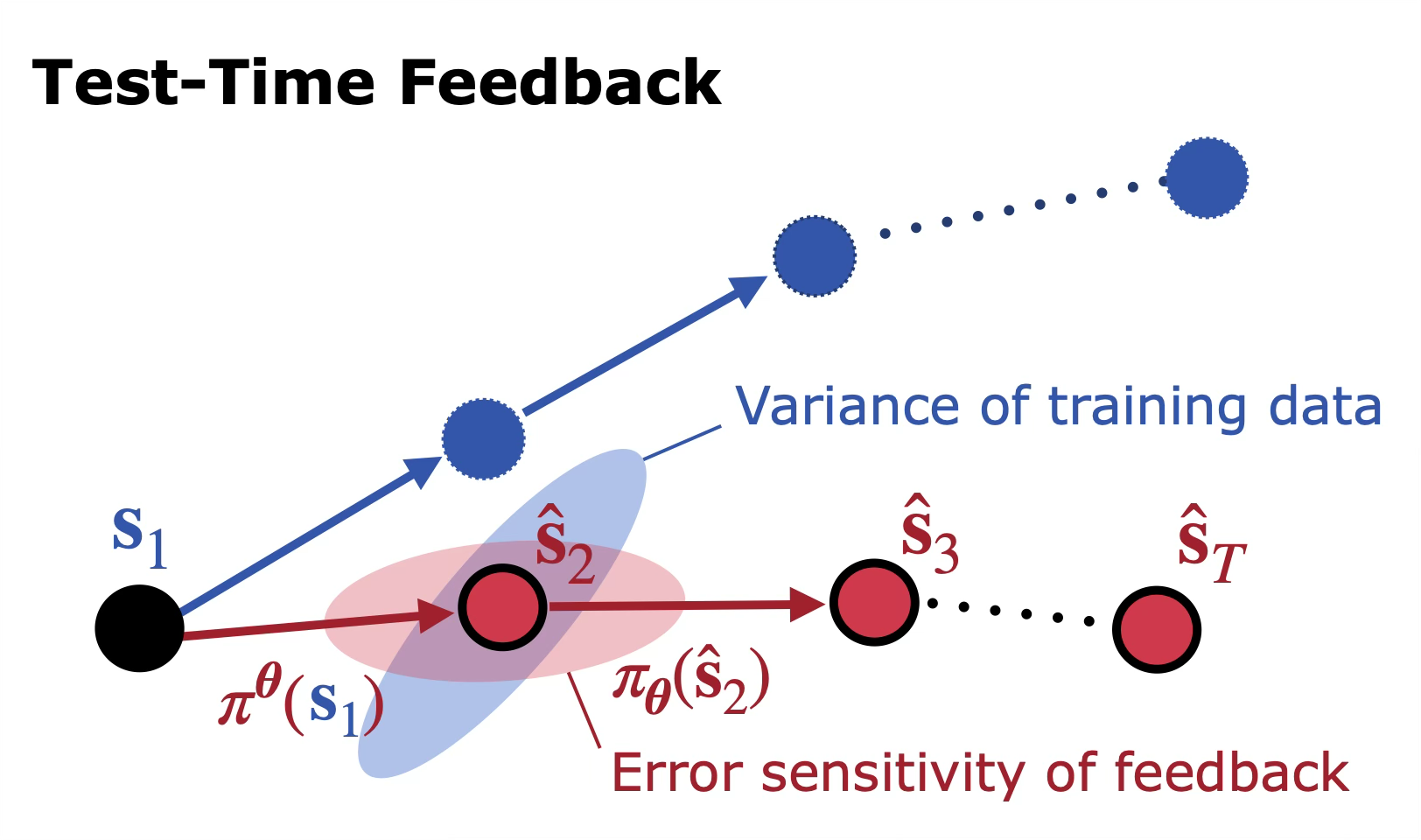}
    \includegraphics[width=0.3\columnwidth]{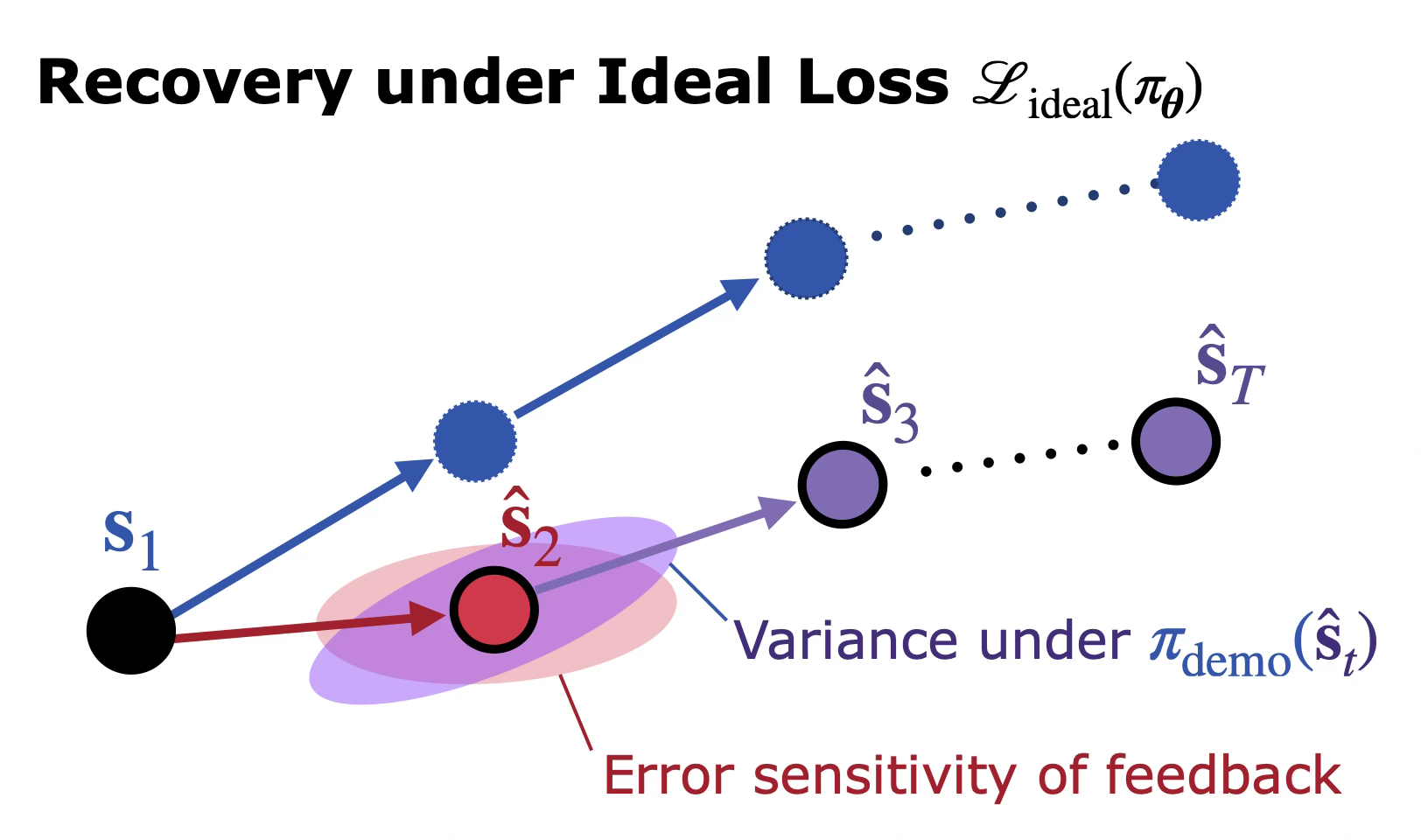}

    \caption{Many settings involve using per-step supervised objectives along training sequences. However, due to rolling out along the model's own predictions, mismatches between directions salient for $\cLval(\pitheta)$ versus $\cRtest(\pitheta)$ cause TTF shift. Hypothetically, instead optimizing for directions salient for $\cLideal(\pitheta)$---often unavailable for offline training---would induce smaller TTF shift.
    }
    \label{fig:TTF rollout error}
  \end{center}
  \vspace{-0.7cm}
\end{figure}
\iftoggle{arxiv}{
    \begin{definition}[TTF Shift]\label{def:test-time feedback} Let $\pidemo$ be the data-collection policy, and $\pitheta$ a learned policy. Given a probability distance/divergence $\Diverg$ (e.g. KL, Wasserstein Distance, etc\dots), we define the \textbf{TTF shift}  via  
    $\Diverg_t( \pitheta) := \Diverg(\Pr_{\sblue}^{\pidemo},  \Pr_{\sred}^{\pitheta})$, where $\Pr_{\bs_t}^{\pi}$ denotes the marginal distribution of $\bs_t \sim \Pr^{\pi}$.
    \end{definition}
}{Let $\pidemo$ be the data-collection policy, and $\pitheta$ a learned policy. Given a probability distance/divergence $\Diverg$ (e.g. KL, Wasserstein Distance etc.), we define the \textbf{TTF shift}  via  
    $\Diverg_t( \pitheta) := \Diverg(\Pr_{\sblue}^{\pidemo},  \Pr_{\sred}^{\pitheta})$, where $\Pr_{\bs_t}^{\pi}$ denotes the marginal distribution of $\bs_t \sim \Pr^{\pi}$.}
When TTF shift occurs, the distribution used to measure $\cLval$, $\Pr^{\pidemo}_{\sblue[t]}$, and that used for $\cRtest$, $\Pr^{\pitheta}_{\sred[t]}$,  differ, 
which has been shown in the literature to degrade performance $\cRtest$. The effects of test-time distribution shifts has been extensively studied in the robotics and sequence-modeling literature \citep{ross2011reduction, bengio2015scheduled,simchowitz2025pitfalls}; see \Cref{appdx:extended related} 
\iftoggle{arxiv}{for a full account}{}.
Importantly, in contrast to general distribution shifts, TTF shift is \textbf{due to model error} propagated through sequential deployment.
To reduce both TTF shift and its effects on downstream reward, past work has suggested that one strive to minimize an \textbf{idealized validation loss} with respect to actions from $\pidemo$, but under the state distribution of the \emph{learned policy} $\pitheta$ \citep{ross2010efficient}:
\begin{align}
    \cLideal(\pitheta) := \Exp^{\pitheta}[\textstyle \frac{1}{T}\sum_{t=1}^T \Ltrain(\pitheta; \sred,\ablue^\star) ], \quad \sred \sim \Pr^{\pitheta}, \ablue^\star \sim \pidemo(\cdot \mid \sred).
\end{align}
In particular, prior work \citep{ross2010efficient, foster2024behavior} show that not only is $\cLideal$ a good surrogate for $\cRtest$, but that minimizing $\cLideal$ would suffice to reduce TTF by implicitly correcting $\pitheta$ towards $\pidemo$ when it veers off distribution (\Cref{fig:TTF rollout error}(c)).
However, minimizing $\cLideal$ in practice is often challenging, due to its circular dependence on $\pitheta$'s own distribution and the necessity of collecting $\pidemo$ actions on $\pitheta$ states. 
\iftoggle{arxiv}{Adaptive or multi-stage data collection, such as the seminal DAgger algorithm \citep{ross2011reduction}, attempt to bridge this gap.}
{\footnote{Adaptive or multi-stage data collection, such as the seminal DAgger \citep{ross2011reduction}, attempt to bridge this gap.}}
However, we are often \emph{only} given data from $\Pr^{\pidemo}$ (e.g., finetuning corpora or teleoperator demonstrations). Hence, we can refine Problem Statement 1:
\begin{AIbox}{Problem Statement 2}
    Can we use \textbf{optimization preconditioning} to encourage minimization of $\cLideal$ by mitigating TTF shift, even if we are only given data from $\pidemo$?
\end{AIbox}

\newcommand{\bGamma}{\bm{\Gamma}}
\nvsp[0.2]
\subsection{TTF Shift and Imperfect Feature Learning}\label{sec:connection TTF feature learning}
\nvsp[0.2]
We consider \textbf{linear dynamical systems} (LDS) as a minimal example of TTF. We consider a LDS with state $\bs \in \R^{n}$, actions $\ba \in \R^m$,  state transitions $\bs_{t+1} = \bA \bs_t + \bB \ba_t + \bw_t$, where $\bw_t \sim \Normal(0,\bm{\Sigma}_{\bw})$ is process noise and $\bA,\bB$ are matrices of appropriate shape. We assume that the demonstrator $\pidemo(\bs) = \Kdemo \bs$ and imitator $\pitheta(\bs) = \Ktheta \bs$ are both state feedback policies described by matrices in $\R^{m \times n}$, and the latter uses an $L_2$ training loss: $\Ltrain(\pitheta;\sblue[],\ablue[]) = \norm{\Ktheta \sblue[] - \ablue[]}^2_2 = \norm{(\Kdemo  - \Ktheta) \sblue[]}^2_2 = \Ltrain(\pitheta;\sblue[])$, where the last equality notes that the $\ablue[]$ argument is redundant. Lastly, we consider a reward $R(\sred[1],\ared[1],\dots,\sred[T],\ared[T]) = - \frac{1}{T}\sum_{t=1}^T \Ltrain(\pitheta;\sred[t]) =- \frac{1}{T}\sum_{t=1}^T \|(\Kdemo  - \Ktheta)\sred[t]\|^2 $. In this setup, we can describe all relevant quantities in closed form:\footnote{See \Cref{appdx:dynamical systems theory} for proofs and further details of the results in this section.}
\begin{restatable}{proposition}{LDS}\label{prop:LDS TTF}
    Define the state-covariance matrix $\bGamma_t(\bK) = \sum_{s=0}^{t-1} (\bA + \bB \bK)^{s} \bm{\Sigma_}{\bw} (\bA + \bB \bK)^{s \top}$, and $\bar \bGamma_t(\bK) = \frac{1}{t}\sum_{s=1}^{t} \bGamma_s(\bK)$. Then,  $ \Pr^{\pidemo}_{\sblue}  = \Normal(0, \bGamma_t(\Kdemo))$, $ \Pr^{\pitheta}_{\sred}  = \Normal(0, \bGamma_t(\Ktheta))$, and
\iftoggle{arxiv}{
\begin{align}
    \cLval(\pitheta) &=\textstyle \left\|(\Kdemo - \Ktheta)\bar\bGamma_T(\Kdemo)^{\nicehalf}\right\|_{\fro}^2, \\ 
    - \cRtest(\pitheta) = \cLideal(\pitheta)  &=\textstyle \left\| (\Kdemo - \Ktheta)\bar\bGamma_T(\Ktheta)^{\nicehalf}\right\|_{\fro}^2.
\end{align}
}{
    $\cLval(\pitheta) =\textstyle \left\|(\Kdemo - \Ktheta)\bar\bGamma_T(\Kdemo)^{\nicehalf}\right\|_{\fro}^2$, $ \cLideal(\pitheta) = - \cRtest(\pitheta) = \textstyle \left\| (\Kdemo - \Ktheta)\bar\bGamma_T(\Ktheta)^{\nicehalf}\right\|_{\fro}^2.$
}
If $\Diverg$ denotes the Wasserstein-$2$ distance, we have $\Diverg_t(\pitheta) = \|\Gamma_t(\Kdemo)^{\nicehalf} - \Gamma_t(\Ktheta)^{\nicehalf}\|^2_{\fro}$.
    \end{restatable}

\icmlpar{Validation loss need not track downstream performance due to TTF.}    
\Cref{prop:LDS TTF} characterizes TTF shift in terms of the state-covariance matrices induced by  $\Kdemo$, via $\bGamma_t(\Kdemo)$ (resp. $\Ktheta $ via $\bGamma_t(\Ktheta)$). These same covariances weight the errors $\Kdemo - \Ktheta$ in both $\cLval$ and $\cLideal$. Thus, TTF and downstream performance depends not only on error magnitudes, but \emph{in what directions} the error between policies  $\Kdemo - \Ktheta$  lies with respect to covariances. In \Cref{appdx:construct LDS Lideal Lval example}, we construct 
two linear policies $\pi_i(\bs) = \bK_{\pi_i}\bs,i \in\{1,2\}$ such that $\Lval(\pi_1) \ll \Lval(\pi_2)$, but $\cLideal(\pi_1)\gg \cLideal(\pi_2)$:
$\bK_{\pi_1}- \Kdemo$ has error along small eigenspaces of $\bGamma_T(\Kdemo)$, but through feedback between $\bK_{\pi_1}$ and the dynamics, these errors lead to large eigenvalues of $\bGamma_T(\bK_{\pi_1})$, driving up $\cLideal$.

\icmlpar{Understanding TTF from Feature Learning.}
To see how feature learning influences TTF, we \emph{over}-parameterize the linear learner policy $\pitheta$ as a product of linear layers $\pitheta(\bs) = \Ktheta \bs = \Ftheta \Gtheta \bs$, where $\Ftheta \in \R^{m \times d}$, $\Gtheta \in \R^{d \times n}$; similarly write $\Kdemo = \Fstar \Gstar$, and assume $ \mathrm{rank}(\Kdemo) = d$ \iftoggle{arxiv}{for simplicity}{}. A necessary condition for globally optimal imitation of $\pidemo(\cdot)$ is accurately recovering the (row-)span of $\Gstar$, e.g., measured by a subspace distance:
\begin{align}\label{eq: subspace dist}
    \dist(\mathbf{G}, \Gstar) := \opnorm{ \Gproj \Gperp}, \;\; 
\end{align}
where $\Gproj,\Gperp$are orthogonal projections onto $\rowsp(\Gproj),\rowsp(\Gstar)^\perp$.
Vanilla \GD, by minimizing $\cLval$,  always ensures that $(\Gtheta-\Gstar)\bs$ is small along directions $\bs$ which align with large eigenvalues of $\bar \bGamma_T(\Kdemo)$. However, large subspace distance $\dist(\Gtheta,\Gstar)$ means that there exist states $\bs$ for which $\Gstar \bs$ is non-zero, but 
$\Gtheta \bs$ vanishes. Effectively, $\Gtheta$ ``zeros out'' certain directions of state-space, and therefore $\Ftheta \Gtheta$ will as well. Errors in these directions can compound, driving up TTF, and in the worst case, no adjustment of later layers can mitigate the errors made in $\Gtheta$:
\begin{restatable}[Poor Feature Learning causes TTF Shift, Informal]{proposition}{FeatureTTF}
    \label{prop:TTF_feature_learning} There exists   $(\bF_1, \bG_1)$, $(\bF_2, \bG_2)$, $(\Fstar,\Gstar)$ and linear dynamics $\bA,\bB,\bSigma_{\bw}$ such that $\cLval(\bF_1\bG_1)\ll \cLval(\bF_2\bG_2)$, but $\dist(\bG_1,\Gstar) > \dist(\bG_2, \Gstar)$, and for any possible ``last-layer'' $\Ftheta$, we have $\cLideal(\bF_2 \bG_2) \ll \cLideal(\Ftheta\bG_1)$.
\end{restatable}
\Cref{prop:TTF_feature_learning} demonstrates that a model that appears better measured by in-distribution validation loss can hide significantly poorer feature learning, and can thus suffer vastly worse TTF, even assuming the downstream layers can be retrained. A natural question would be whether one can encourage better feature learning, as a separate consideration from validation loss performance.

\begin{figure}[t]
  \vskip 0.2in
  \begin{center}
    \includegraphics[width=0.9\columnwidth]{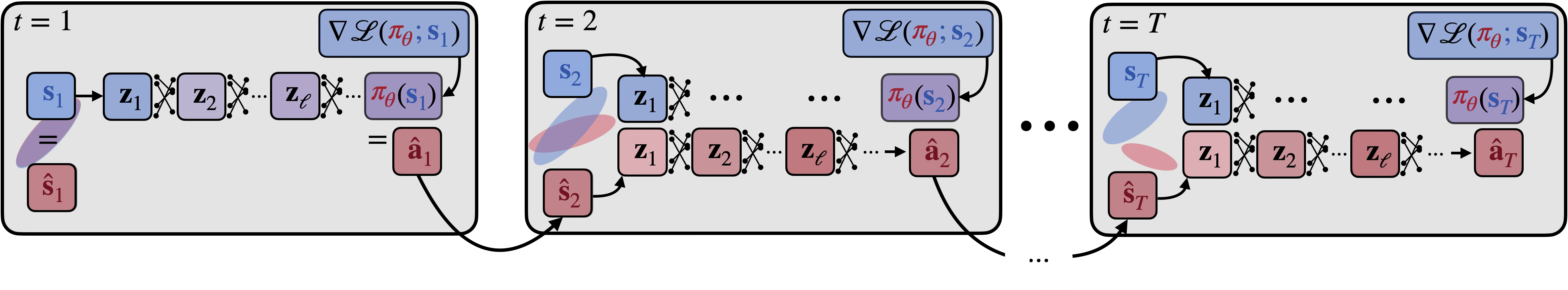}

    \caption{
    A depiction of how test-time feedback exacerbates distribution shift: errors in the network's predictions affect the ensuing states, which changes their distribution away from the one seen in training. Changes in the state distribution (\red{\textbf{red}}) also affect the quality of the learned features (depicted as \blue{\textbf{blue}}, to \purple{\textbf{purple}}, to \red{\textbf{red}}) at intermediate layers (e.g., \Cref{prop:TTF_feature_learning}). Errors propagate layerwise, further exacerbating TTF. Notably, the learning signal (e.g., the loss gradient $\nabla \cLval$) only supervises the predictions on the training distribution (\blue{\textbf{blue}}).
    }
    \label{fig:layerwise TTF}
  \end{center}
  \vspace{-0.7cm}
\end{figure}

\begin{AIbox}{Key Takeaways}
    \begin{enumerate}[
    noitemsep,
    topsep=-5pt,
    parsep=0pt,
    partopsep=0pt,
    left=2pt]
        \item In TTF settings, distribution shift is inevitable: future predictions are rolled out along a model's own (imperfect) predictions.

        \item Therefore, under TTF, what is crucial is not necessarily the training objective, but rather the error directions exacerbated under feedback dynamics, see \Cref{prop:LDS TTF}.

        \item TTF further exposes a mismatch between in-distribution loss convergence and ``feature learning.'' Poor convergence on the latter can further exacerbate TTF \emph{layer-wise}.
    
    \end{enumerate}
\end{AIbox}

\nvsp
\section{Double Preconditioning (\DoPe)}\label{sec:derivation}
\nvsp

Toward addressing the pathologies of TTF, we propose the \textbf{double preconditioning} (\DoPe) framework, which be summarized as applying layer-wise an \emph{activation-covariance} preconditioner (\AP) onto the gradient, then passing the \AP-gradient into a \emph{gradient preconditioner} (\GP) of choice, such as in \Adam or \Muon. For concreteness,  consider  an $L$-layer feedforward network:
\iftoggle{arxiv}{\begin{align}\label{eq:MLP}
    f_{\btheta}(\bx) = \bW_L \phi (\bW_{L-1} \cdots \phi(\bW_1 \bx) \cdots ), 
\end{align}}{$f_{\btheta}(\bx) = \bW_L \phi (\bW_{L-1} \cdots \phi(\bW_1 \bx) \cdots )$,}
where 
$\btheta$ is the concatenation of $\btheta_\ell = \VEC(\bW_\ell)$, $\ell \in [L]$. Denote the intermediate activations $\bz_\ell \triangleq \phi(\bW_{\ell-1} \bz_{\ell-1} \cdots)$, $\bz_{1} \triangleq \bx$ and pre-activations $\bh_{\ell} = \bW_{\ell} \bz_{\ell}$.\iftoggle{arxiv}{ We omit the layer index $\ell$ when it is clear from context.}{\footnote{We omit the layer index $\ell$ when it is clear from context.}} The core double preconditioning update can be summarized as:
\begin{align*}
        &\bM = \hat\nabla_{\bW} \cL(f_{\btheta}) \; \hat\bSigma_{\bz}^{-1}, \; \hat\bSigma_{\bz}\approx \Exp[\bz \bz^\top ] \tag{\AP}\\
        &\bD = \texttt{GP}(\bG), \;\;\bW^{\mathrm{next}} \gets \bW - \eta \bD,
    \tag{{\DoPe}} 
\end{align*}
where $\hat \nabla_{\bW}$ is the minibatch gradient, and $\hat\bSigma_{\bz}$ the batch empirical uncentered covariance of $\bz$.
We display the core algorithm in \Cref{alg:DoPe simple}, and discuss derivations for general architectures in \Cref{sec:AP general architectures}. Full details and practical features are described in \Cref{appdx:derivations}.

\iftoggle{arxiv}{
}{
\begin{wrapfigure}{r}{0.45\textwidth}
  \vspace{-\baselineskip}
  \hrule
  \vspace{0.5em}
  {
  \captionsetup{font=normalsize}
  \captionof{algorithm}{Double Preconditioning (Feedforward Layer, layer index $\ell$ suppressed)}
  \label{alg:DoPe simple}
  }
  \vspace{-0.2em}
  \hrule
  \vspace{0.5em}
  \begin{algorithmic}
    \STATE {\bfseries Input:} learning rate $\eta$, weight decay $\lambda$, damping $\gamma$
    \FOR{$k=1$ {\bfseries to} $K$}
    \STATE Sample batch $B$, $\abs{B} = n$.
    \STATE $\bG \leftarrow \nabla_{\bW} \cLhat(\btheta^{(k-1)})$
    \STATE $\hat\bSigma_{\bz} \leftarrow \frac{1}{n}\sum_{i=1}^n \bz^{(k-1)}_i \mathopen{}\bz^{(k-1)}_i\mathclose{}^\top$
    \STATE $\bM \leftarrow \bG \cdot \big(\hat\bSigma_{\bz} + \gamma \,\trace(\hat\bSigma_{\bz}) \bI \big)^{-1}$ \hfill (\texttt{AP})
    \STATE $\bD \leftarrow \texttt{GP}(\bM,\;^{\ast\ast}\texttt{kwargs})$ \hfill (\DoPe)
    \STATE $\bW^{(k)} \leftarrow (1-\eta\lambda)\bW^{(k-1)} - \eta \bD$
    \ENDFOR
  \end{algorithmic}
  \vspace{0.5em}
  \hrule
  \vspace{-2em}
\end{wrapfigure}
}
Whereas much of the prior literature has focused on the role of \emph{gradient} pre-conditioned optimization for stabilizing training and improving validation loss, we will show that \DoPe's use of \emph{activation} preconditioning mitigates TTF, and therefore \textbf{improve downstream model performance independent of validation loss performance}. %

\iftoggle{arxiv}{
\begin{wrapfigure}{r}{0.45\textwidth}
  \vspace{-\baselineskip}
  \hrule
  \vspace{0.5em}
  {
  \captionsetup{font=normalsize}
  \captionof{algorithm}{Double Preconditioning (Feedforward Layer, layer index $\ell$ suppressed)}
  \label{alg:DoPe simple}
  }
  \vspace{-0.2em}
  \hrule
  \vspace{0.5em}
  \begin{algorithmic}
    \STATE {\bfseries Input:} learning rate $\eta$, weight decay $\lambda$, damping $\gamma$
    \FOR{$k=1$ {\bfseries to} $K$}
    \STATE Sample batch $B$, $\abs{B} = n$.
    \STATE $\bG \leftarrow \nabla_{\bW} \cLhat(\btheta^{(k-1)})$
    \STATE $\hat\bSigma_{\bz} \leftarrow \frac{1}{n}\sum_{i=1}^n \bz^{(k-1)}_i \mathopen{}\bz^{(k-1)}_i\mathclose{}^\top$
    \STATE $\bM \leftarrow \bG \cdot \big(\hat\bSigma_{\bz} + \gamma \,\trace(\hat\bSigma_{\bz}) \bI \big)^{-1}$ \hfill (\texttt{AP})
    \STATE $\bD \leftarrow \texttt{GP}(\bM,\;^{\ast\ast}\texttt{kwargs})$ \hfill (\DoPe)
    \STATE $\bW^{(k)} \leftarrow (1-\eta\lambda)\bW^{(k-1)} - \eta \bD$
    \ENDFOR
  \end{algorithmic}
  \vspace{0.5em}
  \hrule
\end{wrapfigure}
}{}
\icmlpar{Gradient- v.s.\ Activation- preconditioning.}
We remark that activation preconditioning (\AP) and many gradient preconditioners (\GP's) are motivated as different approximations to different curvature preconditioners; see \Cref{appdx:extended related} for an extensive account.
Prior work has modeled \GP's (e.g., \Adam, \Muon) as performing (layer-wise) steepest descent with respect to a salient norm $\norm{\cdot}$ \citep{bernstein2024old, pethick2025lmo}: $\bW^{\mathrm{next}} = \bW - \eta \argmax_{\norm{\bG} \leq 1} \ip{\bG, \nabla_{\bW} \cL}$.\footnote{Practical features like momentum and weight decay can be incorporated with slight modifications \citep{pethick2025lmo}.} 
Versions of \AP have appeared in prior literature as a \emph{standalone} optimizer (see e.g., \citet{amid2022locoprop, benzing2022gradient, zhang2023meta}); we provide a full account in \Cref{appdx:extended related}. 
Rather than establishing the convergence/acceleratory properties of \AP, we demonstrate how it is directly relevant to TTF. However, just as vanilla gradient descent can suffer from numerical instability in deep networks, the \AP update requires stabilization to converge meaningfully on modern taskloads. This motivates the combination of \AP with \GP, yielding \DoPe. In the steepest descent \GP framing  described above, we can express \DoPe by keeping the induced norm constraint $\|\bG\| \le 1$ while replacing the Euclidean inner product with the local ``activation-covariance metric'' induced by the activation covariance matrix:
\begin{align*}
    \bW^{\mathrm{next}} &= \bW - \eta \argmax_{\norm{\bG} \leq 1} \ip{\bG, \nabla_{\bW} \cL}_{\hat\bSigma^{-1}_{\bz}} = \bW - \eta \argmax_{\norm{\bG} \leq 1} \langle \bG, \nabla_{\bW} \cL \cdot \hat\bSigma^{-1}_{\bz}\rangle.
\end{align*}
For example, setting $\norm{\cdot} = \norm{\cdot}_\infty$ as the entry-wise $\ell^\infty$-norm yields a Sign-Descent \AP update ($\sim$\Adam), and $\norm{\cdot}_{\RMS-\RMS}$ recovers a Spectral-Descent AP update ($\sim$\Muon).

\begin{figure}
    \centering
        \includegraphics[width=0.38\linewidth]{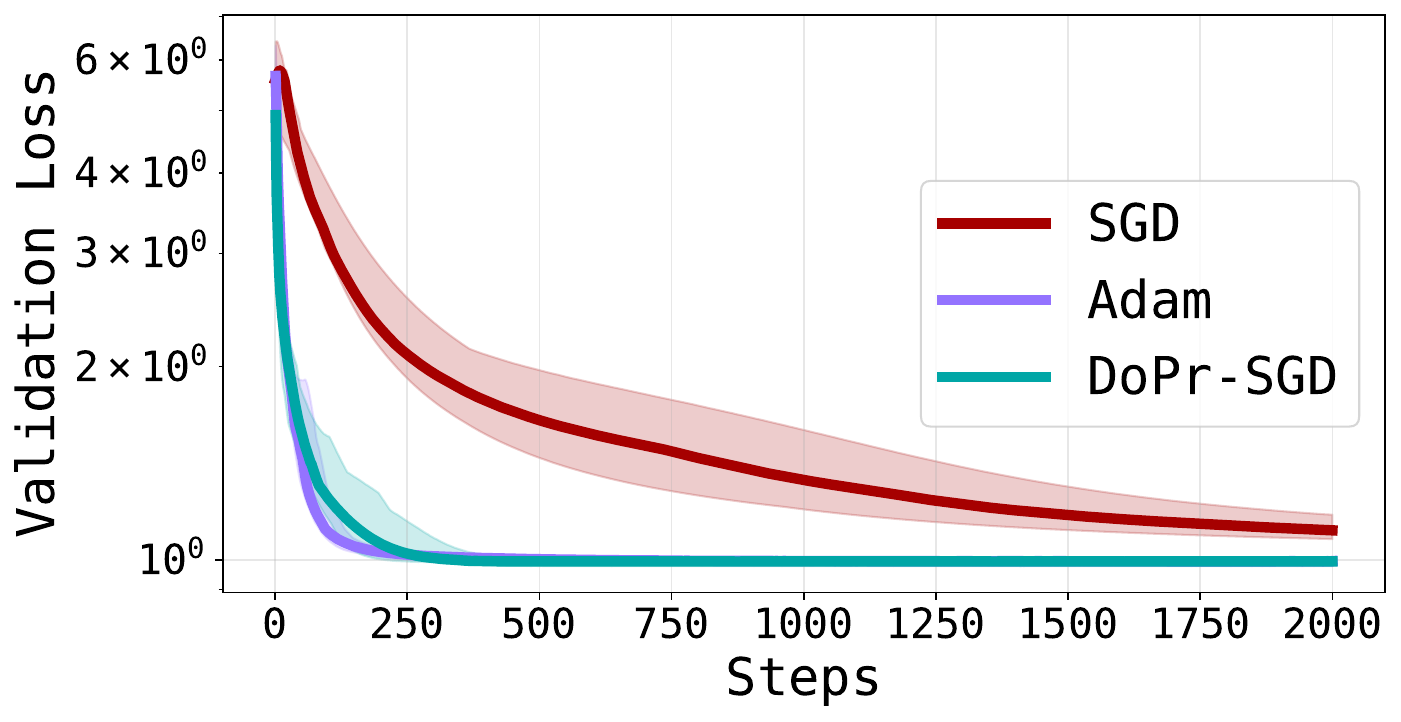}
        \includegraphics[width=0.38\linewidth]{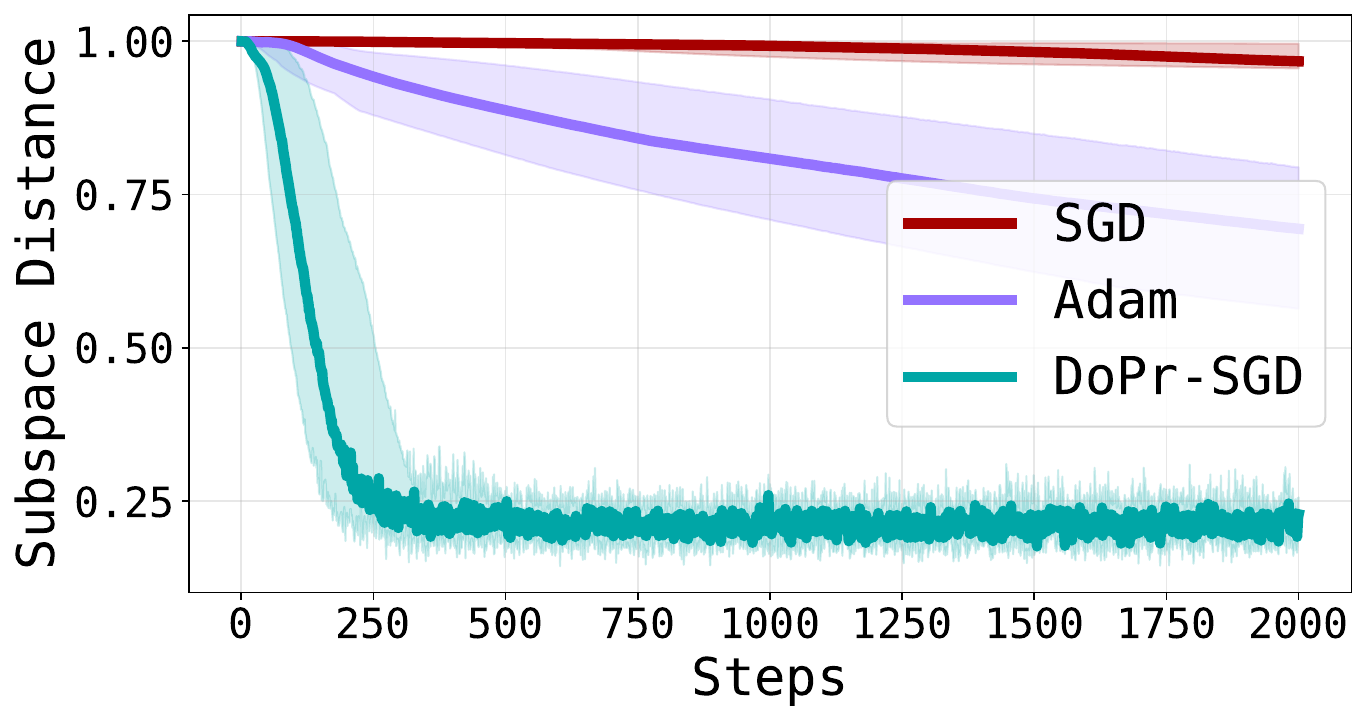} 
        \caption{Mismatch between validation loss and feature learning. In-distribution validation loss converges and is even accelerated by \GP. However, unless \AP is applied, the feature subspace distance \eqref{eq: subspace dist} converges poorly, which can exacerbate TTF (\Cref{prop:TTF_feature_learning}). Details in \Cref{appdx:feature learn exp}.}
        \label{fig:feature_learn}
    \vspace{-0.5cm}
\end{figure}

\nvsp[0.2]
\subsection{Theoretical Motivation: \AP Mitigates TTF via Isotropic Feature Learning}\label{sec:TTF feature learning}
\nvsp[0.2]

\Cref{prop:TTF_feature_learning} therefore reveals that \textbf{non-uniform feature learning} at a given layer (i.e., larger $\dist(\bG_1,\Gstar)$) can \textbf{induce greater TTF shift and reduce downstream performance}, regardless of what is learned at subsequent layers (the choice of $\bF_1$).  We now argue that \AP optimizers exhibit \textbf{more uniform feature learning} by correcting a bias in \GD that occurs when inputs are non-isotropic, i.e. $\Exp^{\pidemo}[\sblue \sblue^\top]$ is ill conditioned \citep{collins2021exploiting, zhang2023meta}. Taken together, these findings imply that \textbf{\AP optimizers have the potential to mitigate TTF shift} and its consequences. %

\newcommand{\cLtrainfirst}{\cL_{\mathrm{train},\bF}}
\begin{proposition}[Informal, adapted from \citet{zhang2023meta}]\label{prop:SGD and AP feature learning}  Let $\eta > 0$ denote the step size, $\GthetaGD$ be $\Gtheta$ after  a step of \emph{full-batch} gradient step on $\Ltrain(\Ftheta\Gtheta)$ with respect to $\Gtheta$, and let $\GthetaAP$ denote the same with the {\normalfont\AP{}} update. Then, setting $\bSigma_{\bs} = \bar\bGamma_T(\Kdemo)$,  we have:
\begin{align}
\begin{split}\label{eq:GD vs AP update}
    \GthetaGD \Gperp  &= (\Gtheta - \eta \Ftheta^\top \Ftheta \Gtheta \bSigma_{\bs}) \Gperp + \eta \Ftheta^\top \Fstar \Gstar \bSigma_{\bs} \Gperp\\
\GthetaAP\Gperp &=  (\bI - \eta \Ftheta^\top \Ftheta) \Gtheta \Gperp.
\end{split}
\end{align}
Consequently, omitting some technical details, when $\mathrm{cond}(\bar \bGamma_T(\Kdemo)) \gg 1$, 
\begin{align*}
    \dist(\GthetaGD, \Gstar) \approx \dist(\Gtheta, \Gstar), \quad \dist(\GthetaAP, \Gstar) \leq \paren{1 - \eta \lmin(\Ftheta^\top \Ftheta)} \dist(\Gtheta, \Gstar).
\end{align*}
\end{proposition}
The proof of \Cref{prop:SGD and AP feature learning} can be found in \Cref{appdx:feature learn theory}. The first statement exposes the \emph{bias} of \GD under non-isotropic covariance, and the non-convergence of \GD and contraction of \AP in feature space. Notably, under \AP the resulting update \eqref{eq:GD vs AP update} behaves \emph{as if the features were isotropic} $\bSigma_{\bs} = \bI$, which is uniquely where \GD and \AP experience the same feature learning, noting $\Gstar \bSigma_{\bs} \Gperp = \Gstar \Gperp = \bzero$.
Notably, we show in \Cref{fig:feature_learn} that \textbf{pure \GP (e.g. \Adam, \Muon) do not yield the same improvements}, even if the $\cLval$ convergence is seemingly accelerated. Conceptually, the pathology illustrated in \Cref{prop:SGD and AP feature learning} can be inductively applied for multi-layer networks, where under TTF, poor feature learning at early layers will under TTF destroy signal for all downstream layers; see e.g., \citet{davis2025spectral, shumaylov2026muon} for detailed investigations of optimizer dynamics in similar settings.

\nvsp[0.2]
\subsection{Lifting AP to General Architectures: An Invariance Principle}\label{sec:AP general architectures}
\nvsp[0.2]

In \Cref{prop:SGD and AP feature learning}, we see that optimizers suffer from impaired feature learning, thus exacerbating TTF, when the covariances of layerwise activations are ill-conditioned, i.e. highly anisotropic, and observed that \AP mitigates this by feature-learning \emph{as if the covariances were isotropic}. Here, we leverage this observation to yield a general \textbf{invariance principle}, providing a reliable schematic for generalizing \DoPe to general network layers, such as convolutions and self-attention.
We first consider the following thought experiment: consider arbitrary layerwise full-rank affine transformations:
\begin{align}\label{eq:affine transform}
    \bzbar \triangleq \bA_\ell \bz_{\ell}, \bWbar \triangleq \bW_{\ell}\bA_{\ell}^{-1},\; \ell \in [L],
\end{align}
and denote the resulting network output $f_{\bthetabar}(\cdot)$, by construction, the pre-activations remain the same: $\bWbar[] \bzbar[] = \bW \bz$, and thus the network outputs are identical $f_{\bthetabar}(\bxbar) = f_{\btheta}(\bx)$.
\iftoggle{arxiv}{
\begin{remark}[Coordinate-dependent optimization paths]\label{rem:coordinate dependence}
    By definition of \eqref{eq:affine transform}, we have $f_{\btheta}(\bx) = f_{\bthetabar}(\bxbar)$ for all $\bx,\bxbar = \bA_1 \bx$. Assume the loss gradient $\nabla_{\btheta} \cL$ is non-zero. Applying the backpropagation formula reveals that after a step of (full-batch) gradient descent:
    \allowdisplaybreaks
    \iftoggle{arxiv}{
    \begin{align}
        \begin{split}\label{eq:GD not invariant}
            \bW_{+} &= \bW - \eta\; \Exp\mem\brac{\frac{\partial L}{\partial\bh} \bz^\top} \\
            \bWbar[+] &= \bWbar[] - \eta\; \Exp\mem\brac{\frac{\partial L}{\partial\overline{\bh}} \bzbar[]^\top} \\
            &= \bWbar[] - \eta\; \Exp\mem\brac{\frac{\partial L}{\partial\bh} \bz^\top}\bA^\top, \;\;\bh \equiv \overline{\bh} \\
            &= \bW_{+} \bA^{-1}  + \eta \frac{\partial\cL}{\partial\bW} (\bA^{-1} - \bA^\top),
        \end{split}
    \end{align}
    }
    {
    \begin{align}
        \begin{split}\label{eq:GD not invariant}
            \bW^{\rmnext} &= \bW - \eta\; \Exp\mem\brac{\frac{\partial L}{\partial\bh} \bz^\top}\mem \\
            \bWbar[]^{\rmnext} &= \bWbar[] - \eta\; \Exp\mem\brac{\frac{\partial L}{\partial\overline{\bh}} \bzbar[]^\top} = \bW^{\rmnext} \bA^{-1}  + \eta \frac{\partial\cL}{\partial\bW} (\bA^{-1} - \bA^\top),
        \end{split}
    \end{align}
    }
    the updated layer outputs are non-identical $\bW_{\ell}^{\rmnext} \bz_\ell \neq \bWbar[\ell]^{\rmnext} \bzbar$ barring the special case of orthogonal $\bA_\ell$ for all $\ell \in [L]$, and thus the updated models no longer have identical outputs:
    $f_{\btheta^{\rmnext}}\mem(\bx) \not\equiv f_{\bthetabar^{\rmnext}}\mem(\bxbar)$.
\end{remark}
}
{Assume the loss gradient $\nabla_{\btheta} \cL$ is non-zero. Applying the backpropagation formula reveals that after a step of (full-batch) gradient descent: $\bW^{\rmnext} = \bW - \eta\; \Exp\mem\brac{\frac{\partial L}{\partial\bh} \bz^\top}\mem$, whereas $\bWbar[]^{\rmnext} = \bWbar[] - \eta\; \Exp\mem\brac{\frac{\partial L}{\partial\overline{\bh}} \bzbar[]^\top} = \bW^{\rmnext} \bA^{-1}  + \eta \frac{\partial\cL}{\partial\bW} (\bA^{-1} - \bA^\top)$, such that
the updated layer outputs are non-identical $\bW_{\ell}^{\rmnext} \bz_\ell \neq \bWbar[\ell]^{\rmnext} \bzbar$ barring the special case of orthogonal $\bA_\ell$ for all $\ell \in [L]$, and thus the updated models no longer have identical outputs:
    $f_{\btheta^{\rmnext}}\mem(\bx) \not\equiv f_{\bthetabar^{\rmnext}}\mem(\bxbar)$.}
\iftoggle{arxiv}{In other words, two NNs of the same architecture initialized to have identical output, can take different optimization paths depending on the ``coordinate system'' of the internal activations, independently of the final output of the network. Consequently, the activation distributions propagated from the training distribution, e.g., $\Sigma_{\bs} = \Exp^{\pidemo}[\sblue\sblue^\top]$ versus $\Sigma_{\bsbar} = \bA_1 \Sigma_{\bs} \bA_1$, \emph{impart a coordinate-dependent bias on the optimizer trajectory}, which \Cref{prop:LDS TTF} and \Cref{prop:TTF_feature_learning} demonstrate can exacerbate TTF.}
{Consequently, the activation distributions propagated from the training distribution, e.g., $\Sigma_{\bs} = \Exp^{\pidemo}[\sblue\sblue^\top]$ versus $\Sigma_{\bsbar} = \bA_1 \Sigma_{\bs} \bA_1$, \emph{impart a coordinate-dependent bias on the optimizer trajectory}, which \Cref{prop:LDS TTF} and \Cref{prop:TTF_feature_learning} demonstrate can exacerbate TTF.}
Deriving \AP for generic layer architectures boils down to the following principle: under an affine transformation on the input to the layer (and appropriate inverse transform to the applied weights) that preserve the layer output, \textbf{what is the update rule that renders the layer outputs of the updated network invariant to the transform}? In the feedforward case, this precisely yields the \AP update as described in (\DoPr).
\begin{restatable}{proposition}{PrecondInvariance}\label{prop:precond invariance}
    Consider a feedforward network, and consider the weights and activations $\scurly{(\bW_\ell, \bz_\ell)}_{\ell\in[L]}$ under non-degenerate layer-wise affine transforms defined in \eqref{eq:affine transform}. Then, given the following layer-wise update rule:
    \iftoggle{arxiv}
    {
    \begin{align*}
        \bW^{\rmnext} = \bW - \eta \nabla_{\bW} \cL(f_{\btheta}), \; \bSigma_{\bz}^{-1},\bSigma_{\bz}= \Exp[\bz \bz^\top], \tag{\AP} 
    \end{align*}
    }
    {
    $\bW^{\rmnext} = \bW - \eta \nabla_{\bW} \cL(f_{\btheta})\bSigma_{\bz}^{-1}$, $\bSigma_{\bz}= \Exp[\bz \bz^\top]$ {\normalfont(\AP)},
    }
    the updated weights satisfy $\bW^{\rmnext}_\ell \bz_\ell = \bWbar^{\rmnext} \bzbar$ for all $\ell \in [L]$, $\bz_1=\bx,\bzbar[1] = \bA_1 \bx$.\nvsp[0.2]
\end{restatable}
In other words, an \AP update for a given layer architecture is one that induces invariance to affine transforms. Thus, \AP updates for convolutional or self-attention layers can be similarly derived; see \Cref{appdx:activation invariance}. This invariance principle can be viewed as a customization of \citet[Section 10]{martens2015optimizing}, itself a specialization of the natural-gradient method's parameterization invariance; see \Cref{appdx:extended related}. We emphasize that the activation-invariance of \AP is an \emph{exact} property of the optimizer trajectory (see \Cref{fig:AP invariance property}) and \Cref{prop:precond invariance} \emph{does not} imply accelerated loss convergence.

\iftoggle{arxiv}{
\begin{figure}
  \begin{center}        \includegraphics[width=0.85\linewidth]{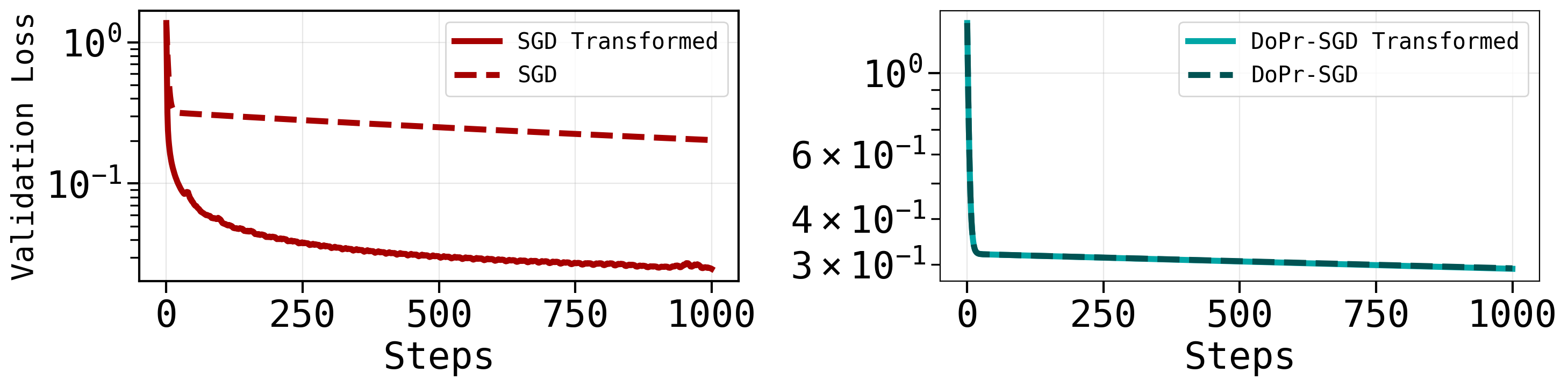}
    \caption{
      When an affine transform is applied to the input distribution, with the initial weights transformed accordingly \eqref{eq:affine transform}, the \SGD trajectories (\textbf{left}) diverge, while the \DoPe-\SGD trajectories (\textbf{right}) match exactly, demonstrating the invariance induced by \DoPe under affine transforms (\Cref{prop:precond invariance}). See \Cref{appdx:activation invariance} for experiment details.
    }
    \label{fig:AP invariance property}
  \end{center}
  \vspace{-0.7cm}
\end{figure}
}{}

\iftoggle{arxiv}{
\begin{remark}\label{remark:isotropy}
    Isotropic activations have been understood to possess desirable properties for enabling feature learning, and thus are often (approximately) \emph{enforced} by normalization layers such as \texttt{BatchNorm} \citep{ioffe2015batch}, \texttt{LayerNorm} \citep{ba2016layer} etc. However, misplacement of normalization layers may have unintended consequences on the conditioning and expressivity of the network. The \AP gradient provides a complementary approach: the effective model change under (\AP) evolves \emph{as if the activations were isotropic}, where $\nabla_{\bW} \cL(f_{\btheta}) \Sigma_{\bz}^{-1} = \nabla_{\bW} \cL(f_{\btheta})$.
\end{remark}
}{}

\nvsp[0.3]
\subsection{Gradient-Preconditioning and Hyperparameter Scaling for \DoPe}\label{sec:hyperparam scaling}
\nvsp[0.3]

\begin{figure}[t]
    \centering
    \includegraphics[width=0.48\linewidth]{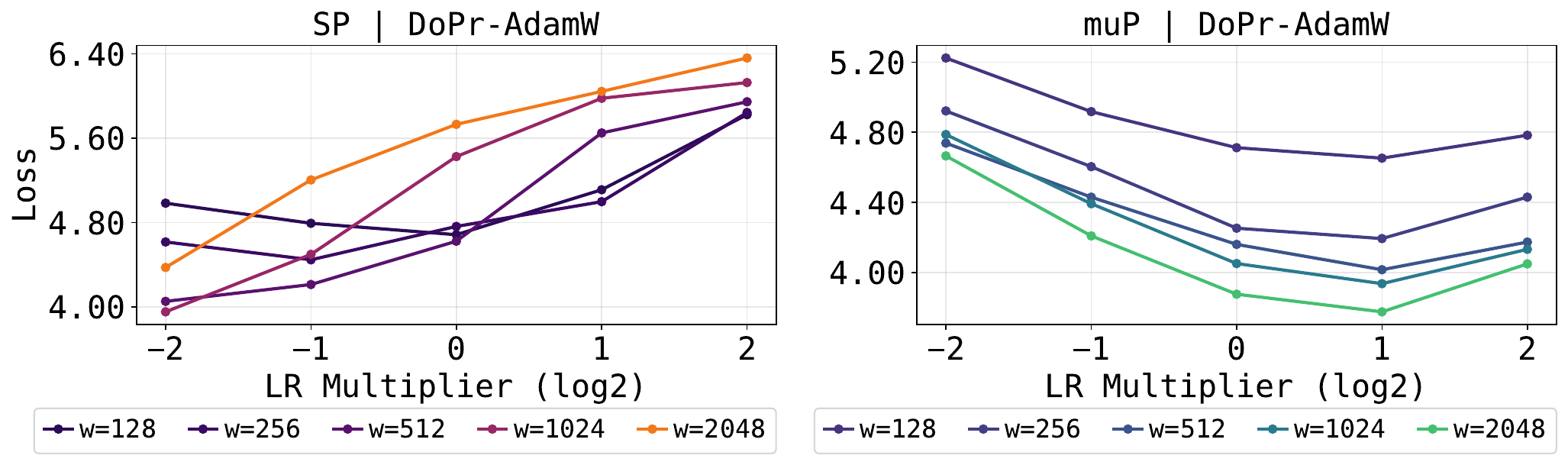}
    \hspace{0.2cm}
    \includegraphics[width=0.48\linewidth]{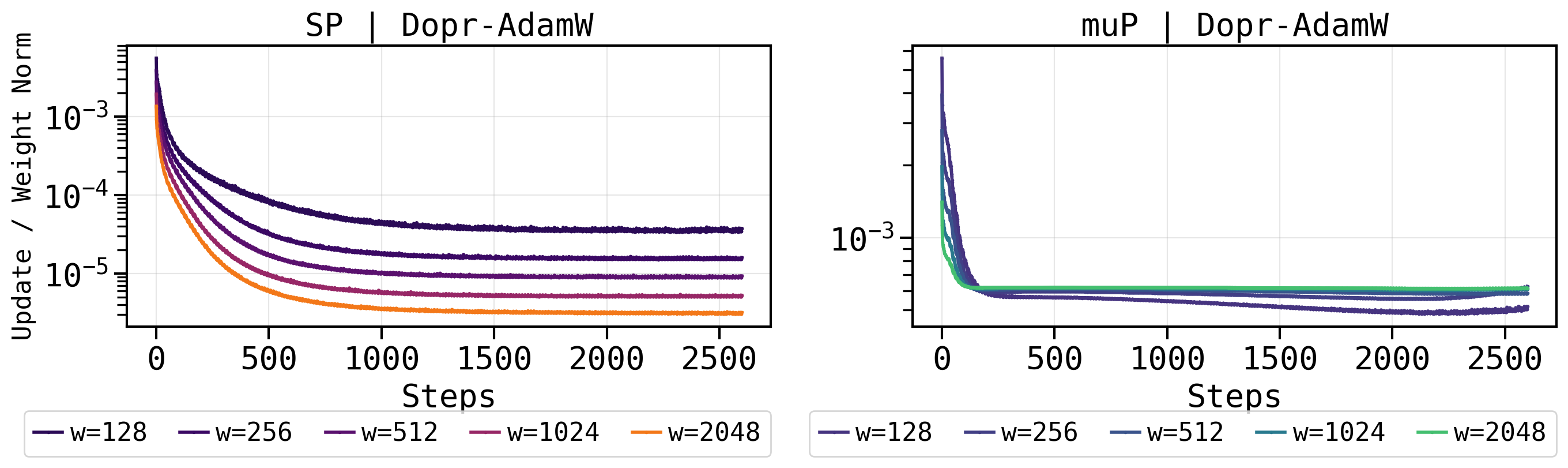}
    \caption{
\textbf{$\mu$P scaling behavior.}
\textbf{Left:} \DoPe-\AdamW's scaling trends under standard (SP) and \AdamW's $\mu$P parameterizations on a GPT2 model. We find base \AdamW's $\mu$P-scaling also enables hyperparameter transfer for \DoPe-\AdamW.
\textbf{Right:} update-to-weight norm ratio scaling trend, under standard constant weight decay (SP) and \AdamW's weight decay $\mu$P scaling. See \Cref{appdx:mup} for full details.
}
    \label{fig:mup_nanogpt_adamw}
    \vspace{-0.7cm}
\end{figure}

Determining favorable hyperparameters for large-scale training runs on a novel optimizer can prove challenging. This is a key motivation behind the ``maximum-update parameterization'' ($\mu$P)  \citep{yang2022tensor, dey2025don}, where the broad goal is to make hyperparameter choices invariant across network scales (e.g., width and depth); see \Cref{appdx:mup} for full discussion. 
Once made scale-invariant, hyperparameters can be tuned at small-scale and zero-shot transferred to the large-scale run. Recent work has demonstrated that in practice $\mu$P is largely determined by coarse statistics such as the magnitude of the update direction \citep{yang2023spectral, hong2025provable}. Thus, in addition to the practical stabilizing and accelerating properties of \GP{s}, a key benefit of using \GP{s} in \DoPe is as follows: whereas deriving scaling rules for a new optimizer can often be laborious, the \emph{normalizing} property of \GP trivializes this process for \DoPe.

\begin{observation}
    Normalization by definition returns outputs of the same norm regardless of input. Thus, substituting the raw gradient with the \AP gradient does not affect the magnitude of the update direction. Consequently, \DoPe hyperparameter scaling rules can be ported directly from established rules based on the \GP of choice, e.g. \Adam, \Muon, \Shampoo etc. \citep{everett2024scaling, qiu2025hyperparameter}.
\end{observation}
\nvsp
We exhibit the immediate transfer of hyperparameter scaling rules in \Cref{fig:mup_nanogpt_adamw}, where we show learning rate and weight decay scaling transfer.

\begin{AIbox}{Key Takeaways}
    \begin{enumerate}[
    noitemsep,
    topsep=-5pt,
    parsep=0pt,
    partopsep=0pt,
    left=2pt]
        \item As shown in \Cref{sec:ttf}, what is crucial is not necessarily the training objective, but rather the error directions most exacerbated under TTF. Thus, the design of \DoPe aims to \textbf{decouple uniform feature learning and accelerated training}/validation loss convergence.
        \item Activation-preconditioning (\AP) debiases the raw gradient direction from the statistics of the activations, which encourages \textbf{uniform feature learning} in all error directions, including those that are underweighted for the training objective but sensitive under TTF.
        \item Gradient-preconditioning (\GP), such as in \Adam and \Muon, stabilizes the \AP gradient, accelerating training, while preserving the \AP geometry. Furthermore, \GP-ed updates allow \textbf{seamless transfer of existing hyperparameter scaling} rules.
    \end{enumerate}
\end{AIbox}

\iftoggle{arxiv}{}{\nvsp}

\section{Capabilities}\label{sec:capabilities}
\iftoggle{arxiv}{}{\nvsp}

In \Cref{sec:derivation}, we derive \DoPr to address the poor feature learning of standard optimizers in TTF settings, with the promise of inducing better downstream behavior. To demonstrate this, we run experiments on diverse tasks and metrics across continuous control, robot policy learning, and LLM training.
As noted in \Cref{sec:AP general architectures}, \DoPe \textbf{need not improve the train/val loss convergence} compared to the base \GP, which we will revisit across our experiments.
We also provide in \cref{app:best_practices} general operating guidelines of our optimizer used in the ensuing experiments.

\nvsp[0.3]
\subsection{Setting the Stage: Drop-in for Continuous Control}\label{sec:mujoco exps}
\nvsp[0.2]

We first evaluate \DoPe for imitation learning in Gymnasium \citep{towers2025gymnasiumstandardinterfacereinforcement}, which precisely aligns with our formalism of behavior cloning in MDPs. Here, we focus on the \texttt{Humanoid-v5} task and refer to \Cref{app:state-based-IL} for full details and further experiments.
We consider four \GP primitives: \Adam, \Muon, \Signum, and \AdaMuon, that roughly cover instantaneous versus momentumized entrywise- and matrix-based normalization. We use a residual MLP architecture \citep{he2015deepresiduallearningimage}, and conduct a full sweep over the optimizer hyperparameters and report the results over independent evaluation trajectories and seeded training runs. We additionally use an EMA-ed copy of the policy parameters \citep{block2024butterfly} for evaluation, which has been shown to be crucial for stabilizing BC policy performance in continuous control. We observe from \Cref{fig:humanoid_il} that: 1.\ different base \GP{s} all have similar terminal rewards, 2. \DoPe always improves terminal reward compared to the baseline counterparts, 3. the train/validation losses of \DoPe variants are not uniformly better than the base \GP. 
\begin{figure}[t]
    \centering
    \includegraphics[width=0.9\linewidth]{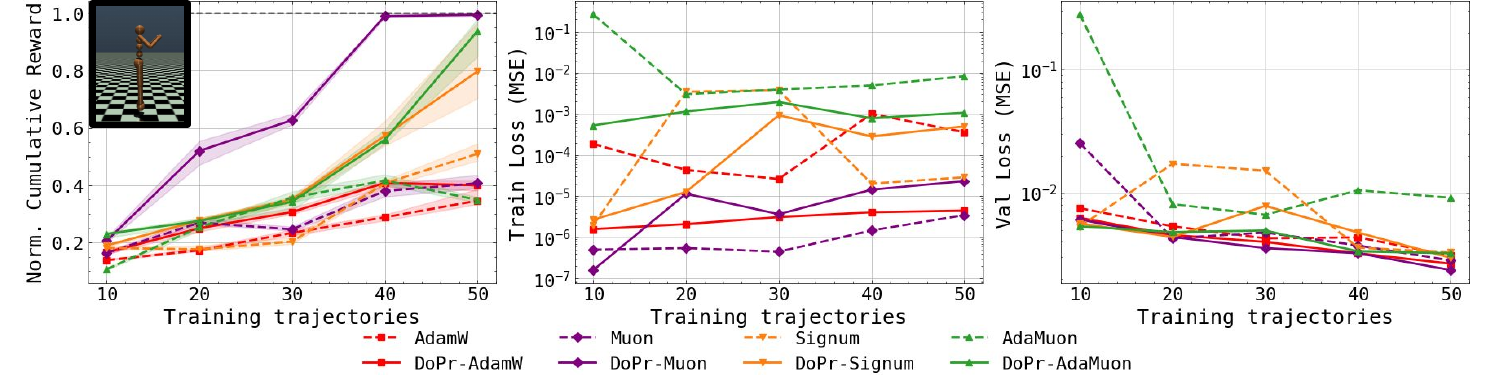}
    \caption{\textbf{Humanoid-v5} \DoPe performance across \AdamW, \Muon, \Signum, and \AdaMuon. \DoPe variants attain higher terminal reward which does not consistently correlate with train or validation loss improvements.}
    \label{fig:humanoid_il}
    \vspace{-0.5cm}
\end{figure}

\nvsp[0.2]
\subsection{Image-Based Robot Policy Learning}
\nvsp[0.2]

\begin{figure}[b]
    \centering
    \includegraphics[width=\iftoggle{arxiv}{0.9\linewidth}{0.82\linewidth}]{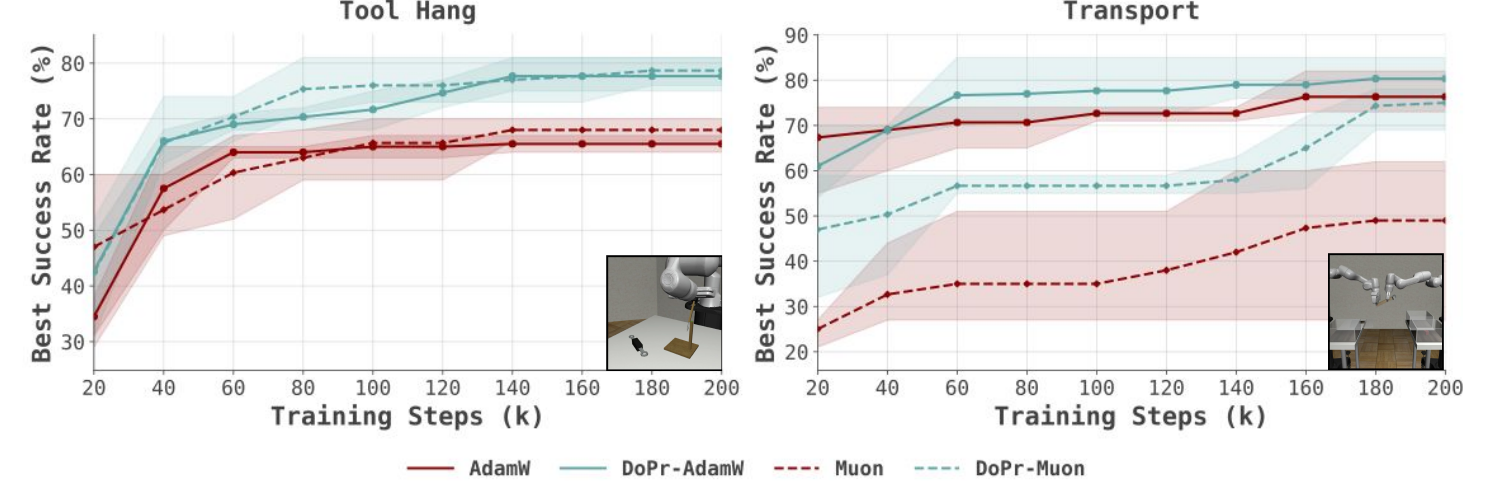}
    \caption{\texttt{Tool Hang (PH)} and \texttt{Transport (PH)} Best Success Rate for \texttt{AdamW}, \texttt{Muon}, and \DoPe variants. Each curve shows the min/median/max over 3 random seeds. \DoPe-variants outperform their baselines.
  }
    \label{fig:robomimic_merged_new.pdf}

    \vspace{-0.6cm}
\end{figure}
We evaluate \DoPe on pixel-based imitation learning with generative policies on \texttt{Robomimic} tasks. We focus on challenging tasks \texttt{Tool-Hang} Proficient-Human \texttt{(PH)} and \texttt{Transport} \texttt{(PH)}, which are not solvable to $\sim100\%$ under modern BC recipes \citep{pan2025much}, and test complementary capabilities: dexterous precision in \texttt{Tool-Hang} and long-horizon coordination in \texttt{Transport}.
Following \citet{pan2025much}, we train a flow-based policy \citep{lipman2023flow} with a U-Net backbone \citep{chi2023diffusion}. As in \Cref{sec:mujoco exps}, we apply model EMA to stabilize policy performance across all settings.
Tuned hyperparameters are listed in \cref{tab:toolhang_hparams,tab:transport_hparams}, and full experiment details are in \cref{app:image-based-il}.
We report the best checkpoint success rate across training in \Cref{fig:robomimic_merged_new.pdf}. 
We observe that \DoPe invariably improves the task success rate compared to the base \GP{s} $\AdamW$ and $\Muon$, while yielding worse or equal training loss---see \Cref{fig:robomimic_image_based_train_loss}. Notably, we see neither base \GP is better than the other on both tasks simultaneously.

\vspace{-0.2cm}
\subsection{Language Models}
\label{sec:sft}
\nvsp[0.2]

We finetune an LLM for mathematical reasoning, where training optimizes token likelihood but task performance depends on sequence accuracy. \Cref{app:sft_details} contains full experimental details.

\vspace{-0.2cm}
\icmlpar{3B Supervised Fine-tuning.}
To characterize the behavior of \DoPe, we first run a smaller-scale SFT experiment with the \texttt{Llama-3.2-3B} base model \citep{llama3modelcard, grattafiori2024llama}, using \texttt{LoRA} \citep{hu2022lora} on a 100K-sample subset of OpenMathInstruct-2 \citep{toshniwal2024openmathinstruct} for one epoch. This experiment probes learning-rate sensitivity, sample efficiency, and the relationship between downstream accuracy and token-level training loss. We report our results in \Cref{fig:sft}. Across most learning rates, \DoPe improves peak \texttt{GSM8K} performance. Notably, these improvements are \emph{not explained by lower token-level training loss}: the final token-level training loss of \DoPe is comparable to, or higher than, that of the base \GP. We report similar trends for \Muon (\Cref{fig:sft-3b-muon}) in \Cref{appdx:sft-3b-addres}.
\begin{figure}[t]
\nvsp
  \centering
  \setlength{\tabcolsep}{1pt}
  \renewcommand{\arraystretch}{0.05}

  \begin{tabular}{@{}cc@{}}
  \iftoggle{arxiv}{
    \includegraphics[width=0.45\linewidth]{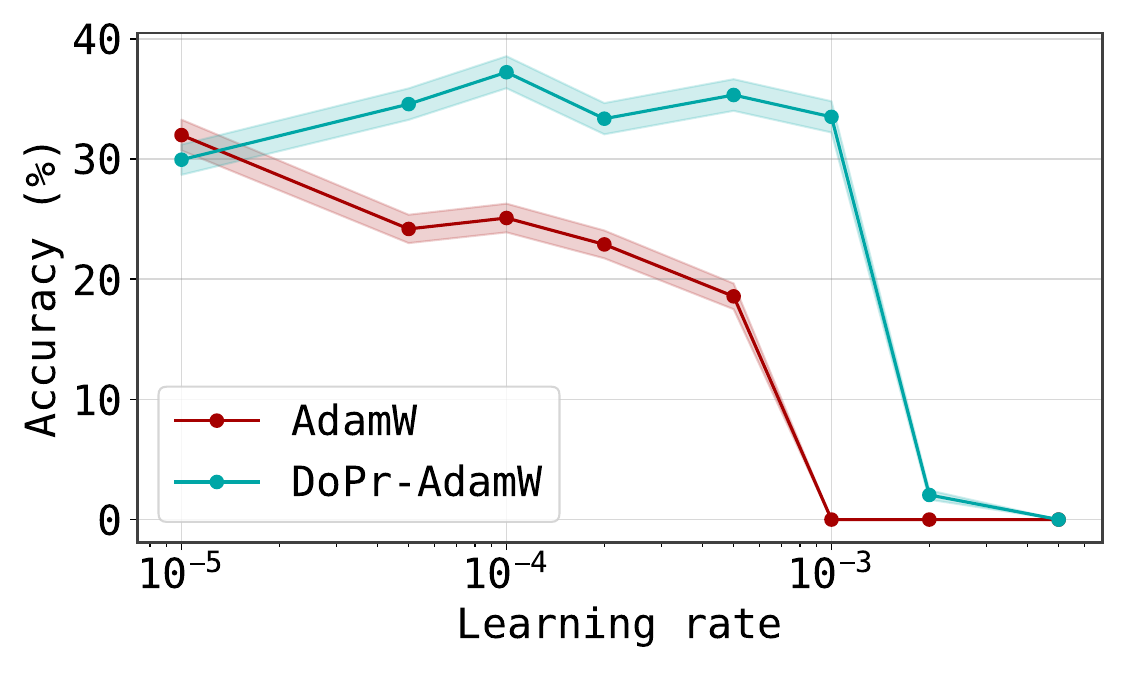}
    &
    \includegraphics[width=0.42\linewidth]{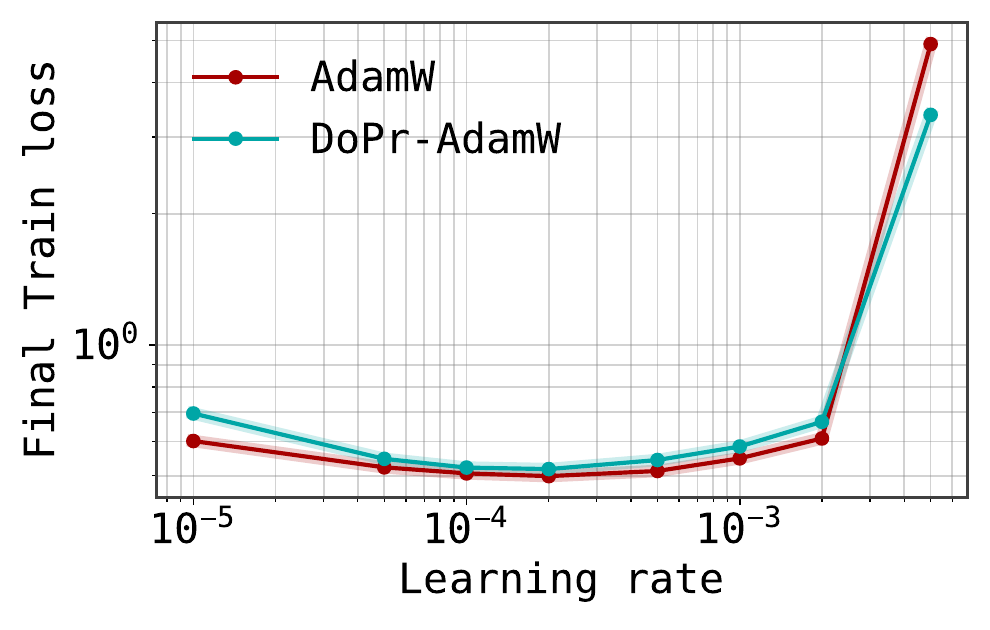}
  }
  {
  \includegraphics[width=0.4\linewidth]{figs/figs_sft/gsm8k__adamw_vs_precondadamw__peak_over_steps_vs_lr.pdf}
    &
    \includegraphics[width=0.38\linewidth]{figs/figs_sft/trainloss_vs_lr_adamw_vs_precondadamw.pdf}
  }
  \end{tabular}

  \vspace{-0.15cm}
  \caption{
  \texttt{GSM8K} \textbf{3B SFT sweep}.
  Peak \texttt{GSM8K} accuracy across training steps vs.\ learning rate, comparing \AdamW with \DoPe-\AdamW. The comparison between \Muon and \DoPe-\Muon is provided in \Cref{appdx:sft-3b-addres}.
  }
  \label{fig:sft}
  \vspace{-0.5cm}
\end{figure}

\nvsp[0.2]
\vspace{-0.2cm}
\icmlpar{8B Supervised Fine-tuning: TTF at Larger Scale.}
Following the core SFT setup of \texttt{OpenMathInstruct-2} \citep{toshniwal2024openmathinstruct}, we fine-tune \texttt{Llama-3.1-8B} \citep{llama3modelcard, grattafiori2024llama} with LoRA \citep{hu2022lora} on the \texttt{OpenMathInstruct-2} \texttt{train\_1M} split for two epochs. We sweep learning rates for \AdamW and \DoPe-\AdamW, and evaluate the final checkpoints on \texttt{GSM8K}, \texttt{GSM8K-CoT}, and \texttt{MATH-500}, and estimate validation NLL on a 10\texttt{K} held-out subset of the SFT data.
We examine in \Cref{fig:sft-8b} how validation loss corresponds to downstream performance. For \DoPe-\AdamW, lower validation NLL generally corresponds to stronger downstream accuracy, especially on \texttt{GSM8K-CoT} and \texttt{MATH-500}. This is desirable, as holding data and architecture equal, one would expect a model with the best held-out NLL also has the best performance. However, this is \textbf{not true} for baseline \AdamW. For larger learning rates, \AdamW-trained models predictably improve in-distribution validation loss 
but task performance (e.g., \texttt{GSM8K}) degrades. This illustrates our thesis that, without intervention, validation loss can be a poor proxy for full-sequence performance. Finally, \DoPe-\AdamW also reduces cross-task conflict: optimal learning rates on \texttt{GSM8K-CoT} are also (near-)optimal on \texttt{MATH-500}, whereas no \AdamW checkpoints are simultaneously near-optimal for \texttt{GSM8K} and \texttt{MATH-500} performance. This suggests that, as previewed in \Cref{sec:TTF feature learning}, \DoPr encourages more robust feature learning.
Additional results from the full sweep can be found in \Cref{appdx:sft-8b}.
\begin{figure}[h]
  \centering    
  \iftoggle{arxiv}{\includegraphics[width=\linewidth]{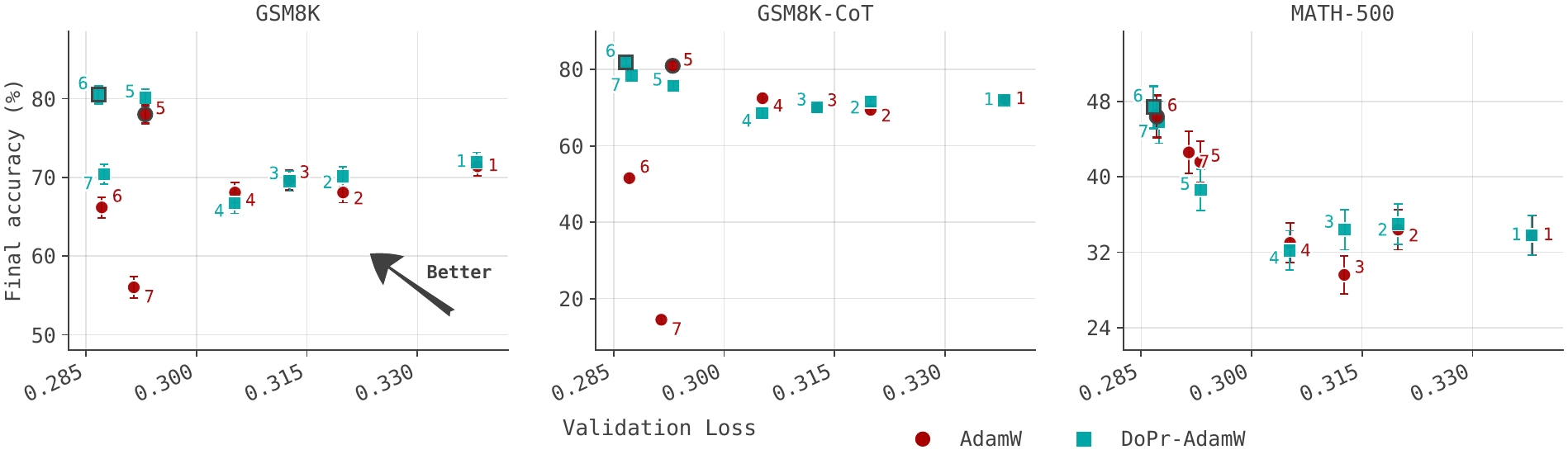}}
  {\includegraphics[width=0.9\linewidth]{figs/figs_sft/final_validation_nll_vs_final_accuracy.pdf}}
  \caption{
  \textbf{8B SFT sweep}.
   We plot final accuracy on \texttt{GSM8K}, \texttt{GSM8K-CoT}, and \texttt{MATH-500} against validation loss (NLL), computed on a 10K held-out subset of the training data. 
   \iftoggle{arxiv}{Each point is a final checkpoint from the learning-rate sweep while label numbers denote index of learning rates sorted by size: \texttt{2e-5, 5e-5, 7e-5, 1e-4, 2e-4, 5e-4, 7e-4}.}{Each point is a checkpoint from a learning-rate sweep with labels indexing by learning rate size (small to large).} Error bars indicate one standard error.
  }
  \label{fig:sft-8b}
  \vspace{-0.4cm}
\end{figure}

We lastly remark additional experiments and discussion on flow-based generative modeling, another TTF setting, can be found in \Cref{app:gen_modeling}.

\begin{AIbox}{Key Takeaways}
    \begin{enumerate}[
    noitemsep,
    topsep=-5pt,
    parsep=0pt,
    partopsep=0pt,
    left=2pt]
        \item We identify across the different TTF domains such as robot policy learning and language modeling a mismatch between train/validation loss and downstream performance metrics.

        \item We find that wrapping \DoPr around existing \GP{s} yields plug-in improvements on downstream performance, \textbf{despite not necessarily accelerating train/validation loss}, indicating an actionable facet of feature learning that is decoupled from loss convergence.

        \item \DoPr may enable commensurate scaling between training objective and downstream performance.

        \item As predicted in \Cref{sec:hyperparam scaling}, optimal \GP hyperparameters in \DoPr are inherited from the base \GP.
    \end{enumerate}
\end{AIbox}

\iftoggle{arxiv}{
}
{%
    \icmlpar{Generative Modeling.} We also view generative modeling as a TTF setting and report \DoPe results on score-based ImageNet-$256$ generation in \Cref{app:gen_modeling} due to space constraints.
}

\nvsp
\vspace{-0.3cm}
\section{Discussion}
\vspace{-0.3cm}
\nvsp

Toward the ultimate goal of optimizing for downstream performance, we: 1.\ introduce Test-Time Feedback as a unified setting for many modern applications, 2.\ identify the mismatch between features that accelerate training convergence versus those sensitive under TTF, 3.\ prescribe Double Preconditioning (\DoPe) that combines Activation Preconditioning (\AP) to equalize feature learning and Gradient Preconditioning (\GP) to stabilize training. We provide evidence for TTF and \DoPe being useful abstractions via experiments across distinct TTF applications. Importantly, we find that improved downstream performance may not accompany improved train/validation loss, suggesting that there remains a viable design space for improving deep learning optimization that is orthogonal to accelerating loss convergence. Notably, our proposed recipe of combining activation- and gradient-preconditioning is just one candidate of double preconditioning; we posit that decoupling considerations for directionality (e.g., \AP) and adaptivity (e.g., \GP) may lead to fruitful optimizer design. We conclude with a question:
\begin{bluequote} Given the incommensurability between validation loss and downstream performance under TTF, we conclude by asking: \textbf{is there still an optimizer free-lunch left on the table?}
\end{bluequote}

\clearpage
\nvsp
\section*{Acknowledgments}
\nvsp

TZ gratefully acknowledges a gift from AWS AI to Penn Engineering's ASSET Center for Trustworthy AI. TZ and NM are supported in part by NSF Award SLES-2331880, NSF CAREER award ECCS-2045834, NSF EECS-2231349, and AFOSR Award FA9550-24-1-0102. MS acknowledges support from a Google Robotics Award, and a Jane Street Fellowship. MS and TZ also thank Jeremy Cohen, Alex Damian, Hamed Hassani, and Behrad Moniri for helpful conversations.

\bibliographystyle{plainnat}
\bibliography{refs, refs_ZM2025}

\clearpage
\appendix
\crefalias{section}{appendix}
\crefalias{subsection}{appendix}
\crefalias{subsubsection}{appendix}

\tableofcontents
\section{Extended Related Work}\label{appdx:extended related}

\icmlpar{Imitation learning and the compounding errors problem.} The compounding errors problem has been recognized as a fundamental issue in policy learning from demonstration data \citep{pomerleau1988alvinn}, where small one-step learning errors may shift the learned model away from the offline data distribution it was trained on, leading to potentially catastrophic failure. Traditional approaches to mitigating compounding error come from modifying data collection, either through iterative online collection of recovery behavior \citep{ross2011reduction, kelly2019hg, foster2024behavior} or noising action execution \citep{laskey2017dart, zhang2025action} to better regularize the policy geometry locally around the data distribution. More recently, modifying \emph{policy parameterization} has been explored as an intervention for enabling consistent longer-horizon sequential deployment, such as the use of iteratively supervised methods (e.g., diffusion and flow-based parameterizations) \citep{chi2023diffusion, pan2025much} and action-chunking (i.e., multi-step prediction and execution) \citep{shafiullah2022behavior, chi2023diffusion, liu2025bidirectional}. Therefore, while data-collection and architecture design have been the primary avenues for mitigating compounding error, comparatively less focus has been directed to the role of the deep learning \emph{optimizer}. An exception of note is \citet{block2024butterfly}, which identifies that small loss fluctuations in gradient-based neural network training can cause disproportionately large or even oscillatory performance fluctuations in behavior cloning settings, precisely aligning with our core thesis that small error directions as measured by training loss can be disproportionately harmful when rolled out. Their proposed intervention is a variance-reduction approach via maintaining an EMA-ed copy of the model parameters for evaluation, which is complementary to our preconditioning approach, and as such we adopt in our continuous-space BC experiments. In most other literature, the optimizer is obviated as an oracle yielding (approximate) empirical risk minimizers (see e.g., \citet{ross2011reduction, foster2024behavior, zhang2025action}). However, the role of specifically the deep learning optimizer has received comparably more attention in deep reinforcement learning, likely due to the compounding difficulties RL presents over supervised learning, such as input- and target- nonstationarity \citep{lyle2022learning, sokar2023dormant, castanyer2025stable}. In this paper, we expose the potential suboptimalities of a wide range of deep learning optimizers that arise \takeawaybold{even in comparatively simpler offline supervised TTF settings}, such as behavior cloning.

\icmlpar{Fisher Information, \KFAC, and prior occurrences of Activation Preconditioning.} The Fisher Information matrix is a core object in mathematical statistics. When used as a preconditioner for local optimization in maximum likelihood estimation, the resulting method is known as the \emph{natural gradient method}. From this perspective, the Fisher Information in some cases is related to other curvature matrices such as the Hessian (or the Generalized Gauss-Newton approximation thereof), thus in settings where feasibly implementable, the natural gradient method often greatly accelerates optimization; see e.g., \citep{amari1998natural, amari2000methods, martens2020new}. Specified to neural networks, which have too many parameters for full second-order methods to be feasible, factorized versions of the Fisher Information have been proposed \citep{roux2007topmoumoute, ollivier2013riemannian, pascanu2013revisiting}. Most notably, under certain statistical independence assumptions, \citet{martens2015optimizing} derive a block-wise \emph{Kronecker-Factored} approximation of the Fisher Information, which renders preconditioning into left- and right- preconditioning the layerwise gradient \emph{matrix} $\bA^{-1} \nabla_\bW \cL \bB^{-1}$, which spawned the Kronecker-Factored Approximate Curvature (\KFAC) line of neural network optimizers, seeing consistent development and adaptation \citep{grosse2016kronecker, george2018fast, pauloski2021kaisa, eschenhagen2023kronecker, dangel2025kronecker}. However, most literature deriving from KFAC focus on the standard signal of improving train/validation convergence; due to the additional memory overhead of left- and right- preconditioner states per-layer, the numerical nuance of matrix inversion, and relative scarcity rules-of-thumb compared to popular \GP{s}, \AdaGrad-based (i.e., \GP) methods typically remain favored in mainstream training set-ups, though \KFAC finds utility in many scenarios that directly require curvature estimation \citep{he2025spectral, ikram2026crispedit}. 
Whereas \KFAC involves a ``left-side'' derivative-based preconditioner, its ``right-side'' preconditioner is simply the activation covariance, which only requires forward passes to compute. This has motivated ``right-side-only'' approximations to \KFAC (i.e., \AP in our parlance) \citep{benzing2022gradient}, which have been independently derived from proximal or local layer-wise loss perspectives \citep{frerix2018proximal, amid2022locoprop}. We note \AP has also been encountered from shallow neural-network learning theory \citep{zhang2023meta}, which we discuss in the dedicated feature learning theory paragraph. 

Furthermore, a perhaps underutilized perspective of \KFAC's approximation of the natural gradient method (or equivalently an alternate derivation of it) is reparameterization invariance; recall that natural gradient \emph{flow} traces a path over distributions that is invariant to smooth reparameterizations. Similarly, \KFAC, as an approximation thereof, induces a \emph{layerwise} invariance to affine reparameterizations, see \citet[Section 10]{martens2015optimizing}. The resulting optimizers can be equivalently recovered by a Euclidean-norm steepest descent formulation under a \emph{weighted} local inner product: $\ip{\bG, \nabla_{\btheta} \cL}_{\Sigma(\btheta)}$, $\mathrm{s.t. } \norm{\btheta}_2 \leq \eta$, where $\ip{\cdot, \cdot}_{\Sigma(\btheta)}$ is the classical ``Fisher metric'' for the natural gradient method, and the self-explanatory ``\KFAC'' metric for \KFAC \citep{luk2018coordinate}. Similarly, we may further reduce \KFAC to the one-sided \AP preconditioning that preserves the invariance property that we argue is salient for TTF in \Cref{prop:precond invariance}, where $\ip{\cdot, \cdot}_{\Sigma(\btheta)}$ is simply the activation-covariance reweighted inner product. Departing from the Euclidean-norm steepest descent to other (layerwise) weight norms recovers \DoPr for different \GP{s} \citep{pethick2025lmo}.

We also note concurrent work in \citep{du2026newton}, which independently derives activation-preconditioning from an ``isotropic curvature model'' \citep{su2025isotropic}. In fact, \DoPr-\Muon in our parlance precisely recovers their proposed algorithm. Crucially, the focus of their work is on improving pretraining loss convergence, while our central thesis studies the benefit of \AP independent of loss convergence; notably, we found that \DoPr-\Muon does not universally accelerate training/validation loss convergence across applications. Overall, we opted to use \AP rather than the full \KFAC-preconditioned gradient (which would additionally introduce a ``left-side'' preconditioner) due to subtleties of either requiring an MLE interpretation of the loss function \citep{martens2015optimizing}, contending with possible ``incorrectness'' of the approximation due to using the Empirical Fisher \citep{kunstner2019limitations}, and most importantly computing and possibly storing an additional optimizer preconditioner state. Anecdotally, using \AP- instead of \KFAC- preconditioning in \DoPr sufficed to provide the boosts in downstream performance documented in this paper, while avoiding compounding numerical instability or memory boundedness; however, we emphatically do not rule out the possibility for an efficient implementation or adaptation to yield further improvements. Ultimately, we view \DoPr as a framework for decoupling \emph{directionality} (e.g., from preconditioning) and \emph{self-normalization} (e.g., from popular \GP mechanisms).

\icmlpar{Adaptive gradient optimizers in deep learning.} Adaptive optimizers have served as the backbone of training deep learning models for more than a decade. In particular, the \AdaGrad family of gradient whiteners/normalizers \citep{duchi2011adaptive} serves as a unifying template for many ensuing optimizers. \RMSprop and \Adam use a ``diagonal'' approximation of the \AdaGrad gradient covariance matrix for adaptive optimization, while \Shampoo \citep{gupta2018shampoo, anil2020scalable, shi2023distributed} can be viewed as a layer-wise \emph{Kronecker-Factored} approximation (with one-sided variants \citep{an2026asgo}), and \Muon can be viewed as an instantaneous version of \Shampoo, which amounts to (approximately) orthogonalizing layerwise gradient matrices. From there, many further reductions or augmentations have been introduced. To name just a few, within the ``\Adam family'', we have the precursor \RMSprop \citep{tieleman2012lecture}, \texttt{AdaBelief} \citep{zhuang2020adabelief}, \texttt{LAProp} \citep{ziyin2020laprop}, \texttt{Lion} \citep{chen2023symbolic}, \texttt{MVN-Grad} \citep{patitucci2026adaptive}, augmented versions \texttt{AdEMAMix} \citep{pagliardini2025ademamix}, as well as reduced versions \texttt{AdaFactor} \citep{shazeer2018adafactor}, \texttt{GaLore} \citep{zhao2024galore}, \texttt{Adam-mini} \citep{zhang2025adammini}, among many more; see e.g., \citep{schmidt2021descending} for a more comprehensive account. On the matrix-shaped adaptive optimizer side, there have been recent developments on top of \Shampoo, e.g., through adaptive optimization in an eigenbasis-adjusted gradient (\texttt{SOAP} \citep{vyas2024soap}, \texttt{SPlus} \citep{frans2026splus}, \texttt{KL-Shampoo} \citep{lin2025understanding}), or additional adaptivity applied to \Muon, e.g., \texttt{AdaMuon} \citep{si2025adamuon}, \texttt{NorMuon} \citep{li2025normuon}, \texttt{MuonMax} \citep{crawshaw2025exploration}.
We also note the \texttt{PSGD} line of work \citep{li2017preconditioned, li2018preconditioner}, which spawns distinct preconditioners using various Lie group factorizations. Lastly, we mention a line of work that models (and creates) optimizers based on interpretations as steepest descent with respect to different choices of layerwise norm \citep{carlson2015preconditioned, bernstein2024modular, kovalev2025understanding, pethick2025lmo}.
We note that, with few exceptions, all the aforementioned optimizers are of a \emph{whitening} or \emph{normalizing} flavor, with the goal of ensuring the update direction is (approximately) well-conditioned in various senses, e.g., entry-wise, row/column-wise, or spectrum-wise. Though there are far too many optimizers in this list whose efficacy as \GP's in our \DoPe paradigm can be verified, our experiments on representative examples and general hypothesis is that any such ``normalizing-type'' \GP is compatible with the \AP gradient.

\icmlpar{Neural network feature learning theory.}
A by now classical approach to understanding how neural networks learn is via the neural tangent kernel (NTK) \citep{jacot2018neural, mei2022generalization}. However, kernel methods rely on fixed, features, which fundamentally limits their expressivity and leads to suboptimal sample complexity for learning nonlinear functions \citep{ghorbani2021linearized, ghorbani2021neural}.
This has motivated a large body of subsequent work studying feature learning in neural networks. In the setting of \textbf{isotropic covariates}, it has been shown that even a single step of SGD on the first layer of a two-layer network can learn sufficiently informative features to achieve improved sample complexity over kernel methods, provided the target function has low-dimensional structure \citep{damian2022neural, moniri2024asymptotics}. In contrast, for general anisotropic covariates, \cite{zhang2023meta} and \cite{zhang2025concurrence} identify that SGD is heavily distorted by anistropy of the inputs, which can cause the subspace distance between the learned and ground-truth representations to grow, even as training loss decreases. They reconcile this issue with a layerwise preconditioning schemes; the proposed algorithm in \cite{zhang2023meta} is a two-layer version of pure \AP (or equivalently \DoPr-\GD).

\icmlpar{Hyperparameter scaling conditions.}
A complementary line of work studies how optimizer hyperparameters should scale with model width, depth, and parameterization to ensure stable training and hyperparameter transfer. The Tensor Programs framework \citep{yang2019wide, yang2020tensor} provides a unifying theory of infinite-width limits, and maximal update parameterization ($\mu$P) \citep{yang2021tensor,yang2022tensor} derives learning-rate and initialization scalings that keep activations $\Theta(1)$ while preserving maximal feature learning across widths. Subsequent work extends these principles to second-order and matrix-preconditioned optimizers such as \KFAC and \Shampoo \citep{ishikawa2023parameterization}, alternative formulations based on spectral norms \citep{yang2023spectral}, depth-wise scaling laws and residual multipliers \citep{dey2025don,yang2024depth}, and newer optimizers including \texttt{SOAP} and \Muon \citep{qiu2025hyperparameter}. Recent work further shows that for long training runs, weight decay scaling can dominate stability, with some works advocating constant learning-rate--weight-decay products \citep{wang2025setadamsweightdecay,kosson2025weightdecaymattermup} and others proposing width-dependent weight decay scaling \citep{qiu2025hyperparameter}. Broadly, these methods derive optimizer-specific scaling rules, often through single-step update analyses, to enable efficient hyperparameter tuning across model sizes. Our work enables painless integration of these findings: by layering the \GP{s} studied in these works on top of the \AP update, we are able to make immediate use of these scaling rules.

\icmlpar{Mismatch between training performance and downstream behavior.} Beyond anecdotal or folklore understanding that training/validation loss across optimizers or even hyperparameters need not correlate to better downstream performance, there have been various works trying to understand or ameliorate this phenomenon. From the optimization perspective, recent work has identified that increasing pre-training effort (thus improving training performance) may lead to worse downstream finetunability \citep{springer2025overtrained, horoi2025less}. On the other hand, works since Chinchilla \citep{hoffmann2022training} have steadily supported that pretraining token-per-parameter scaling laws can be in practice be much larger \citep{gadre2025language} while downstream performance improves; notably, this includes prominent open-weight models such as \texttt{Llama 3} \citep{meta2024introducing}. Beyond coarse optimization interventions such as data quantity and training duration, certain works in the context of LLM fine-tuning have identified suboptimality in downstream performance when using different optimizers during pretraining and finetuning, e.g., \Adam v.s.\ \Muon, notably documented in Moonshot's whitepaper introducing large-scale \Muon-trained models, as well as follow-up studies \citep{liu2026optimizer}. We note that our fine-tuning experiments use open-weight pre-trained models, and thus using \DoPr for finetuning departs from the optimizer-consistency paradigm---it remains to verify: 1.\ if pretraining with \DoPr can further extend its benefits, 2.\ if our observed benefits using \DoPr for finetuning hold at scale and for general architectures. 
In a related perspective to ours regarding the benefits of ``uniform'' feature learning, some works have proposed using Sharpness-Aware Minimization (SAM) and extensions thereof as an \emph{optimization} intervention to yield better zero-shot or few-shot performance of the model on downstream tasks \citep{foret2021sharpness, bahri2022sharpness, zhao2022penalizing}, even at the cost of (pre-)training validation loss \citep{watts2026sharpness}. However, the mechanisms behind when SAM can yield reliable benefits as well as the feature geometry it actually induces do not necessarily reflect the nominal purpose of SAM \citep{andriushchenko2023sharpness, schapiro2024towards}; in the same light, \takeawaybold{the precise mechanisms by which \DoPr---or any modern optimizer---yield improvements in general remain largely ill-understood.}

As a parallel line of literature, many works have explored proxy statistics or metrics beyond naive held-out/validation loss on the training objective toward predicting downstream performance, robustness, or adaptability. As a small sample, we highlight certain works that explore \emph{jointly} accounting for gradient- and activation-statistics \citep{liu2026measure, liu2026spectral}. In \citet{liu2026measure}, a gradient-weighted activation statistic is proposed in the context of maintaining plasticity in RL networks, a phenomenon by which much of a network's capacity is silently wasted; notably, in RL contexts the training objective is nonstationary, and thus instantaneous training losses are often less predictive than in supervised learning. As our work focuses on supervised contexts, we leave investigating \DoPr's effectiveness in RL contexts for future work. Lastly, we note evidence cautioning using internal network statistics for concluding downstream performance; e.g. recent work finds that certain network trends typically associated with better/worse downstream performance may rather reflect current training practices, rather than a causal relationship with performance \citep{kulkarni2026disentangling, liu2026spectral}.

\section{Extended Method Derivation and Mechanisms}\label{appdx:derivations}

\subsection{Full \DoPe Algorithm}

In this section, we list the full \DoPe algorithm and draw attention to various important design decisions we made to obtain performant results.

\begin{algorithm}
\caption{Double Preconditioning (Feedforward Layer, layer index $\ell$ suppressed)}
\label{alg:DoPr_Full}
\begin{algorithmic}[1]
\STATE \textbf{Input:} Base Optimizer \texttt{GP},learning rate $\eta$, weight decay $\lambda$, damping $\gamma$, EMA rate $\beta$
\FOR{$k=1$ \textbf{to} $K$}
    \STATE Sample batch $B$, $\abs{B} = n$
    \STATE $\bG \leftarrow \nabla_{\bW} \cLhat(\btheta^{(k-1)})$
    \STATE $\Sigma \leftarrow \frac{1}{n}\sum_{i=1}^n \bz_i^{(k-1)} \bz_i^{(k-1)\top}$
    \STATE $\bGbar^{(k)} \leftarrow \beta \bGbar^{(k-1)} + (1-\beta)\bG$
    \STATE $\Sigmabar^{(k)} \leftarrow \beta \Sigmabar^{(k-1)} + (1-\beta)\Sigma$
    \STATE $\bM^{(k)} \leftarrow \bGbar^{(k)}\big(\Sigmabar^{(k)} + \lambda\,\trace(\Sigmabar^{(k)})\bI\big)^{-1}$ \hfill (\texttt{AP})
    \STATE $\bD^{(k)} \leftarrow \texttt{GP}(\bM^{(k)},\,^{\ast\ast}\texttt{kwargs})$ \hfill (\DoPe)
    \STATE $\bW^{(k)} \leftarrow (1-\eta\lambda)\bW^{(k-1)} - \eta \bD^{(k)}$
\ENDFOR
\end{algorithmic}
\end{algorithm}

\icmlpar{Handling Different Layer Types}
While we define \AP for linear layers, we extend it to other common architectural primitives by reducing them to equivalent linear maps. We unfold Conv1d layers and treat them akin to linear layers. For Conv2d layers, see \Cref{app:conv-affine-invariance}, and for Attention Layers, see \Cref{app:attention-affine-invariance}. Finally, for affine conditioning modules such as FiLM \citep{perez2018film} and AdaLN \citep{peebles2023scalable}, we recommend first trying \DoPr, as the performance of these modules can be sensitive to the relative learning speeds of the network backbone (which is presumably being trained with \DoPr). We note that zero-initialized networks (e.g., in AdaLN-Zero) may require slight attention as initial activations may be zero, potentially destabilizing early training due to ill-conditioned \AP. We found that with damping and the self-normalizing property of popular \GP{s}, \DoPr can work without modification, but if early training is unstable, we recommend a short warmup period with the base \GP before switching to the full \DoPr.

\icmlpar{Decoupled Weight Decay}
This stabilizes training by controlling parameter norms independently of the gradient update \citep{loshchilov2017decoupled}.

\icmlpar{EMA}
We provide the option to maintain exponential moving averages of both the gradient and the activation covariance. For a generic estimator $\bG^{(k)}$, the EMA update is
\[
\bar{\bG}^{(k)} = \beta \bar{\bG}^{(k-1)} + (1-\beta)\bG^{(k)}.
\]
Unrolling the recursion gives
\[
\bar{\bG}^{(k)} = (1-\beta)\sum_{t=0}^{k-1} \beta^t \bG^{(k-t)} + \beta^k \bar{\bG}^{(0)},
\]
i.e., a geometrically weighted average of past iterates.
Assuming $\bG^{(k)}$ are unbiased estimators of some quantity $\mathbf{\mu}$ with covariance $\Sigma_G$, we have
\[
\mathbb{E}[\bar{\bG}^{(k)}]
=
(1-\beta^k)\mathbf{\mu} + \beta^k \bar{\bG}^{(0)},
\]
so the estimator is biased at early iterations, but converges to $\mathbf{\mu}$ at rate $\beta^k$.
For the variance, assuming independence for simplicity,
\[
\mathrm{Var}(\bar{\bG}^{(k)})
=
(1-\beta)^2 \sum_{t=0}^{k-1} \beta^{2t} \Sigma_G
\;\longrightarrow\;
\frac{1-\beta}{1+\beta}\,\Sigma_G
\quad \text{as } k \to \infty.
\]
Thus EMA reduces variance by a factor of $\frac{1-\beta}{1+\beta}$ while introducing a bias, which becomes negligible for sufficiently large $k$.
We provide the option of applying EMA to both the gradient and the activation covariance. \takeawaybold{We recommend first trying \AP in \DoPr \emph{without} additional EMAs:} 1.\ most of our experiments compute \AP with just the batch-wise gradient and activation covariance, 2.\ many base \GP{s} are equipped with their own gradient EMAs (e.g.\ \Adam and \Muon), and thus applying gradient EMA in the \AP computation requires accounting for doubly applying EMA, 3.\ using batch-wise statistics in \AP removes the additional persistent memory requirement for the activation covariance buffer, making \DoPr's memory cost equal to its base \GP.

\icmlpar{Damping}
We add a damping term to the activation covariance estimate of the form
$\Sigma_{\bz} + \lambda\,\frac{\trace(\Sigma_{\bz})}{\mathrm{d}_{\bz}}\bI $
to ensure invertibility. We note that the damping is set \emph{relative} to the magnitude of $\Sigma_{\bz}$ via $\frac{\trace(\Sigma_{\bz})}{\mathrm{d}_{\bz}}$ such that the damping hyperparameter is scale-invariant \citep{ishikawa2023parameterization, dey2025don}. Though damping the activation covariance has conceptual connections with ridge regression, we posit that damping, and indeed using generic PD linsolvers that requires full-rankness such as Cholesky, is likely crude, as damping causes small activation directions (which \AP aims to upweight) to be boosted equally as \emph{zero} directions (which \AP should in principle keep zero). We leave designing flexible, fast linsolvers that better handle rank degeneracy and pseudo-inversion (e.g., extensions of the lauded Newton-Schulz iteration in \Muon) to future work.

\icmlpar{Inversion and Approximations}
For most layers, we compute the inverse of the activation covariance using PyTorch's implementation of Cholesky decomposition \texttt{torch.linalg.cholesky\_ex}.
In our experiments, this was tractable for the considered model sizes, and usually not the compute bottleneck. This is in line with empirical insight that matrix operations that scale independently of the batch size (such as in computing $\bG \Sigma^{-1}$) become negligible in large-batch training. For certain settings with inflated activation dimensions, such as in the im2col parameterization of convolutional layers, we also support structured approximations such as rank-$1$+diagonal (derived in \cref{app:r1d}) and rank-$k$ plus diagonal decompositions \citep{ahn2025dion2}. For convolutional layers specifically, in line with \KFAC's development, we employ the Spatially Uncorrelated Approximation (SUA) \cref{app:sua} to bypass the (mostly redundant) cost of forming and inverting large patch covariance matrices. We also remark that for LLM embedding layers lead to input (i.e., ``activation'') dimensions that scale with comparatively huge vocabulary size: however, we note that due to the one-hot nature of the inputs, the activation covariance is diagonal, where each non-zero entry is just the count of that token in the batch, and thus \AP-ing the embedding layer gradient can be performed extremely efficiently and does not require forming a full $\mathrm{d}_{\mathrm{vocab}} \times \mathrm{d}_{\mathrm{vocab}}$ matrix.

\subsection{\AP For General Architectures: Derivations via Invariance}

This section contains proofs for invariance under non-degenerate layer-wise affine transforms for feedforward, convolutional, and attention layers.

\icmlpar{Feedforward Affine Invariance}
\label{app:feedforward-affine-invariance}
\PrecondInvariance*

\begin{proof}
    Recall that: 
    \begin{align*}
        \bSigma_{\bz}^{-1} &= \paren{ \E \brac{ \bzbar \bzbar^\top}}^{-1} \\
        &= \paren{ \E \brac{ \bA_\ell \bz_\ell \bz_\ell^\top \bA_\ell^\top }}^{-1} \\
        &= \paren{ \bA_\ell^\top}^{-1} \bSigma_\ell^{-1} \bA_\ell ^{-1}.
    \end{align*}
By the chain rule
    \begin{align*}
        \nabla_{\bWbar} \cL \paren{f_{\bar{\theta}}} &= \E \brac{ \frac{\partial \cL}{\partial h} \bzbar^\top} \\
        &= \E \brac{ \frac{\partial \cL}{\partial h} \bz_\ell^\top } \bA_\ell^\top.
    \end{align*}
    Therefore:
    \begin{align*}
        \bWbar^{\rmnext} \bzbar &= \paren{ \bWbar - \eta \nabla_{\bWbar} \cL \paren{ f_{\bar{\theta}}} \bSigma_{\bzbar}^{-1} } \bzbar \\
        &= \paren{ \bW_\ell \bA_\ell^{-1} - \eta \Exp \brac{ \frac{\partial \cL}{\partial h} \bz_\ell^\top } \bA_\ell^\top \paren{ \bA_\ell^\top}^{-1} \bSigma_{\bz_\ell}^{-1} \bA_\ell ^{-1} } \bA_\ell \bz_\ell \\
        &= \paren{ \bW_\ell - \eta \Exp \brac{ \frac{\partial \cL}{\partial h} \bz_\ell^\top } \bSigma^{-1}_{\bz_\ell} } \bz_\ell \\
        &= \bW^{\rmnext}_\ell \bz_\ell
    \end{align*}
    Since the argument holds for arbitrary $\ell$, the conclusion holds for all layers, and thus the network output.

\end{proof}

\icmlpar{Convolutional Affine Invariance}
\label{app:conv-affine-invariance}

A convolutional layer can be written as a linear map applied to unfolded input patches via the standard \emph{im2col}parameterization. Specifically, let
$\bZ_\ell = \operatorname{im2col}(\bz_\ell) \in \mathbb{R}^{d_{\rm patch} \times N}$
denote the matrix of flattened patches, where each column $\bz_{\ell,b,p}$ corresponds
to a spatial patch indexed by batch element $b$ and location $p$, and let
$\bW_\ell \in \mathbb{R}^{C_{\rm out} \times d_{\rm patch}}$ denote the flattened
convolutional kernel. Then the layer can be written as $\bH_\ell = \bW_\ell \bZ_\ell$.

\begin{corollary}[Convolutional Affine invariance]
Consider a convolutional layer in im2col form with patch matrix $\bZ_\ell$ and weights $\bW_\ell$, and consider the transformed variables $\bZbar = \bA_\ell \bZ_\ell$, $\bWbar = \bW_\ell \bA_\ell^{-1}$ for invertible $\bA_\ell$. Then under the AP update $\bW_\ell^{\rm next} = \bW_\ell - \eta \nabla_{\bW_\ell} \cL(f_\theta)\bSigma_{\bZ_\ell}^{-1}$ with $\bSigma_{\bZ_\ell} = \mathbb{E}_{b,p}[\bz_{\ell,b,p}\bz_{\ell,b,p}^\top]$, the updated weights satisfy $\bWbar^{\rm next}\bZbar = \bW_\ell^{\rm next}\bZ_\ell$.
\end{corollary}

\begin{proof}
Under the im2col representation, the convolutional layer is a linear map applied to patch vectors. The transformation $\bZbar = \bA_\ell \bZ_\ell$, $\bWbar = \bW_\ell \bA_\ell^{-1}$ preserves pre-activations, i.e. $\bWbar \bZbar = \bW_\ell \bZ_\ell$. Thus the setting reduces exactly to the feedforward case of Proposition~4.3 applied to the patch vectors $\bz_{\ell,b,p}$, yielding the claim.
\end{proof}

\icmlpar{Attention Affine Invariance}
\label{app:attention-affine-invariance}

A self-attention layer first applies learned linear projections to an input token
matrix $\bZ_\ell \in \mathbb{R}^{d_{\rm model}\times N}$:
\[
\begin{aligned}
    \bQ_\ell &= \bW_\ell^Q \bZ_\ell, \\
    \bK_\ell &= \bW_\ell^K \bZ_\ell, \\
    \bV_\ell &= \bW_\ell^V \bZ_\ell .
\end{aligned}
\]
The attention output is then computed as a deterministic function of these
projected activations, for example
\[
    \operatorname{Attn}(\bQ_\ell,\bK_\ell,\bV_\ell)
    =
    \bV_\ell
    \operatorname{softmax}
    \left(
        \bK_\ell^\top \bQ_\ell / \sqrt d
    \right),
\]
up to convention-dependent transposes. Thus, if an affine change of coordinates
is applied to the input token representation and the inverse transformation is
absorbed into each projection matrix, the projected queries, keys, and values are
unchanged.

\begin{corollary}[Attention projection AP invariance]
Consider a self-attention layer with input activations $\bZ_\ell$ and projection weights $\{\bW_\ell^r\}_{r\in\{Q,K,V\}}$. Let $\bA_\ell$ be invertible, and define $\bZbar=\bA_\ell\bZ_\ell$ and $\bWbar^r=\bW_\ell^r\bA_\ell^{-1}$ for each $r\in\{Q,K,V\}$. Then under the AP update $(\bW_\ell^r)^{\rm next}=\bW_\ell^r-\eta\nabla_{\bW_\ell^r}\cL(f_\theta)\bSigma_{\bZ_\ell}^{-1}$, where $\bSigma_{\bZ_\ell}=\mathbb{E}_{b,t}[\bz_{\ell,b,t}\bz_{\ell,b,t}^\top]$, the updated projections satisfy $(\bWbar^r)^{\rm next}\bZbar=(\bW_\ell^r)^{\rm next}\bZ_\ell$ for each $r\in\{Q,K,V\}$. Consequently, the updated attention queries, keys, values, and attention outputs are invariant under the transformed parameterization.
\end{corollary}

\begin{proof}
For each projection $r\in\{Q,K,V\}$, the map
$\bZ_\ell \mapsto \bW_\ell^r\bZ_\ell$ is a feedforward linear layer applied to
token activations. Moreover, the transformed parameterization preserves the
corresponding projected activations:
\[
\begin{aligned}
    \bWbar^r \bZbar
    &=
    \bW_\ell^r \bA_\ell^{-1}\bA_\ell\bZ_\ell \\
    &=
    \bW_\ell^r\bZ_\ell.
\end{aligned}
\]
Therefore Proposition~4.3 applies separately to the query, key, and value
projections, yielding
\[
    (\bWbar^r)^{\rm next}\bZbar
    =
    (\bW_\ell^r)^{\rm next}\bZ_\ell
    \qquad
    \text{for each } r\in\{Q,K,V\}.
\]
Thus the updated queries, keys, and values are identical in the two
parameterizations. Since the attention output is a deterministic function of
$\bQ_\ell,\bK_\ell,\bV_\ell$, the updated attention output is also identical.
\end{proof}

\subsection{Approximation Schemes for \AP}

For larger models, particularly those with 2d convolutions, storing and computing the full activation covariance used in the preconditioner can be expensive. To subvert this cost, we develop multiple preconditioning approximation schemes.

\subsubsection{Rank-1 + Diagonal}
\label{app:r1d}

Let $\mu = \E \left[ \bX \right]$ and $\var = \E \left[ \bX\bX^\top \right] - \mu \odot \mu^\top $. We approximately decompose
\begin{align}
    \Sigma + \lambda \bI &\approx \text{diag} \paren{ \var }  + \lambda \bI + \mu \mu^\top \notag
\end{align}

Let $\bD = \text{diag} \paren{ \var } + \lambda I $. Since $\bD$ is diagonal and $\mu \mu^\top$ is rank-1, by the Sherman-Morrison formula we may compute
\begin{align}
    \paren{ \Sigma + \lambda \bI}^{-1} &= \paren{ \bD + \mu \mu^\top }^{-1} \\
    &= \bD^{-1} - \frac{ \paren{\bD^{-1} \mu \mu^\top  \bD^{-1}}}{1 + \mu^\top \bD^{-1} \mu}
\end{align}
Which can be computed in linear time with respect to the hidden dimension $D$. 

\subsubsection{Spatially Uncorrelated Activations (SUA)}
\label{app:sua}
We derive the SUA \citep{george2018fast} approximation for a single 2D convolutional layer using an \texttt{im2col} parameterization.

Let the (tensor-shaped) convolutional weights be $\bW \in \R^{C_{\out} \times C_{\inp} \times k_h \times k_w}$ and let $\btheta = \VEC(\bW)$.
For each example $i \in [n]$ and output spatial index $u \in [m]$ (where $m = H_{\out}W_{\out}$), let
$\ba_{i,u} \in \R^{d}$ denote the vectorized input patch (concatenating channels and kernel offsets) with
\[
d \;=\; C_{\inp} k_h k_w,
\]
and let $\bg_{i,u} \in \R^{C_{\out}}$ denote the gradient of the loss w.r.t.\ the pre-activation at location $u$.
Define the stacked matrices
\[
\bA_i \triangleq \bmat{\ba_{i,1} & \cdots & \ba_{i,m}}^\top \in \R^{m \times d},
\qquad
\bG_i \triangleq \bmat{\bg_{i,1} & \cdots & \bg_{i,m}}^\top \in \R^{m \times C_{\out}}.
\]

Then the per-example weight gradient (reshaped to $C_{\out}\times d$) is
\[
\nabla_{\bW}\ell(\bx_i) \;=\; \bG_i^\top \bA_i \;\in\; \R^{C_{\out}\times d}.
\]
Using $\VEC(\bX\bY\bZ) = (\bZ^\top \otimes \bX)\VEC(\bY)$ gives
\[
\nabla_{\btheta}\ell(\bx_i)
\;\triangleq\;
\VEC(\nabla_{\bW}\ell(\bx_i))
\;=\;
\VEC(\bG_i^\top \bA_i)
\;=\;
(\bA_i^\top \otimes \bI_{C_{\out}})\,\VEC(\bG_i^\top).
\]

The empirical Fisher for this layer is
\[
\hat\bF
\;\triangleq\;
\hatExp 
\Big[\nabla_{\btheta}\ell(\bx)\,\nabla_{\btheta}\ell(\bx)^\top\Big]
\;=\;
\frac{1}{n}\sum_{i=1}^n
\sum_{u=1}^m\sum_{v=1}^m
\Big(\ba_{i,u}\ba_{i,v}^\top\Big)\otimes\Big(\bg_{i,u}\bg_{i,v}^\top\Big).
\]
The double sum over $(u,v)$ shows that the exact curvature couples \emph{pairs of spatial locations}, making exact storage/inversion expensive when $m$ is large.

KFAC \cite{grosse2016kronecker} applies the standard approximation that (patch) activations and backpropagated gradients are independent:
\[
\Exp\Big[(\ba_u\ba_v^\top)\otimes(\bg_u\bg_v^\top)\Big]
\;\approx\;
\Exp[\ba_u\ba_v^\top] \;\otimes\; \Exp[\bg_u\bg_v^\top].
\]
This reduces the Fisher to products of second moments of $\ba$ and $\bg$.

SUA further assumes spatial uncorrelatedness of activations:
\[
\Exp[\ba_u\ba_v^\top] \;=\; \mathbf{0}
\qquad\text{for } u\neq v,
\]
(and analogously, in many implementations, $\Exp[\bg_u\bg_v^\top]\approx \mathbf{0}$ for $u\neq v$).
Under this approximation, only the diagonal terms $u=v$ remain, yielding
\[
\bF
\;\approx\;
\sum_{u=1}^m \Exp[\ba_u\ba_u^\top] \;\otimes\; \Exp[\bg_u\bg_u^\top].
\]
If we additionally assume stationarity across spatial locations (i.e. the second moments do not depend on $u$), then
\[
\Sigma_{\ba} \;\triangleq\; \Exp[\ba_u\ba_u^\top] \in \R^{d\times d},
\qquad
\Sigma_{\bg} \;\triangleq\; \Exp[\bg_u\bg_u^\top] \in \R^{C_{\out}\times C_{\out}},
\]
and thus
\[
\bF \;\approx\; m\,(\Sigma_{\ba}\otimes \Sigma_{\bg}),
\]
where the scalar factor $m$ is absorbed into the step size / damping.

We estimate the factors by averaging over both batch and spatial indices:
\[
\hat{\Sigma}_{\ba}
\;=\;
\frac{1}{nm}\sum_{i=1}^n\sum_{u=1}^m \ba_{i,u}\ba_{i,u}^\top
\;=\;
\hatExp_{i,u}[\ba_{i,u}\ba_{i,u}^\top],
\qquad
\hat{\Sigma}_{\bg}
\;=\;
\frac{1}{nm}\sum_{i=1}^n\sum_{u=1}^m \bg_{i,u}\bg_{i,u}^\top
\;=\;
\hatExp_{i,u}[\bg_{i,u}\bg_{i,u}^\top].
\]
Notably, $\hat{\Sigma}_{\ba}$ has size $(C_{in}k_hk_w)\times(C_{in}k_hk_w)$ and is independent of the feature-map resolution or batch size (and reduces to $C_{in}\times C_{in}$ for $1\times1$ convolutions).

Let $\bg_W \triangleq \nabla_{\bW}\ell$ be the weight gradient in matrix form $\R^{C_{\out}\times d}$.
Under the Kronecker-factored approximation $\bF \approx \Sigma_{\ba}\otimes \Sigma_{\bg}$, the (damped) KFAC/SUA preconditioner applies
\[
\Delta \bW
\;\propto\;
(\Sigma_{\bg}+\lambda \bI)^{-1}\;\bg_W\;(\Sigma_{\ba}+\lambda \bI)^{-1},
\]
followed by reshaping $\Delta\bW$ back to the tensor shape of $\bW$.
Importantly, SUA does not drastically reduce the expressivity of the preconditioner as the discarded curvature directions correspond to spatial interaction modes that are unidentifiable under weight sharing. As such, they lie largely in the nullspace of the parameterization. Note that we only use the right-sided SUA.

\subsection{Feature Learning Theory}\label{appdx:feature learn theory}

We adapt parts of \citet{zhang2023meta} and \citet{zhang2025concurrence} toward establishing the feature learning result in \Cref{prop:SGD and AP feature learning}. We note for presentational convenience, we make some technical assumptions:
\begin{assumption}\label{assumption:Gtheta}
    The current iterate $\Gtheta$ satisfies:
    \begin{enumerate}[left=5pt]
        \item $\Gtheta \in \R^{d \times n}$ is row-orthonormal.
        \item $\dist(\Gtheta, \Gstar)\leq \frac{\lmin(\Fstar^\top \Fstar) \lmin(\bSigma_{\bs})}{\lmax(\Fstar^\top \Fstar) \lmax(\bSigma_{\bs})}$.
    \end{enumerate}
\end{assumption}

\begin{proposition}[\Cref{prop:SGD and AP feature learning}, formal ver.]\label{prop:SGD and AP feature learning formal}
    Let \Cref{assumption:Gtheta} hold. Given $0 < \eta \leq 0.5 \lmax(\Fstar^\top \Fstar)^{-1}$. Then, one step of full-batch \GD (resp.\ \AP) satisfies:
    \begin{align}
    \begin{split}
        \GthetaGD \Gperp  &= (\Gtheta - \eta \Ftheta^\top \Ftheta \Gtheta \bSigma_{\bs}) \Gperp + \eta \Ftheta^\top \Fstar \Gstar \bSigma_{\bs} \Gperp\\
    \GthetaAP\Gperp &=  (\bI - \eta \Ftheta^\top \Ftheta) \Gtheta \Gperp.
    \end{split}
    \end{align}
    Then, the \AP update satisfies the following one-step contraction in subspace distance: $\dist(\GthetaAP, \Gstar) \leq \paren{1 - \eta \lmin(\Ftheta^\top \Ftheta)} \dist(\Gtheta, \Gstar)$. On the other hand, one may construct $\Ftheta$, $\Fstar$, $\Gstar$, $\bSigma_{\bs}$ such that $\dist(\GthetaGD, \Gstar) > \dist(\Gtheta, \Gstar)$. 
    
\end{proposition}

\begin{proof}[Proof of \Cref{prop:SGD and AP feature learning}]
    First, recalling the loss: $\cL(\Ftheta \Gtheta) = \Exp^{\pidemo} \norm{\Ftheta\Gtheta - \sblue}_2^2 = \norm{(\Ftheta \Gtheta - \Fstar\Gstar)\bSigma_{\bs}^{1/2}}_F^2$, we have the \GD and \AP update:
    \begin{align*}
        \GthetaGD &= \Gtheta - \eta \nabla_{\Gtheta}\cL \\
        &= \Gtheta - \eta \Ftheta^\top \Ftheta \Gtheta \bSigma_{\bs}  + \eta \Ftheta^\top \Fstar \Gstar \bSigma_{\bs} \\
        \GthetaAP &= \Gtheta - \eta \nabla_{\Gtheta}\cL \cdot \Sigmas^{-1} \\
        &= \Gtheta - \eta \Ftheta^\top \Ftheta \Gtheta  + \eta \Ftheta^\top \Fstar \Gstar.
    \end{align*}
\end{proof}
Noting $\Gstar \Gperp = \bzero$, applying $\Gperp$ to the above yields \eqref{eq:GD vs AP update}. Now, given that
\begin{align*}
    \GthetaAP\Gperp &= (\bI - \eta \Ftheta^\top \Ftheta) \Gtheta \Gperp,
\end{align*}
applying $\opnorm{\cdot}$ and submultiplicativity on the right-hand side yields:
\begin{align*}
    \dist(\GthetaAP, \Gstar) &\leq \opnorm{\bI - \eta \Ftheta^\top \Ftheta} \opnorm{\Gtheta \Gperp} \\
    &= (1 - \eta \lmin(\Ftheta^\top \Ftheta)) \dist(\Gtheta, \Gstar),
\end{align*}
completing the one-step contraction bound for a step of \AP.

To show that conversely $\GthetaGD$ may in general be non-convergent in subspace distance, we construct a numerical example to illustrate this.
Take \(d=1\), \(n=2\), and let
\[
\Gstar = \begin{bmatrix}1 & 0\end{bmatrix},
\qquad
\Gperp =
\begin{bmatrix}
0 & 0 \\
0 & 1
\end{bmatrix}.
\]
Choosing $\Gtheta
=
\begin{bmatrix}\cos(\pi/3) & \sin(\pi/3)\end{bmatrix}
=
\begin{bmatrix}1/2 & \sqrt{3}/2\end{bmatrix}$, we have
\[
\operatorname{dist}(\Gtheta,\Gstar)
=
\|\Gtheta\Gperp\|_2
=
\left\|\begin{bmatrix}0 & \sqrt{3}/2\end{bmatrix}\right\|_2
=
\frac{\sqrt{3}}{2}.
\]
Now setting
\[
\Sigmas
=
\begin{bmatrix}
5 & 2 \\
2 & 1
\end{bmatrix} \succ \bzero,
\quad
\Ftheta=1,
\quad
\Fstar=1,
\quad
\eta=1,
\]
we have
\[
\GthetaGD \Gperp
=
(\Gtheta
-\Gtheta\Sigmas)\Gperp
+\Gstar\Sigmas\Gperp.
\]
Computing the three terms:
\[
\Gtheta\Gperp
=
\begin{bmatrix}0 & \sqrt{3}/2\end{bmatrix},
\]
\[
\Gtheta\Sigma_s\Gperp
=
\begin{bmatrix}
0 & 1+\sqrt{3}/2
\end{bmatrix},
\qquad
\Gstar\Sigma_s\Gperp
=
\begin{bmatrix}0 & 2\end{bmatrix},
\]
this yields
\begin{align*}
\GthetaGD \Gperp
&=
\begin{bmatrix}0 & \sqrt{3}/2\end{bmatrix}
-
\begin{bmatrix}0 & 1+\sqrt{3}/2\end{bmatrix}
+
\begin{bmatrix}0 & 2\end{bmatrix} \\
&=
\begin{bmatrix}0 & 1\end{bmatrix}.
\end{align*}
Hence
\[
\dist(\GthetaGD, \Gstar)
=1 > \frac{\sqrt{3}}{2}
=
\operatorname{dist}(\Gtheta,\Gstar),
\]
completing the construction and the proof.

\subsection{Dynamical Systems Theory}\label{appdx:dynamical systems theory}

\paragraph{Proof of \Cref{prop:LDS TTF}.}
We first establish the linear systems preliminaries in \Cref{prop:LDS TTF}
\LDS*
\begin{proof}
    Recall the state transition $\bs_{t+1} = \bA \bs_t + \bB \ba_t + \bw_t$, $\bw_t \sim \Normal(\bzero,\bm{\Sigma}_{\bw})$. Under $\ba_t = \bK \bs_t$, we have for given $\bs_t$:
    \begin{align*}
        \bs_{t+1} \sim \Normal((\bA + \bB \bK)\bs_t, \Sigma_\bW).
    \end{align*}
    Applying this recursively to $\bs_0 = \bzero$ yields:
    \begin{align*}
        \bs_t \sim \Normal\paren{\bzero, \sum_{s=0}^{t-1} (\bA + \bB \bK)^{s} \bm{\Sigma_}{\bw} (\bA + \bB \bK)^{s \top}} =: \Normal(\bzero, \Gamma_t(\bK)).
    \end{align*}
    Plugging in $\bK = \Kdemo, \Ktheta$ completes the proof of the distribution of the state marginals: $\Pr^{\pidemo}_{\sblue}  = \Normal(0, \bGamma_t(\Kdemo))$, $\Pr^{\pitheta}_{\sred}  = \Normal(0, \bGamma_t(\Ktheta))$. As for the losses, we have:
    \begin{align*}
        \cLval(\pitheta) &= \Exp^{\pidemo}\norm{\Ktheta \sblue[] - \ablue[]}_2^2 \\
        &= \Exp^{\pidemo}\frac{1}{T}\sum_{t=1}^{T}\norm{(\Ktheta - \Kdemo)\sblue}_2^2 \\
        &= \frac{1}{T}\sum_{t=1}^{T} \Exp^{\pidemo}\trace\paren{(\Ktheta - \Kdemo) \sblue \sblue^\top (\Ktheta - \Kdemo)^\top} \\
        &= \trace\paren{(\Ktheta - \Kdemo) \frac{1}{T}\sum_{t=1}^{T} \Exp^{\pidemo}[\sblue \sblue^\top] (\Ktheta - \Kdemo)^\top} \\
        &= \trace\paren{(\Ktheta - \Kdemo) \bar \bGamma_t(\Kdemo) (\Ktheta - \Kdemo)^\top} \\
        &= \norm{(\Ktheta - \Kdemo) \bar \bGamma_t(\Kdemo)^{1/2}}_F^2, \\
        \cLideal(\pitheta) &= \Exp^{\pitheta}\norm{\Ktheta \sred[] - \ared[]}_2^2  \\
        &= \trace\paren{(\Ktheta - \Kdemo) \frac{1}{T}\sum_{t=1}^{T} \Exp^{\pitheta}[\sred \sred^\top] (\Ktheta - \Kdemo)^\top} \\
        &= \norm{(\Ktheta - \Kdemo) \bar \bGamma_t(\Ktheta)^{1/2}}_F^2.
    \end{align*}
    Finally, when the divergence $\Diverg$ is the Wasserstein-$2$ distance, we may use the formula for Wasserstein-$2$ distance between two multivariate Gaussians $\bbP = \Normal(\bzero, \bSigma_1)$, $\bbQ = \Normal(\bzero, \bSigma_2)$:
    \begin{align*}
        \cW^2_2(\bbP, \bbQ) &= \trace(\bSigma_1 + \bSigma_2 - 2(\bSigma_1^{1/2} \bSigma_2 \bSigma_1^{1/2})^{1/2}).
    \end{align*}
    Plugging in $\bSigma_1 = \bar \bGamma_t(\Kdemo)$, $\bSigma_2 = \bar \bGamma_t(\Ktheta)$ and simplifying yields the result and completes the proof.

\end{proof}

\subsubsection{Mismatch between $\cLideal$ and $\cLval$: a construction}\label{appdx:construct LDS Lideal Lval example}

We want to demonstrate that due to feedback between a policy and the dynamics, two policies that may satisfy $\cLval(\pi_1) \ll \cLval(\pi_2)$ may in fact satisfy $\cLideal(\pi_1) \gg \cLideal(\pi_2)$.
Consider the dynamical system
\begin{align}
    \bA = \begin{bmatrix}
        0 & 0\\
        0 & 1
    \end{bmatrix}, \quad \bB = \eye, \quad \bSigma_{\bw} = \eye
\end{align}
Next, we set 
\begin{align}
    \Kdemo = \begin{bmatrix} 0 & 0\\
     0 &   -\alpha 
    \end{bmatrix}, \quad \Fstar = \be_2, \quad \Gstar = -\alpha \be_2^\top
\end{align}

\newcommand{\Kpsi}{{\color{\nicepurple}\bK_{\psi}}}

\newcommand{\Fpsi}{{\color{\nicepurple}\bF_{\psi}}}

\newcommand{\Gpsi}{{\color{\nicepurple}\bG_{\psi}}}

Finally, let us consider two learned parameters
\begin{align}
    &\Ktheta = \Ftheta \Gtheta, \quad \Ftheta = \mu \be_2, \quad \Gtheta = -\alpha\mu^{-1}\begin{bmatrix} 
\cos\theta \\
\sin \theta  \end{bmatrix}^\top\\
    &\Kpsi = \Fpsi \Gpsi, \quad \Fpsi = \begin{bmatrix} \alpha \\
        1
    \end{bmatrix}, \quad \Gpsi = \Gstar
\end{align}

Then, we have that
\begin{align}
    \Ktheta &= \Kdemo - \alpha  \begin{bmatrix} 0 & 0 \\
    - \cos  \theta & 1 - \sin \theta 
    \end{bmatrix}\\
    \Kpsi &= \Kdemo - \epsilon  \begin{bmatrix}  0 & 1\\
     0 &0 
    \end{bmatrix}
\end{align}
And the following computation will be useful:
\begin{lemma}\label{lem:err_comp}It holds that 
\begin{align}
    (\Ktheta - \Kdemo)^\top (\Ktheta - \Kdemo) &= \alpha^2 \begin{bmatrix} \cos^2 \theta & - \cos \theta (1-\sin \theta) \\
    -\cos \theta (1-\sin \theta) & (1-\sin \theta)^2
    \end{bmatrix}\\
    (\Kpsi - \Kdemo)^\top (\Kpsi - \Kdemo) &= \begin{bmatrix} 0& 0 \\
    0 & \epsilon^2
    \end{bmatrix}
\end{align}
\end{lemma}

For simplicity, we work now in the limit where $T \to \infty$.
The following is standard:
\begin{lemma}\label{lem:ricatti} Let $\bK$ be such that $\rho(\bA + \bB \bK) < 1$, where $\rho(\cdot)$ denotes the spectral radius. Then,  $\bGamma_{\infty}(\bK) := \lim_{T\to\infty}\bGamma_T(\bK)$ exists and is finite, is equal to $\lim_{T\to\infty}\bar\bGamma_T(\bK)$, and is given by the unique solution to the following Riccati equation:
\begin{align}
    \bGamma_{\infty}(\bK) = \bSigma_{\bw} + (\bA + \bB \bK) \bGamma_\infty(\bK) (\bA + \bB \bK)^\top.
\end{align}
\end{lemma}

We now compute the stationary covariances under $\Kdemo, \Ktheta, \Kpsi$.
\begin{lemma}  $\bGamma_{\infty}(\Kdemo)$ exists, and is given by 
\begin{align}
    \bGamma_{\infty}(\Kdemo) = \begin{bmatrix} 1 & 0 \\
    0 & \frac{1}{1-(1-\alpha)^2} 
    \end{bmatrix}
\end{align}
\end{lemma}
\begin{proof} $\bA + \bB \Kdemo$ has eigenvalues $0.5,1-\alpha$, so \Cref{lem:ricatti} holds. In this case, the limiting $\Gamma$ satisfies
\begin{align}
    \bGamma =  \begin{bmatrix} 1 & 0 \\
    0 & 1-\alpha \end{bmatrix} \bGamma \begin{bmatrix} 0 & 0 \\
    0 & 1-\alpha \end{bmatrix} + \begin{bmatrix}
        1 & 0 \\
        0 & 1
    \end{bmatrix}
\end{align}
decoupling across coordinates. So we select $\bGamma$ to be diagonal with coordinates $\gamma_1,\gamma_2$ on the diagonals. These must satisfy $\gamma_1$, and $\gamma_2 (1-\alpha)^2 + 1 = \gamma_2$, so that $\gamma = 1/(1 - (1-\alpha)^2)$
\end{proof} 
\begin{lemma}\label{lem:ricatti_theta} For $\theta \in (0,\pi/2)$, $\bGamma_\infty(\Ktheta)$ exists, 
\begin{align}
\bGamma_\infty(\Ktheta) =
\begin{bmatrix}
1 & \displaystyle 0 \\
0 &\frac{1+\alpha^2\cos^2\theta}{\alpha\sin\theta
\left(2-\alpha\sin\theta\right)}
\end{bmatrix}
\end{align}
\end{lemma}
\begin{proof}[Proof of \Cref{lem:ricatti_theta}]
The spectral radius condition can be checked for $\theta \in (0,\pi/2)$ by upper triangularity. Now set
\begin{align}
\bGamma
=
\begin{bmatrix}
0 & 0 \\
-\alpha\cos \theta & 1-\alpha\sin \theta
\end{bmatrix}
\bGamma
\begin{bmatrix}
0 &  \\
-\alpha\cos \theta & 1-\alpha\sin \theta
\end{bmatrix}^\top
+
\begin{bmatrix}
1 & 0 \\
0 & 1
\end{bmatrix}
\end{align}
We now solve the above:
\begin{align*}
\bX
&:=
\begin{bmatrix}
0 & 0 \\
-\alpha\cos\theta & 1-\alpha\sin\theta
\end{bmatrix},
\qquad
\bQ
=
\begin{bmatrix}
1 & 0 \\
0 & 1
\end{bmatrix}.
\end{align*}

Now write
\begin{align*}
\bGamma
&=
\begin{bmatrix}
\gamma_{11} & \gamma_{12} \\
\gamma_{12} & \gamma_{22}
\end{bmatrix}.
\end{align*}

The Lyapunov equation to solve is $\bGamma
=
\bX\bGamma\bX^\top + Q$.
Computing
\begin{align*}
\bX\bGamma\bX^\top
&=
\begin{bmatrix}
0 & 0 \\
0 &
\alpha^2\cos^2\theta \, \gamma_{11}
-
2\alpha\cos\theta
\left(1-\alpha\sin\theta\right)\gamma_{12}
+
\left(1-\alpha\sin\theta\right)^2\gamma_{22}
\end{bmatrix}.
\end{align*}

Matching entries gives
\begin{align*}
\gamma_{11}
&=
1,
\\
\gamma_{12}
&=
0,
\\
\gamma_{22}
&=
\alpha^2\cos^2\theta \, \gamma_{11}
-
2\alpha\cos\theta
\left(1-\alpha\sin\theta\right)\gamma_{12}
+
\left(1-\alpha\sin\theta\right)^2\gamma_{22}
+
1.
\end{align*}

Using \(\gamma_{11}=1\) and \(\gamma_{12}=0\), the final equation becomes
\begin{align*}
\gamma_{22}
&=
\alpha^2\cos^2\theta
+
\left(1-\alpha\sin\theta\right)^2\gamma_{22}
+
1.
\end{align*}

Therefore, $\left[
1-\left(1-\alpha\sin\theta\right)^2
\right]\gamma_{22}
=
1+\alpha^2\cos^2\theta$, and since $1-\left(1-\alpha\sin\theta\right)^2
=
\alpha\sin\theta
\left(2-\alpha\sin\theta\right)$,
we obtain
\begin{align*}
\gamma_{22}
&=
\frac{
1+\alpha^2\cos^2\theta
}{
\alpha\sin\theta
\left(2-\alpha\sin\theta\right)
}.
\end{align*}

Hence
\begin{align*}
\bGamma
=
\begin{bmatrix}
\displaystyle
1
&
\displaystyle
0
\\[1.2em]
\displaystyle
0
&
\displaystyle
\frac{
1+\alpha^2\cos^2\theta
}{
\alpha\sin\theta
\left(2-\alpha\sin\theta\right)
}
\end{bmatrix}.
\end{align*}

\end{proof}

\begin{lemma}\label{lem:Gamma_psi} $\bGamma_{\infty}(\Kpsi)$ exists, and satisfies
\begin{align*}
   \bGamma_{\infty}(\Kpsi) = 
\begin{bmatrix}
1+\frac{\epsilon^2}{\alpha(2-\alpha)}
&
-\frac{\epsilon(1-\alpha)}{\alpha(2-\alpha)}
\\
-\frac{\epsilon(1-\alpha)}{\alpha(2-\alpha)}
&
\frac{1}{\alpha(2-\alpha)}
\end{bmatrix}.
\end{align*}
\end{lemma}
\begin{proof}[Proof of \Cref{lem:Gamma_psi}]
Again, we can check the stability conditions from upper tri-angularity. We have
\begin{align}
\bGamma
=
\underbrace{
\begin{bmatrix}
0 & -\epsilon \\
0 & 1-\alpha
\end{bmatrix}
}_{:= \bX}
\bGamma
\begin{bmatrix}
0 & -\epsilon \\
0 & 1-\alpha
\end{bmatrix}^\top
+
\begin{bmatrix}
1 & 0 \\
0 & 1
\end{bmatrix}
\end{align}

\begin{align*}
\bX
&:=
\begin{bmatrix}
0 & -\epsilon \\
0 & 1-\alpha
\end{bmatrix},
\qquad
\bQ
=
\begin{bmatrix}
1 & 0 \\
0 & 1
\end{bmatrix}.
\end{align*}

Now writing $
\bGamma
=
\begin{bmatrix}
\gamma_{11} & \gamma_{12} \\
\gamma_{12} & \gamma_{22}
\end{bmatrix}$, the Lyapunov equation to solve is \(\bGamma = \bX \bGamma \bX^\top + \bQ\). First, we compute
\begin{align*}
\bX \bGamma \bX^\top
&=
\begin{bmatrix}
\epsilon^2 \gamma_{22}
&
-\epsilon(1-\alpha)\gamma_{22}
\\[0.8em]
-\epsilon(1-\alpha)\gamma_{22}
&
(1-\alpha)^2\gamma_{22}
\end{bmatrix}.
\end{align*}

Matching entries gives
\begin{align*}
\gamma_{11}
&=
\epsilon^2\gamma_{22}+1,
\\
\gamma_{12}
&=
-\epsilon(1-\alpha)\gamma_{22},
\\
\gamma_{22}
&=
(1-\alpha)^2\gamma_{22}+1.
\end{align*}

We solve from bottom to top. The third equation gives \(\left(1-(1-\alpha)^2\right)\gamma_{22} = 1\).
Since \(1-(1-\alpha)^2 = \alpha(2-\alpha)\), we obtain \(\gamma_{22} = \frac{1}{\alpha(2-\alpha)}\).
Therefore,
\begin{align*}
\gamma_{12}
&=
-\epsilon(1-\alpha)\gamma_{22}
=
-\frac{\epsilon(1-\alpha)}{\alpha(2-\alpha)}
\gamma_{11}
&=
1+\epsilon^2\gamma_{22}
=
1+\frac{\epsilon^2}{\alpha(2-\alpha)}.
\end{align*}
Hence
\begin{align*}
\bGamma
=
\begin{bmatrix}
\displaystyle
1+\frac{\epsilon^2}{\alpha(2-\alpha)}
&
\displaystyle
-\frac{\epsilon(1-\alpha)}{\alpha(2-\alpha)}
\\[1.4em]
\displaystyle
-\frac{\epsilon(1-\alpha)}{\alpha(2-\alpha)}
&
\displaystyle
\frac{1}{\alpha(2-\alpha)}
\end{bmatrix}.
\end{align*}
\end{proof}

\noindent From here, we can compute the various error terms

\begin{lemma}\label{lem:scaling_limit} Let $\theta \in (0,\pi/4)$. Then, taking the $T \to \infty$ limit, 
\begin{align}
\cLval(\Ktheta) &\sim \alpha, \quad \cLideal(\Ktheta) \sim \frac{\alpha}{\sin \theta},
\end{align}
where as 
\begin{align} \cLval(\Kpsi), \cLideal(\Kpsi) \sim \frac{\epsilon^2}{\alpha}
\end{align}
\end{lemma}
\begin{proof}
From \Cref{prop:LDS TTF}, the fact that $\|\Delta\bGamma^{1/2}\|_{\fro}^2 = \trace(\Delta^\top\Delta\bGamma )$, and that $\Delta = \Kdemo - \Ktheta$, we have that $\trace(\Delta^\top\Delta\bGamma ) = \alpha^2\bGamma_{22} (1-\sin \theta)^2$ by \Cref{lem:err_comp}, we obtain
\begin{align}
\cLval(\Ktheta) &= \alpha^2 \cos^2 \theta + \alpha^2(1-\sin \theta)^2 \cdot \frac{1}{1-(1-\alpha)^2} \sim \alpha 
\end{align}
for $\alpha$ small and $\theta \in (0,\pi/4)$. Moreover, from these same argument
\begin{align*}
\cLideal(\Ktheta) = \alpha^2 \cos^2 \theta + \alpha^2(1-\sin \theta)^2 \cdot \frac{1+\alpha^2 \cos^2 \theta}{\alpha \sin \theta (2-\alpha \sin \theta)} \sim \alpha^2 +\frac{\alpha}{\sin \theta} \sim \frac{\alpha}{\sin \theta}.
\end{align*}
We may compute the quantities on $\Kpsi$ similarly.
\end{proof}
By modulating $\epsilon, \alpha$ appropriately, the following proposition follows immediately:
\begin{proposition}\label{prop:Ktheta Kpsi bound} For $ \epsilon \ll \alpha$ and $\sin \theta \ll \epsilon^2 /\alpha^2$, we have
we have that \begin{align}\cLval(\Ktheta) \ll \cLval(\Kpsi) \sim \cLideal(\Kpsi) \ll \cLideal(\Ktheta).
\end{align}
\end{proposition}

Continuing with these quantities, we may go on to prove \Cref{prop:TTF_feature_learning}.

\subsubsection{Proof of \Cref{prop:TTF_feature_learning}}

We recall \Cref{prop:TTF_feature_learning}:
\FeatureTTF*

Notably, compared to the constructions above, the difference is that \Cref{prop:TTF_feature_learning} introduces the additional complication of allowing the flexibility of re-fitting the last-layer $\Ftheta$. Inheriting $\Kdemo, \Ktheta, \Kpsi$ and relevant quantities therein from earlier, the key result to establish is:

\begin{lemma}\label{lemma:infF_lb}
    For $\theta \in (0,\pi/8)$, we have 
    $\inf_{\bF} \cLideal(\bF \Gtheta) \ge \frac{\alpha^2}{4 \sin \theta}$. 
\end{lemma}

By rescaling $\Gtheta,\Ftheta$ by $\mu^{-1}$ and $\mu$, and using \Cref{prop:LDS TTF}  as in the proof of \Cref{lem:scaling_limit}, it suffices to prove the following.
\begin{lemma} Let $\theta \le \pi/8$.
Define the vector
\begin{align}
\bw = -\alpha
\begin{bmatrix}
\cos\theta \\
\sin \theta
\end{bmatrix}^\top,
\quad
\bA =
\begin{bmatrix}
0 & 0 \\
0 & 1
\end{bmatrix},
\quad
\bQ =
\begin{bmatrix}
1 & 0 \\
0 & 1
\end{bmatrix}
\end{align}
For any $\bv$, define two terms: $\bX_{\bv} = \bA - \bv \bw^\top$, $\Delta_{\bv} = \alpha \bA - \bv \bw^\top$.
When $\bv$ is such that a solution $\bGamma_{\bv}$ to the Riccati equation
\begin{align}
\bGamma_{\bv} = \bX_{\bv} \bGamma_{\bv}\bX_{\bv}^\top + \bQ.
\end{align}
exists, where $\bGamma_{\bv}$ is defined as above. Then, we have
\begin{align}
\inf_{\bv \in \R^2: \bGamma_{\bv} \text{ exists}} \trace\left(\Delta_{\bv}^\top\Delta_{\bv}\cdot \bGamma_{\bv}\right) \ge \frac{\alpha^2}{4\sin \theta}.
\end{align}
\end{lemma}
\begin{proof}

\begin{align*}
\bw
&=
-\alpha
\begin{bmatrix}
\cos\theta \\
\sin\theta
\end{bmatrix},
\qquad
\bA
=
\begin{bmatrix}
0 & 0 \\
0 & 1
\end{bmatrix},
\qquad
\bQ
=
\begin{bmatrix}
1 & 0 \\
0 & 1
\end{bmatrix}.
\end{align*}
For any \(\bv \in \mathbb{R}^2\), define \(\bX_{\bv} = \bA-\bv\bw^\top\) and \(\Delta_{\bv} = \alpha \bA-\bv\bw^\top\).
Since
$-\bv\bw^\top
=
\alpha \bv
\begin{bmatrix}
\cos\theta & \sin\theta
\end{bmatrix}$,
we have
\begin{align*}
\bX_{\bv}
&=
\bA+\alpha \bv
\begin{bmatrix}
\cos\theta & \sin\theta
\end{bmatrix},
&
\Delta_{\bv}
&=
\alpha\bA+\alpha \bv
\begin{bmatrix}
\cos\theta & \sin\theta
\end{bmatrix}.
\end{align*}
Let \(c := \cos\theta\), \(s := \sin\theta\), and \(\tau := \tan\theta = \frac{s}{c}\).
Writing $\bv
=\begin{bmatrix}
v_1 \\
v_2
\end{bmatrix}$, 
define \(p := \alpha v_1\) and \(q := \alpha v_2\).
Then
\begin{align*}
\bX_{\bv}
&=
\begin{bmatrix}
pc & ps \\
qc & 1+qs
\end{bmatrix},
&
\Delta_{\bv}
&=
\begin{bmatrix}
pc & ps \\
qc & \alpha+qs
\end{bmatrix}.
\end{align*}
Now reparameterize by \(r := pc = \alpha v_1 \cos \theta\) and \(h:= -\alpha v_2 \sin \theta \theta\).
Since \(s,c>0\), this is an equivalent parameterization. In these variables,
\begin{align*}
\bX_{\bv}
&=
\begin{bmatrix}
r & r\tau \\
-\frac{h}{\tau} & 1-h
\end{bmatrix},
&
\Delta_{\bv}
&=
\begin{bmatrix}
r & r\tau \\
-\frac{h}{\tau} & \alpha-h
\end{bmatrix}.
\end{align*}
The characteristic polynomial of \(\bX_{\bv}\) is \(\chi(z) = z^2-(1+r-h)z+r\); therefore, by the Schur stability conditions, the Lyapunov/Riccati solution \(\bGamma_{\bv} = \bX_{\bv}\bGamma_{\bv}\bX_{\bv}^{\top} + \bQ\) exists if and only if \(r<1\), \(h>0\), and \(2+2r-h>0\). Equivalently, \(-1<r<1\) and \(0<h<2(1+r)\).
Define the objective \(J(r,h) := \operatorname{tr}\left(\Delta_{\bv}^{\top}\Delta_{\bv}\bGamma_{\bv}\right)\).
Solving the Lyapunov equation and substituting gives
\begin{align*}
J(r,h)
&=
\frac{
N(r,h,\tau,\alpha)
}{
h\tau^2(1-r)(2+2r-h)
},
\end{align*}
where
\begin{align*}
N(r,h,\tau,\alpha)
&=
\alpha^2(1+r)h^2
+
\alpha^2\tau^2(1-r)^2(1+r)
+
2(1-\alpha)h^3
\\
&\quad
+
\tau^2
\Big[
2\alpha^2hr
-
2\alpha h^2r
+
2\alpha hr^2
-
2\alpha h
+
h^2(1+r)
+
2hr^2
\Big]
\\
&\quad
+
\tau^4r^2(1+r).
\end{align*}

For fixed \(0<\alpha\leq 1\) and \(\theta\) sufficiently small, the numerator satisfies the lower bound
\begin{align*}
N(r,h,\tau,\alpha)
&\ge
\frac{\alpha^2}{2}(1+r)
\left[
h^2+\tau^2(1-r)^2
\right].
\end{align*}

Also, since \(0<h<2(1+r)\), we have \(2+2r-h \leq 2(1+r)\). Therefore,
\begin{align*}
J(r,h)
&\ge
\frac{
\frac{\alpha^2}{2}(1+r)
\left[
h^2+\tau^2(1-r)^2
\right]
}{
h\tau^2(1-r)\,2(1+r)
}
\\
&=
\frac{\alpha^2}{4}
\frac{
h^2+\tau^2(1-r)^2
}{
h\tau^2(1-r)
}.
\end{align*}

By AM-GM, \(h^2+\tau^2(1-r)^2 \geq 2h\tau(1-r)\). Hence
\begin{align*}
J(r,h)
&\geq
\frac{\alpha^2}{4}
\cdot
\frac{
2h\tau(1-r)
}{
h\tau^2(1-r)
}
\\
&=
\frac{\alpha^2}{2\tau}.
\end{align*}

Thus,
\begin{align*}
\inf_{\bv:\,\bGamma_{\bv}\ \mathrm{exists}}
\operatorname{tr}
\left(
\Delta_{\bv}^{\top}\Delta_{\bv}\bGamma_{\bv}
\right)
\geq
\frac{\alpha^2}{2\tan\theta}.
\end{align*}

Since \(\tan\theta \leq 2\sin\theta\) for \(\theta\) sufficiently small, this further implies
\begin{align*}
\inf_{\bv:\,\bGamma_{\bv}\ \mathrm{exists}}
\operatorname{tr}
\left(
\Delta_{\bv}^{\top}\Delta_{\bv}\bGamma_{\bv}
\right)
\geq
\frac{\alpha^2}{4\sin\theta}.
\end{align*}

\end{proof}

Thus, chaining \Cref{prop:Ktheta Kpsi bound} and \Cref{lemma:infF_lb} together yields the following quantitative form of \Cref{prop:TTF_feature_learning}.
\begin{proposition}For $\sin \theta \ll \frac{\epsilon^2}{\alpha^3}$ and $\epsilon \ll \alpha$, we have
    \begin{align}
\cLval(\Ktheta) \ll \cLval(\Kpsi) \sim \cLideal(\Kpsi) \ll  \inf_{\bF} \cLideal(\bF\Gtheta) \le \cLideal(\Ktheta).
    \end{align}
\end{proposition}

\subsection{Limitations}\label{appdx:limitations}

\DoPe incurs a nontrivial computational and memory overhead. For a linear layer with input dimension $d_{\mathrm{in}}$, output dimension $d_{\mathrm{out}}$, and batch size $N$, the per-step overhead consists of accumulating the input covariance, factoring the damped preconditioner, and applying the inverse preconditioner to the gradient:
\[
\bS = \bX^\top \bX, \qquad \bX \in \mathbb{R}^{B \times d_{\mathrm{in}}},
\]
which costs approximately $B d_{\mathrm{in}}^2$
FLOPs (due to symmetry of $\bS$. Computing the Cholesky factorization of the damped matrix $\bA = \bS + \lambda I \in \mathbb{R}^{d_{\mathrm{in}} \times d_{\mathrm{in}}}$
costs approximately $\frac{1}{3} d_{\mathrm{in}}^3$
FLOPs. Finally, applying the inverse preconditioner to the weight gradient
\[
\bG(\bS+\lambda \bI)^{-1}, \qquad \bG \in \mathbb{R}^{d_{\mathrm{out}} \times d_{\mathrm{in}}},
\]
via two triangular solves costs $2 d_{\mathrm{in}}^2 d_{\mathrm{out}}$
FLOPs. Thus, the total per-layer per-step idealized FLOP overhead is
\[
\boxed{
N d_{\mathrm{in}}^2
+ \frac{1}{3} d_{\mathrm{in}}^3
+ 2 d_{\mathrm{in}}^2 d_{\mathrm{out}}
}.
\]
In line with empirical know-how (see e.g., \citet{jordan2024muon}), in the large-batch regime, the cost of forming the activation covariance $\bS = \bX^\top \bX$ dominates. Thus, for layers where $d_{\mathrm{in}} \approx d_{\mathrm{out}}$, the overhead is approximately equal to an additional forward pass. Indeed, throughout our experiments, we found the overhead to be $\approx 1.25x$ in wall-clock time, commensurate with the estimate of a backward pass equal to $\sim3$ forward passes, such that \AP's overhead coarsely comes out to $5/4 = 1.25 \times$ regular backprop.

\DoPe also introduces additional memory overhead from storing the matrix-valued preconditioner statistics. With an exponential moving average (EMA), this requires storing one $d_{\mathrm{in}} \times d_{\mathrm{in}}$ symmetric matrix per preconditioned layer, i.e.,
$\sim d_{\mathrm{in}}^2/2$ additional parameters of optimizer state per layer.

Our current implementation of \DoPr only supports Distributed Data Parallelism, and not model parameter sharding such as FSDP. However, we do not anticipate any fundamental barriers for such schemes beyond existing related matrix-preconditioning methods such as \texttt{Shampoo}, which have distributed and sharded implementations \citep{shi2023distributed, ahn2025dion}.
Finally, \DoPe can be sensitive to the choice of matrix-inversion solver and damping parameter. In practice, we found PyTorch's Cholesky solve together with trace-scaled damping to provide a sufficiently stable and effective default.

\DoPe may improve downstream performance without improving training or validation loss, and it may not improve training-loss convergence relative to the base \GP optimizer. Moreover, exact \AP for convolutional or other weight-sharing layers can incur substantial memory overhead due to unfolding. While we introduce approximation schemes to account for the redundancy therein, developing hardware-accelerated approximations and exploring model--optimizer codesign remain important directions for future work.

Finally, we observe that large dropout rates can conflict with \AP EMA buffers and accelerate staleness due to zero-ing out random subsets of activations; therefore, in our current practice, we recommend using no or low dropout. In principle, we note more advanced schemes for updating the activation buffer (e.g., analogous to \citet{ahn2025dion2}) exist, but may require case-by-case treatment depending on how the architecture is implementated.
Improving the compatibility of \DoPe with dropout is left for future work.

\section{Experiment Details and Additional Results}\label{appdx:exps}

\icmlpar{Hardware}
\begin{itemize}[left=0pt]
    \item All imitation learning experiments were performed on 4$\times$ NVIDIA RTX A5000 GPUs.
    \item All language modeling experiments were performed on 8$\times$ NVIDIA RTX PRO 6000 Blackwell GPUs.
    \item All image generation experiments were performed on 8$\times$ NVIDIA A100 GPUs.
\end{itemize}

\subsection{General Operating Guideline}
\label{app:best_practices}
\vspace{-0.2cm}
In practice, we summarize a small set of operating guidelines that we have found to be broadly useful across settings.

\begin{itemize}[
    noitemsep,
    topsep=-5pt,
    parsep=0pt,
    partopsep=0pt,
    left=0pt]

    \item \textbf{EMA smoothing for very small-batch, continuous-valued inputs.}
    For continuous-valued inputs (in contrast to discrete-token embeddings in LLMs), we suggest applying to the \AP an exponential moving average (EMA) to both the activation covariance and gradient to reduce update variance when the batch size is significantly smaller than the model width. We note that appropriate considerations should be made if the \GP of choice also contains EMAs, as this may doubly apply smoothing on the gradient.

    \item \textbf{Use no/low dropout.}
    Dropout \citep{srivastava2014dropout} is designed to prevent spurious co-adaptation between features in the network. Firstly, due to masking different subsets of activations per update, large dropout can conflict with accumulation in the activation preconditioner buffers, accelerating staleness. Furthermore, \AP can be viewed as similarly de-biasing the gradient from its activation (i.e.\ feature) correlations. Therefore, we recommend using no dropout for most settings, or reduced amounts in large-scale convolutional architectures.

    \item 
    \textbf{Handling high-redundancy/weight-sharing layers.}
    \AP is most directly derived for linear layers. For layers that have high-redundancy and weight-sharing properties, such as convolutional layers, the exact \AP (resulting from \Cref{prop:precond invariance}) can incur needlessly large memory overhead from unfolding. 
    Following prior literature in Kronecker-Factored preconditioners \citep{grosse2016kronecker, george2018fast}, we may derive an efficient approximation of \AP that empirically preserves the performance of the full \AP as in \Cref{app:sua}. We also note that one of the motivations of heavy weight-sharing is to make the \emph{architecture} ``well-conditioned'' \citep{wilson2017marginal, keskar2017improving, xiao2018dynamical}, and thus may be partially redundant to our optimizer that is designed to handle heavy anisotropy. As we do not modify baseline model architectures in this work, we leave exploring model-optimizer co-design for future work.

\end{itemize}

\subsection{Feature Learning: Validation Loss versus Subspace Distance}
\label{appdx:feature learn exp}

This section provides details on the experiments shown in \Cref{fig:feature_learn}.

\paragraph{Data generation.}
We follow the synthetic two-layer setup of \citet{zhang2025concurrence}. We first generate a base matrix $\bF_0 \in \R^{\dy \times k}$ with i.i.d.\ $\normal(0,1)$ entries. For each task $\bs \in \{\mathsf{train}, \mathsf{test}\}$, we sample a matrix $\bB_\bs \in \R^{\dy \times \dy}$ and define the task-specific head
\[
\bF_\star^{\bs} = \exp\!\left(0.005(\bB_{\bs} - \bB_{\bs}^\top)\right)\bF_0,
\]
which produces small random rotations of a shared base feature matrix. The shared representation $\bG_\star \in \R^{k \times \dx}$ is sampled uniformly from the set of row-orthonormal matrices. For each task $\bs$, inputs are generated as
\[
\bx_i^{\bs} \sim \bSigma_{\bx,\bs}^{1/2} \cdot \mathrm{Unif}(\{\pm 1\}^{\dx}),
\]
and outputs are generated according to
\[
\by_i^{\bs} = \bF_\star^{\bs} \bG_\star \bx_i^{\bs} + \bveps_i^{\bs}, \quad
\bveps_i^{\bs} \sim \normal(0, \sigma^2_{\ep,\bs} \bI_{\dy}).
\]

\paragraph{Anisotropy.}
We sample a random rotation $\bO \in \R^{\dx \times \dx}$ and define
\[
\bSigma_{\bx,\bs} = \bO \bD \bO^\top,
\quad
\bD = \mathbf{diag}(\texttt{logspace}(0,5,\dx)),
\]
yielding highly ill-conditioned covariances. This setting amplifies the feature-learning bias described in \Cref{prop:SGD and AP feature learning}.

\paragraph{Model and training.}
We train linear models of the form $\bF_\theta \bG_\theta$ on the training task $\mathsf{train}$ using squared loss, i.e.,
\[
\cL_{\mathrm{train}}(\theta)
=
\mathbb{E}_{(\bx,\by)\sim \cD_{\mathsf{train}}}
\left[
\|\bF_\theta \bG_\theta \bx - \by\|^2
\right].
\]
We compare \texttt{SGD}, \texttt{Adam}, and \DoPe-\texttt{SGD}, using batch size $n=1024$ and averaging results over 10 random seeds. All methods are initialized identically and use matched learning rates.

\paragraph{Metrics.}
We report (i) the in-distribution validation loss $\cLval$ on held-out samples from $\mathsf{train}$, and (ii) the subspace distance $\dist(\Gtheta,\Gstar)$, which measures alignment between the learned and ground-truth feature subspaces. The latter isolates representation learning independent of downstream fitting.

\paragraph{Results.}
As shown in \Cref{fig:feature_learn}, validation loss suggests successful optimization across all methods. However, the subspace distance reveals a clear discrepancy: standard and gradient-preconditioned methods fail to recover the true feature subspace under high anisotropy, whereas \AP significantly improves alignment. This demonstrates that low validation loss can mask poor feature learning, and supports the claim that \AP mitigates TTF by promoting more uniform feature recovery.

\subsection{Activation Invariance of \DoPe}\label{appdx:activation invariance}
This section provides details on the experiments shown in \Cref{fig:AP invariance property}.

\iftoggle{arxiv}{
}{
\begin{figure}[t]

  \begin{center}        \includegraphics[width=0.85\linewidth]{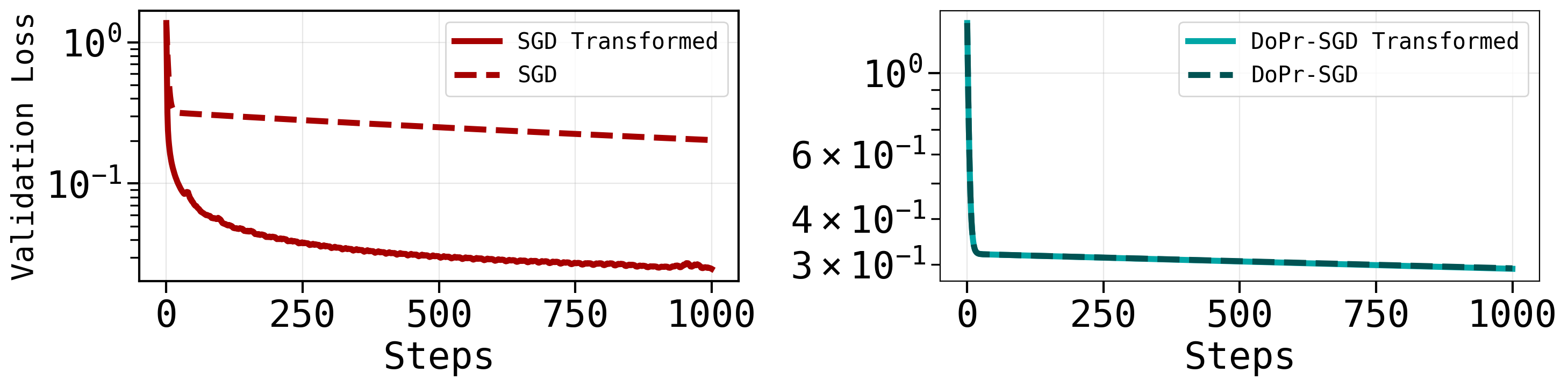}
    \caption{
      When an affine transform is applied to the input distribution, with the initial weights transformed accordingly \eqref{eq:affine transform}. The \SGD trajectories (\textbf{left}) diverge, while the \DoPe-\SGD trajectories (\textbf{right}) match exactly, demonstrating the invariance induced by \DoPe under affine transforms (\Cref{prop:precond invariance}). See \Cref{appdx:activation invariance} for experiment details.
    }
    \label{fig:AP invariance property}
  \end{center}
  \vspace{-0.7cm}
\end{figure}
}

\paragraph{Models and datasets.}
We use 3 layer MLPs trained on 2 synthetically generated datasets. For the first dataset, the input data is generated with an anisotropic covariance $\bSigma\in \R^{32\times 32}$, where $\bSigma$ has log space eigenvalues ranging from $1$ to $100$. Then, the output is generated as $y=\bM_2\sigma(\bM_1\bx)$, where $\bM_1\in \R^{256\times 32}$, $\sigma$ is the GeLU activation function, and $\bM_2\in \R^{10\times 256}$, for each input data $\bx$. $\bM_1$ is initialized with each element drawn from $\mathcal{N}(0, \frac{1}{32})$ and $\bM_2$ is initialized with each element drawn from $\mathcal{N}(0, \frac{1}{256})$.

The second dataset, denoted the transformed dataset, generates each sample as $\bx_{t}=\bQ\bx+\bd$, where $\bx$ is a sample from the original. Here, $Q\in \R^{32\times 32}$ is a symmetric matrix with linear space eigenvalues ranging from $1$ to $100$, and $\bd\in \R^{32}$ is a random vector with elements drawn from the standard normal. Then, $\by_{t}=\bM_2\sigma(\bM_1\bx_{t})$, using the same method as in the other dataset. Both datasets are split into a $0.8:0.2$ into training data and validation data. This induces an affine transform of the input activations.

We then initialize $4$ models, $2$ to be trained on the original dataset (one optimized with \SGD, and one optimized with \DoPe-\SGD), and $2$ to be trained on the transformed dataset (one optimized with \SGD, and one optimized with \DoPe-\SGD). The two models trained on the original dataset are initialized identically, and the two trained on the transformed dataset are also initialized identically. All models have an input layer of dimension $(32, 256)$, hidden layer $(256, 256)$, and output layer $(256, 10)$, and use \texttt{SiLU} activations.

For all models, the hidden layer and output layers are identical. For the original dataset models, let the input layer have weight $\bW$, and bias $\bb$. Then, for the transformed dataset models, the input layer is initialized with weight $\bW\bQ^{-1}$, and bias $\bb-\bW\bQ^{-1}\bd$.

We note here, inputting $\bx$ into the models trained on the original data yield $\bh=\bW\bx+\bb$. Then, $\bx_{t}=\bQ\bx+\bd$, and when passed through the transformed models the output will be \[\bh_{t}=\bW\bQ^{-1}(\bQ\bx+\bd)+\bb-\bW\bQ^{-1}\bd=\bW\bx+\bb\]
So, for all models, the pre-activations after the first layer will be identical, and thus at initialization, \emph{all models will have identical outputs}.

From here, both models are trained for $50$ epochs of $20$ steps per epoch, with batch size $2048$. At each step, the models are validated on the $0.2$ validation set for their respective datasets, and the results reported. The models are trained so that at each training step, each batch exactly corresponds (e.g. for a batch $\bX\in \R^{32\times 2048}$ of the Original data, the transformed batch will exactly be $\bX_{t}=\bQ\bX+\bb\mathbf{1}^T$).

For the models trained using \DoPe, damping is set at $10^{-8}\cdot \trace(\bSigma)$, where $\bSigma$ is the preconditioner.

The identical trajectories of both \DoPe models demonstrates the affine invariance demonstrated in \Cref{prop:precond invariance} -- when learning the same function, and with identical initialization, affine transforms to the input distribution do not affect the loss trajectory when trained with \DoPe. However, loss trajectories under vanilla SGD do exhibit a drift, demonstrating the coordinate-dependent trajectory as stated in \Cref{sec:AP general architectures}. 

\subsection{Hyperparameter Scaling \DoPe}\label{appdx:mup}

This section provides details on the experiments shown in \Cref{fig:mup_nanogpt_adamw}, \Cref{fig:mup_nanogpt_muon}, and \Cref{fig:mup_CIFAR}. The goals of these experiments are to provide guidelines for scaling hyperparameters as width increases, as discussed in \citet{yang2022tensor} and related work (see \Cref{appdx:extended related}). These scalings allow for zero-shot hyperparameter transfer across models with differing width.

\begin{figure}[ht]
  \begin{center}  
  \centerline{\includegraphics[width=\columnwidth]{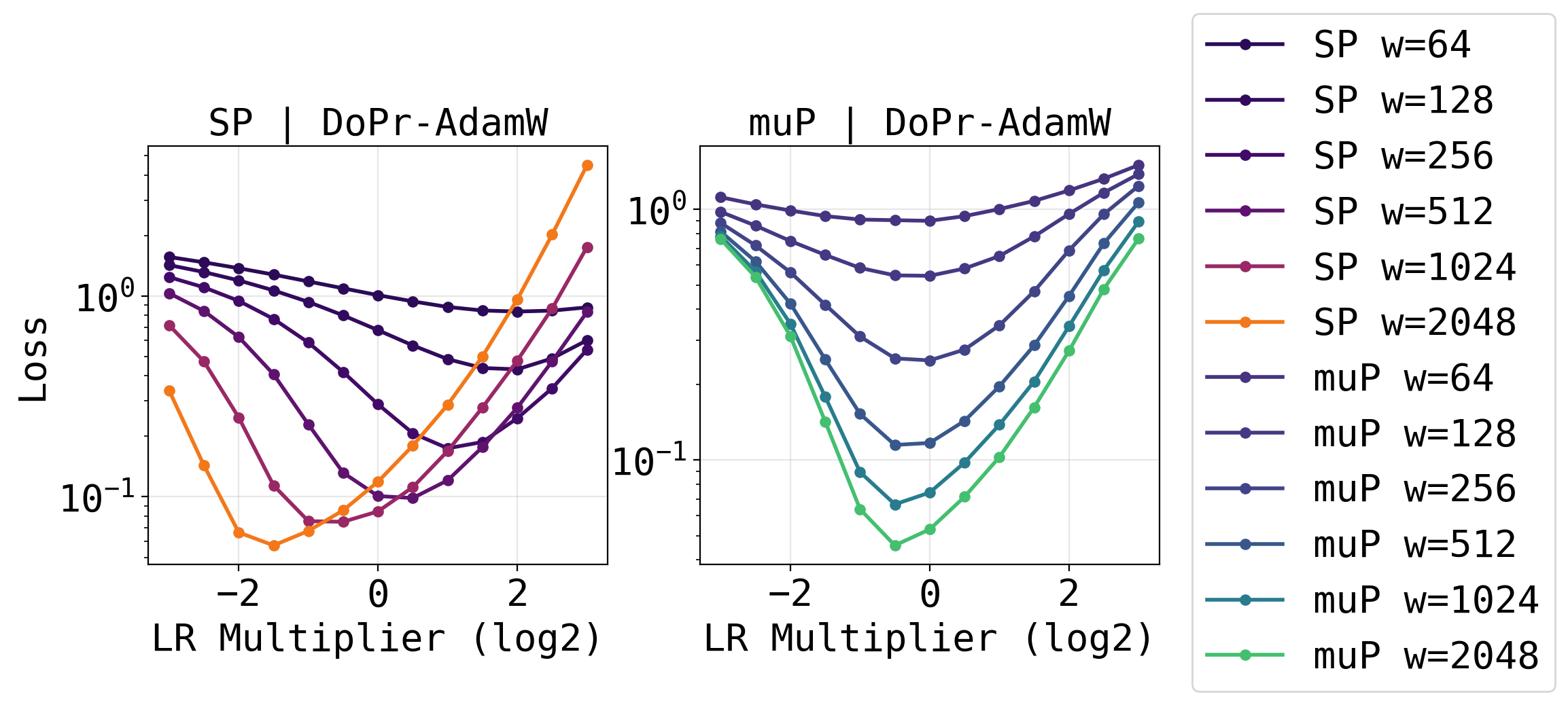}}
    \caption{
      We train a 3-layer MLP on CIFAR, and sweep over learning rates. We compare Standard Parametrization (SP, left) and Maximal Update parametrization ($\mu$P, right). The final train loss is reported for both $\mu$P and SP. $\mu$P achieves learning rate transfer as width increases.
    }
    \label{fig:mup_CIFAR}
  \end{center}
  \vspace{-1cm}
\end{figure}

\paragraph{Models and datasets.}
In \Cref{fig:mup_nanogpt_adamw} and \Cref{fig:mup_nanogpt_muon}, for language modeling, we use a GPT2-style Transformer with 4 layers, head dimension 64, sequence length 1024, and batch size 12, trained on the OpenWebText dataset. Each experiment processes approximately 100M training tokens.

For \Cref{fig:mup_CIFAR} and the weight-decay $\mu$P plots in \Cref{fig:mup_nanogpt_adamw}, we use 3 layer MLPs trained on CIFAR10 \cite{cifar10dataset}, normalizing input data to be between $[-1, 1]$. The MLP layers have no bias, and for width $w$, the layers are of size $(3072, w), (w,w)$, and $(w, 10)$, with cross entropy loss. For the learning rate sweeps (\Cref{fig:mup_CIFAR}), we use batch size $512$, and train for $8$ epochs. For the weight decay experiment, we use batch size $4096$ and train for $200$ epochs. 

\paragraph{Optimization.}
For the language modeling and CIFAR learning rate experiments we sweep over learning rate. For the language modeling we use base weight decay of $0.1$, and for CIFAR we use base weight decay of $0.01$. For language modeling, a base learning rate of $6\cdot 10^{-3}$ was used. For the Standard Parametrization (SP) sweep for CIFAR, we use base learning rate of $2^{-10}$, and for the $\mu$P sweep for CIFAR, we use base learning rate of $2^{-4}$. Then, the learning rate used is: 
\[\eta_{base} = \eta \cdot 2^k\]
where $k$ is the sweep multiplier. In all experiments, all layers are preconditioned, with damping that scales with $\trace(\bSigma)$, where $\bSigma$ is the preconditioner. The language model used damping $10^{-4}\cdot \trace(\bSigma)$, and the CIFAR experiments used damping $10^{-3}\cdot \trace(\bSigma)$.

Noting that the preconditioner $\bSigma$ is the same as the preconditioner in \KFAC, this damping aligns with the rescaled damping in \citet{ishikawa2023parameterization}, which naturally achieves the $\mu$P scaling with width. 
We also EMA the preconditioner with $\beta=0.9$, only in the CIFAR experiments. In language modeling, we report the best training loss achieved during training. For the CIFAR learning rate sweep, we report the training loss at the final epoch for the learning rate sweep. 

For the CIFAR weight decay experiment, we report $\frac{||\Delta \bW||_F}{||\bW||_F}$, the ratio of the update norm to the weight norm for the hidden layer $\bW$. We report this at each step in training. We use a base learning rate of $\eta_{base}=2^{-4}$, and a base weight decay of $\lambda_{base}=0.01$. 

\paragraph{$\mu$P vs SP.}
In all experiments, we follow the $\mu$P scaling rules as described in \citet{yang2022tensor}. As further explained below, we note that for \DoPe-\AdamW and \DoPe-\Muon in particular, the same scaling rules apply as for \AdamW and \Muon. This is because of the normalizing nature of these optimizers - despite \DoPe altering the gradient that \AdamW and \Muon act on, the final object still has the same normalized size.

For the $\mu$P sweeps, the input layers, are initialized as $\bW_{ij}\sim \mathcal{N}(0, d_{in}^{-1})$. The hidden layers are initialized as $W_{ij}\sim \mathcal{N}(0, w^{^{-1}})$, and the output layers are initialized as $\bW_{ij}\sim \mathcal{N}(0, w^{-2})$. The input layer learning rate is $\eta_{base}$, the hidden layer learning rate is $\eta_{base}\cdot w^{-1}$, and the output layer learning rate is $\eta_{base}\cdot w^{-1}$. In $\mu$P, during attention instead of computing $\frac{q^Tk}{\sqrt{d}}$, $\frac{q^Tk}{d}$ is used in attention. 

Given a base weight decay of $\lambda_{base}=0.01$, the input layer's weight decay is $\lambda_{base}$, the hidden layers' weight decay is $\lambda_{base}\cdot w$, and the output layer's weight decay is $\lambda_{base}\cdot w$. This style of weight decay keeps $\eta \cdot \lambda$ width-independent has been proposed by several past work \citep{kosson2024rotationalequilibriumweightdecay, kosson2025weightdecaymattermup, wang2025setadamsweightdecay, bergsma2025powerlinesscalinglaws}.  Other works, such as \citet{qiu2025hyperparameter, xiao2025rethinkingconventionalwisdommachine}, propose other methods of scaling weight decay, such as scaling weight decay with $\frac{1}{w}$. We find that the scaling of weight decay with width allows the relative update size to remain invariant towards the infinite width limit. 

For SP in the learning rate sweeps, all layers are set at their PyTorch defaults. In particular, the input layers are initialized at $\bW_{ij}\sim\mathcal{N}(0, d_{in}^{-1})$. The hidden and output layers are initialized as $\bW_{ij}\sim\mathcal{N}(0, w^{-1})$. All layers have constant learning rate, and constant weight decay. 

For the weight decay experiments, for the SP run we keep the $\mu$P initialization and learning rate scaling, and use a constant weight decay for all layers. 
\paragraph{Justification of \DoPe-\AdamW and \DoPe-\Muon $\mu$P scaling}
Here we provide explanation for why \DoPe does not affect the $\mu$P scaling rules for \AdamW and \Muon. Intuitively, the reason for this is because \AdamW and \Muon can both be seen as normalizing the gradient object - so regardless of whether it is applied to the raw gradient, or the \DoPe-gradient, the result will still have similar magnitude elements. 

Let the weight in question be $\bW$ at initialization, with raw gradient $\bG$, and preconditioner $\bSigma$. For simplicity we consider batch size of $1$, let $\bW$ be square, and let $\bz\in \R^{w}$ be the incoming data that the optimizer updates on. We assume $\bz$ has $\Theta(1)$ entries with respect to width. We examine the one-step update.

The preconditioned gradient is thus, for a damping constant $\rho$ \[\bM=\bG(\bSigma+\rho\cdot \trace(\bSigma) \bI)^{-1}\]
\[\bD=\Muon(\bM)\]
And so, the update is 
\[\bW_+=\bW-\eta \bD\]

We note that \Muon orthonormalizes the gradient matrix (in this case $\bM$), up to numerical errors in the Newton-Schulz iterations. Thus, we may treat $\bD$ as orthonormal. Then, we get:
\[\bW_+\bz=\bW\bz - \eta \bD\bz\]
The preconditioning only affects the $\bD\bz$ term, so we focus on this term. Since $\bD$ is approximately orthonormal, and $\bz$ has $\Theta(1)$ entries, $\bD\bz$ will also have $\Theta(1)$ entries, as $\bD$ is orthonormal (for square $\bM$). Thus, \DoPe-\Muon achieves learning rate transfer across width with constant learning rate, matching \Muon. 

In the case of \AdamW, we consider 
\[\bW_+=\bW-\eta \AdamW(\bM)\]
We consider \AdamW without the EMA buffers, since momentum does not affect $\mu$P \citep{qiu2025hyperparameter}. In this case, \AdamW is simply an elementwise sign function. We can also note that $\bG$ can be written as $\delta\bz^\top$, where $\delta$ is the gradient of the loss with respect to $\bW\bz$. Similarly, we have $\bSigma=\bz\bz^\top$. Thus, we see \[\bM=(\delta\bz^\top)(\bz\bz^\top+\rho\cdot \trace(\bz\bz^\top)\bI)^{-1}\]
Thus, applying an elementwise sign function to this will create a matrix $\bD=\AdamW(\bM)$ with all elements $\Theta(1)$, but correlated with $\bz$. Thus, we get
\[\bW_+=\bW\bz-\eta\bD\bz\]
Again, the first term is the same as in regular \AdamW. In the second term, as $\bD$ and $\bz$ are correlated, by the Law of Large Numbers the resulting terms will be $\Theta(w)$, and so $\eta$ should be scaled with $w^{-1}$, exactly matching the learning rate scaling for regular \AdamW. Notice in this case, $w$ is just the dimension of $\bz$. 

So, in either case, \DoPe-\AdamW and \DoPe-\Muon retains the same $\mu$P scaling rules as \AdamW and \Muon themselves.

\begin{figure}[t]
    \centering
    \includegraphics[width=\linewidth]{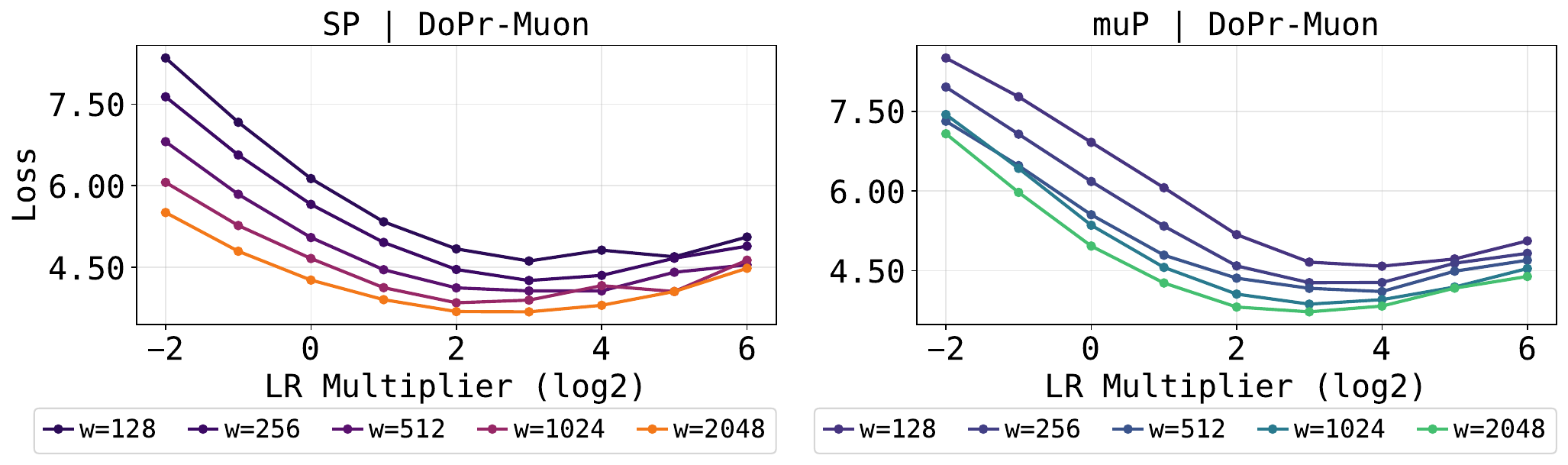}
    \caption{
\textbf{$\mu$P scaling behavior.}
Learning-rate sweeps comparing for \DoPe-\Muon under non-$\mu$P(SP) and $\mu$P parameterizations on GPT2 language model.
}
    \label{fig:mup_nanogpt_muon}
\end{figure}

\subsection{State-based Imitation Learning}
\label{app:state-based-IL}

This section provides implementation details and supporting results for \texttt{Humanoid-v5} and \texttt{Half-Cheetah-v5} Gymnasium benchmarks \citep{towers2025gymnasiumstandardinterfacereinforcement}

\textbf{Model and architecture.}  
To maintain consistency with prior work, we use the Residual MLP architecture together with model EMA and a cosine learning rate scheduler following \cite{block2024butterfly}. Full architectural hyperparameters are reported in \Cref{tab:architecture}.

\textbf{Training setup.}  
For each optimizer variant, we conduct grid searches over learning rate, damping, and weight decay. The best-performing hyperparameter configurations are reported in \Cref{tab:humanoid_consolidated} for \texttt{Humanoid-v5} and \Cref{tab:halfcheetah_consolidated} for \texttt{Half-Cheetah-v5}. We train three independent seeds over 1000 training steps.

\textbf{Evaluation Protocol.}  
We evaluate model performance using episodic return on the corresponding control task. For each seed, we use a bootstrapped-mean evaluation procedure described in \Cref{alg:bootstrap} and report the 10/50/90 quantiles.  

\begin{algorithm}[t]
\label{alg:bootstrap}
\caption{Bootstrap Mean Estimation of Success Rate}
\begin{algorithmic}[1]
\STATE \textbf{Input:} checkpoints $\{c_k\}_{k=1}^K$, rollouts per checkpoint $R$, bootstrap trials $B$, sample size $m$
\STATE $\mathcal{D} \gets [\,]$

\FOR{$k=1,\dots,K$}
    \STATE Load checkpoint $c_k$
    \FOR{$r=1,\dots,R$}
        \STATE Run one rollout episode
        \STATE Append outcome $x \in \{0,1\}$ to $\mathcal{D}$ \COMMENT{$1=$ success}
    \ENDFOR
\ENDFOR

\FOR{$b=1,\dots,B$}
    \STATE Draw subset $\mathcal{S}_b \subset \mathcal{D}$ of size $m$ (without replacement)
    \STATE $\hat{\mu}_b \gets \frac{1}{m}\sum_{x \in \mathcal{S}_b} x$
\ENDFOR

\STATE \textbf{Return:} 
$\mathrm{mean}(\hat{\mu}),\;
q_{0.10}, q_{0.50}, q_{0.90}$

\end{algorithmic}
\end{algorithm}

\textbf{Results.}  
Across both \texttt{Humanoid-v5} (\Cref{fig:humanoid_il}) and \texttt{Half-Cheetah-v5} (\Cref{fig:halfcheetah_il}), \DoPe{} variants consistently improve reward relative to their baseline optimizer counterparts. These gains do not reliably track training or validation loss trends.

\begin{figure}[t]
    \centering
    \includegraphics[width=1\linewidth]{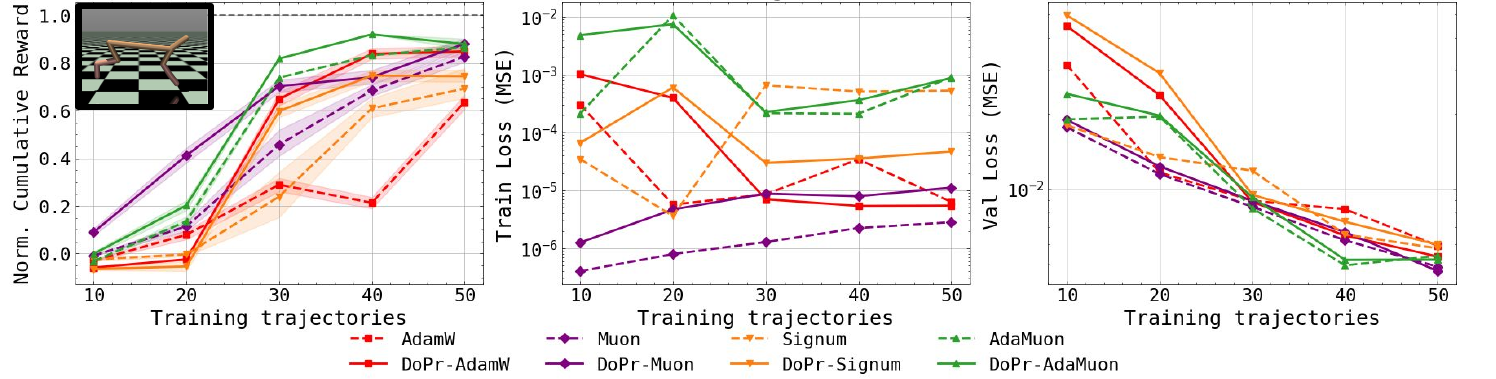}
    \caption{\textbf{Half-Cheetah-v5} \DoPe performance across \AdamW, \Muon, \Signum, and \AdaMuon. \DoPe variants attain higher terminal reward which does not consistently correlate with train or validation loss improvements.}
    \label{fig:halfcheetah_il}
\end{figure}

\begin{table}[h]
\centering
\caption{State-based IL Network Architecture and Training Configuration}
\label{tab:architecture}
\begin{tabular}{ll}
\toprule
\textbf{Parameter} & \textbf{Value} \\
\midrule
Hidden layers & [256, 256] \\
Activation & ReLU \\
Batch size & 256 \\
Training steps & 1000 \\
Eval trajectories & 100 \\
LR scheduler & Cosine annealing \\
Warmup fraction & 0.05 \\
Min LR ($\eta_{min}$) & 0.0 \\
EMA decay & 0.99 \\
EMA update freq & Every step \\
\bottomrule
\end{tabular}
\end{table}

\begin{table}[h]
\centering
\caption{Best optimizer hyperparameters across trajectory counts for Humanoid.}
\label{tab:humanoid_consolidated}
\small
\setlength{\tabcolsep}{4pt}
\renewcommand{\arraystretch}{1.25}
\begin{tabular}{lccccc}
\toprule
\textbf{Method} & \textbf{10 traj} & \textbf{20 traj} & \textbf{30 traj} & \textbf{40 traj} & \textbf{50 traj} \\
\midrule
AdamW
& \makecell{$\eta=3{\times}10^{-3}$ \\ $\lambda=1{\times}10^{-2}$}
& \makecell{$\eta=3{\times}10^{-3}$ \\ $\lambda=1{\times}10^{-2}$}
& \makecell{$\eta=3{\times}10^{-3}$ \\ $\lambda=0$}
& \makecell{$\eta=3{\times}10^{-3}$ \\ $\lambda=0$}
& \makecell{$\eta=1{\times}10^{-3}$ \\ $\lambda=0$} \\

\DoPe-AdamW
& \makecell{$\eta=1{\times}10^{-3}$ \\ $\lambda=0$ \\ $d=1{\times}10^{-4}$}
& \makecell{$\eta=1{\times}10^{-3}$ \\ $\lambda=0$ \\ $d=1{\times}10^{-3}$}
& \makecell{$\eta=1{\times}10^{-3}$ \\ $\lambda=0$ \\ $d=1{\times}10^{-4}$}
& \makecell{$\eta=1{\times}10^{-3}$ \\ $\lambda=0$ \\ $d=1{\times}10^{-3}$}
& \makecell{$\eta=1{\times}10^{-3}$ \\ $\lambda=0$ \\ $d=1{\times}10^{-5}$} \\

\midrule
Muon
& \makecell{$\eta_h=1{\times}10^{-2}$ \\ $\lambda=1{\times}10^{-2}$}
& \makecell{$\eta_h=2{\times}10^{-2}$ \\ $\lambda=5{\times}10^{-3}$}
& \makecell{$\eta_h=1{\times}10^{-2}$ \\ $\lambda=1{\times}10^{-2}$}
& \makecell{$\eta_h=2{\times}10^{-2}$ \\ $\lambda=1{\times}10^{-2}$}
& \makecell{$\eta_h=1{\times}10^{-2}$ \\ $\lambda=5{\times}10^{-3}$} \\

\DoPe-Muon
& \makecell{$\eta_h=2{\times}10^{-2}$ \\ $\lambda=1{\times}10^{-2}$ \\ $d=1{\times}10^{-4}$}
& \makecell{$\eta_h=2{\times}10^{-2}$ \\ $\lambda=5{\times}10^{-3}$ \\ $d=1{\times}10^{-4}$}
& \makecell{$\eta_h=2{\times}10^{-2}$ \\ $\lambda=1{\times}10^{-2}$ \\ $d=1{\times}10^{-4}$}
& \makecell{$\eta_h=2{\times}10^{-2}$ \\ $\lambda=1{\times}10^{-2}$ \\ $d=1{\times}10^{-5}$}
& \makecell{$\eta_h=2{\times}10^{-2}$ \\ $\lambda=1{\times}10^{-2}$ \\ $d=1{\times}10^{-5}$} \\

\midrule
Signum
& \makecell{$\eta=3{\times}10^{-4}$ \\ $\lambda=0$}
& \makecell{$\eta=1{\times}10^{-3}$ \\ $\lambda=0$}
& \makecell{$\eta=1{\times}10^{-3}$ \\ $\lambda=0$}
& \makecell{$\eta=3{\times}10^{-4}$ \\ $\lambda=0$}
& \makecell{$\eta=3{\times}10^{-4}$ \\ $\lambda=0$} \\

\DoPe-Signum
& \makecell{$\eta=1{\times}10^{-3}$ \\ $\lambda=0$ \\ $d=1{\times}10^{-3}$}
& \makecell{$\eta=3{\times}10^{-4}$ \\ $\lambda=0$ \\ $d=1{\times}10^{-4}$}
& \makecell{$\eta=3{\times}10^{-4}$ \\ $\lambda=0$ \\ $d=1{\times}10^{-5}$}
& \makecell{$\eta=3{\times}10^{-4}$ \\ $\lambda=0$ \\ $d=1{\times}10^{-4}$}
& \makecell{$\eta=3{\times}10^{-4}$ \\ $\lambda=0$ \\ $d=1{\times}10^{-5}$} \\

\midrule
AdaMuon
& \makecell{$\eta=1{\times}10^{-3}$ \\ $\lambda=1{\times}10^{-1}$}
& \makecell{$\eta=3{\times}10^{-3}$ \\ $\lambda=1{\times}10^{-2}$}
& \makecell{$\eta=1{\times}10^{-3}$ \\ $\lambda=1{\times}10^{-2}$}
& \makecell{$\eta=3{\times}10^{-3}$ \\ $\lambda=1{\times}10^{-1}$}
& \makecell{$\eta=3{\times}10^{-3}$ \\ $\lambda=1{\times}10^{-1}$} \\

\DoPe-AdaMuon
& \makecell{$\eta=3{\times}10^{-3}$ \\ $\lambda=1{\times}10^{-2}$ \\ $d=1{\times}10^{-2}$}
& \makecell{$\eta=3{\times}10^{-3}$ \\ $\lambda=1{\times}10^{-1}$ \\ $d=1{\times}10^{-2}$}
& \makecell{$\eta=1{\times}10^{-3}$ \\ $\lambda=1{\times}10^{-1}$ \\ $d=1{\times}10^{-4}$}
& \makecell{$\eta=1{\times}10^{-3}$ \\ $\lambda=1{\times}10^{-2}$ \\ $d=1{\times}10^{-4}$}
& \makecell{$\eta=3{\times}10^{-3}$ \\ $\lambda=1{\times}10^{-2}$ \\ $d=1{\times}10^{-3}$} \\
\bottomrule
\end{tabular}%
\end{table}

\begin{table}[h]
\centering
\caption{Best optimizer hyperparameters across trajectory counts for HalfCheetah.}
\label{tab:halfcheetah_consolidated}
\small
\setlength{\tabcolsep}{4pt}
\renewcommand{\arraystretch}{1.25}
\small
\begin{tabular}{lccccc}
\toprule
\textbf{Method} & \textbf{10 traj} & \textbf{20 traj} & \textbf{30 traj} & \textbf{40 traj} & \textbf{50 traj} \\
\midrule
\AdamW
& \makecell{$\eta=2{\times}10^{-3}$ \\ $\lambda=3{\times}10^{-3}$}
& \makecell{$\eta=2{\times}10^{-4}$ \\ $\lambda=0$}
& \makecell{$\eta=3{\times}10^{-3}$ \\ $\lambda=1{\times}10^{-2}$}
& \makecell{$\eta=2{\times}10^{-3}$ \\ $\lambda=3{\times}10^{-2}$}
& \makecell{$\eta=2{\times}10^{-3}$ \\ $\lambda=3{\times}10^{-2}$} \\

\DoPe-\AdamW
& \makecell{$\eta=1{\times}10^{-4}$ \\ $\lambda=0$ \\ $d=1{\times}10^{-5}$}
& \makecell{$\eta=3{\times}10^{-4}$ \\ $\lambda=0$ \\ $d=1{\times}10^{-3}$}
& \makecell{$\eta=2{\times}10^{-3}$ \\ $\lambda=0$ \\ $d=1{\times}10^{-3}$}
& \makecell{$\eta=2{\times}10^{-3}$ \\ $\lambda=1{\times}10^{-2}$ \\ $d=1{\times}10^{-6}$}
& \makecell{$\eta=2{\times}10^{-3}$ \\ $\lambda=1{\times}10^{-2}$ \\ $d=1{\times}10^{-3}$} \\

\midrule
\Muon
& \makecell{$\eta_h=2{\times}10^{-2}$ \\ $\lambda=5{\times}10^{-3}$}
& \makecell{$\eta_h=2{\times}10^{-2}$ \\ $\lambda=5{\times}10^{-3}$}
& \makecell{$\eta_h=3{\times}10^{-2}$ \\ $\lambda=5{\times}10^{-3}$}
& \makecell{$\eta_h=3{\times}10^{-2}$ \\ $\lambda=1{\times}10^{-2}$}
& \makecell{$\eta_h=3{\times}10^{-2}$ \\ $\lambda=2{\times}10^{-2}$} \\

\DoPe-\Muon
& \makecell{$\eta_h=5{\times}10^{-3}$ \\ $\lambda=5{\times}10^{-3}$ \\ $d=3{\times}10^{-4}$}
& \makecell{$\eta_h=3{\times}10^{-2}$ \\ $\lambda=2{\times}10^{-2}$ \\ $d=1{\times}10^{-3}$}
& \makecell{$\eta_h=3{\times}10^{-2}$ \\ $\lambda=2{\times}10^{-2}$ \\ $d=1{\times}10^{-3}$}
& \makecell{$\eta_h=3{\times}10^{-2}$ \\ $\lambda=5{\times}10^{-3}$ \\ $d=1{\times}10^{-3}$}
& \makecell{$\eta_h=3{\times}10^{-2}$ \\ $\lambda=2{\times}10^{-2}$ \\ $d=1{\times}10^{-3}$} \\

\midrule
\Signum
& \makecell{$\eta=1{\times}10^{-3}$ \\ $\lambda=0$}
& \makecell{$\eta=3{\times}10^{-4}$ \\ $\lambda=0$}
& \makecell{$\eta=3{\times}10^{-3}$ \\ $\lambda=0$}
& \makecell{$\eta=3{\times}10^{-3}$ \\ $\lambda=0$}
& \makecell{$\eta=3{\times}10^{-3}$ \\ $\lambda=0$} \\

\DoPe-\Signum
& \makecell{$\eta=1{\times}10^{-4}$ \\ $\lambda=0$ \\ $d=1{\times}10^{-4}$}
& \makecell{$\eta=1{\times}10^{-3}$ \\ $\lambda=0$ \\ $d=1{\times}10^{-5}$}
& \makecell{$\eta=1{\times}10^{-3}$ \\ $\lambda=0$ \\ $d=1{\times}10^{-3}$}
& \makecell{$\eta=1{\times}10^{-3}$ \\ $\lambda=0$ \\ $d=1{\times}10^{-3}$}
& \makecell{$\eta=1{\times}10^{-3}$ \\ $\lambda=0$ \\ $d=1{\times}10^{-3}$} \\

\midrule
\AdaMuon
& \makecell{$\eta=1{\times}10^{-2}$ \\ $\lambda=1{\times}10^{-2}$}
& \makecell{$\eta=1{\times}10^{-2}$ \\ $\lambda=1{\times}10^{-2}$}
& \makecell{$\eta=1{\times}10^{-2}$ \\ $\lambda=1{\times}10^{-2}$}
& \makecell{$\eta=1{\times}10^{-2}$ \\ $\lambda=1{\times}10^{-2}$}
& \makecell{$\eta=3{\times}10^{-3}$ \\ $\lambda=1{\times}10^{-2}$} \\

\DoPe-\AdaMuon
& \makecell{$\eta=1{\times}10^{-2}$ \\ $\lambda=1{\times}10^{-2}$ \\ $d=1{\times}10^{-3}$}
& \makecell{$\eta=1{\times}10^{-2}$ \\ $\lambda=1{\times}10^{-2}$ \\ $d=1{\times}10^{-2}$}
& \makecell{$\eta=1{\times}10^{-3}$ \\ $\lambda=1{\times}10^{-2}$ \\ $d=1{\times}10^{-3}$}
& \makecell{$\eta=1{\times}10^{-2}$ \\ $\lambda=1{\times}10^{-2}$ \\ $d=1{\times}10^{-2}$}
& \makecell{$\eta=3{\times}10^{-3}$ \\ $\lambda=1{\times}10^{-2}$ \\ $d=1{\times}10^{-3}$} \\
\bottomrule
\end{tabular}%
\end{table}

\subsection{Image-Based Robot Policy Learning}
\label{app:image-based-il}
This section provides additional implementation details and supporting results for the \texttt{Tool-Hang} and \texttt{Transport} Robomimic tasks \citep{robomimic2021}

\textbf{Model and architecture.}  
We evaluate \DoPe performance on state-of-the-art imitation learning architectures with generative control policies (GCPs) using the image-based \texttt{Tool Hang} and \texttt{Transport} tasks at Proficient Human (PH) expert level from Robomimic \citep{robomimic2021}. We use the training pipeline of \cite{pan2025much}, inheriting all architectural and environment hyperparameters. Our primary policy architecture is ChiUNet \citep{chi2023diffusion}, with approximately 20M parameters.

\textbf{Training setup.}  
We conduct learning rate sweeps over \AdamW, \Muon, and their corresponding \DoPe variants. We precondition all layers (including FiLM) with full preconditioning (via Cholesky) except for Conv2d layers which use SUA \cref{app:sua}. For each method, we tune learning rates while keeping the remaining training configuration fixed to the reference implementation. Full optimizer hyperparameters are reported in \Cref{tab:toolhang_hparams} and \Cref{tab:transport_hparams}.

\begin{table*}[t]
\centering
\small
\caption{Best optimizer hyperparameters for Tool Hang PH Image using ChiUNet (seeds 42, 43, 44).}
\label{tab:toolhang_hparams}
\begin{tabular}{lcccc}
\toprule
Hyperparameter & AdamW & \DoPe-AdamW & Muon & \DoPe-Muon \\
\midrule
lr                   & $1\mathrm{e}{-4}$ & $1\mathrm{e}{-4}$ & $3\mathrm{e}{-4}$ & $3\mathrm{e}{-4}$ \\
weight\_decay        & $1\mathrm{e}{-5}$ & $1\mathrm{e}{-5}$ & $0.1$ & $0.1$ \\
muon\_momentum       & --- & --- & $0.95$ & $0.95$ \\
grad\_clip\_norm     & $10.0$ & $10.0$ & $1.0$ & $1.0$ \\
gradient\_steps      & $300$k & $300$k & $200$k & $200$k \\
batch\_size          & $256$ & $256$ & $256$ & $256$ \\
precond\_ema\_rate   & --- & $0.9$ & --- & $0.9$ \\
grad\_ema\_rate      & --- & $0.9$ & --- & $0.9$ \\
precond\_include\_film & --- & true & --- & true \\
damping              & --- & $1\mathrm{e}{-4}$ & --- & $1\mathrm{e}{-4}$ \\
damping\_mode        & --- & trace & --- & trace \\
\bottomrule
\end{tabular}
\end{table*}

\begin{table*}[t]
\centering
\small
\caption{Best optimizer hyperparameters for Transport MH Image using ChiUNet (seeds 42, 43, 44).}
\label{tab:transport_hparams}
\begin{tabular}{lcccc}
\toprule
Hyperparameter & AdamW & \DoPe-AdamW & Muon & \DoPe-Muon \\
\midrule
lr                   & $1\mathrm{e}{-4}$ & $1\mathrm{e}{-4}$ & $1\mathrm{e}{-4}$ & $1\mathrm{e}{-4}$ \\
weight\_decay        & $1\mathrm{e}{-5}$ & $1\mathrm{e}{-5}$ & $0.1$ & $0.1$ \\
muon\_momentum       & --- & --- & $0.95$ & $0.95$ \\
grad\_clip\_norm     & $10.0$ & $10.0$ & $1.0$ & $1.0$ \\
gradient\_steps      & $300$k & $300$k & $200$k & $200$k \\
batch\_size          & $256$ & $256$ & $256$ & $256$ \\
precond\_ema\_rate   & --- & $0.0$ & --- & $0.0$ \\
grad\_ema\_rate      & --- & $0.0$ & --- & $0.0$ \\
precond\_include\_film & --- & true & --- & true \\
damping              & --- & $1\mathrm{e}{-4}$ & --- & $1\mathrm{e}{-4}$ \\
damping\_mode        & --- & trace & --- & trace \\
\bottomrule
\end{tabular}
\end{table*}

\textbf{Evaluation Protocol.}  
Policies are evaluated using task success rate under the standard Robomimic benchmark protocol. During training, we evaluate checkpoints every 20k steps using 9-step action sampling and report the corresponding rollout success rate using the bootstrapped mean procedure \Cref{alg:bootstrap}.

\textbf{Results.}  
Results are summarized in \Cref{fig:robomimic_merged_new.pdf}. Across both tasks, \DoPe{} variants consistently match or exceed their baseline counterparts, demonstrating that the proposed preconditioning strategy transfers effectively to high-dimensional visuomotor control.

\begin{figure*}[t]
\centering

\begin{subfigure}[t]{0.48\textwidth}
    \centering
    \includegraphics[width=\linewidth]{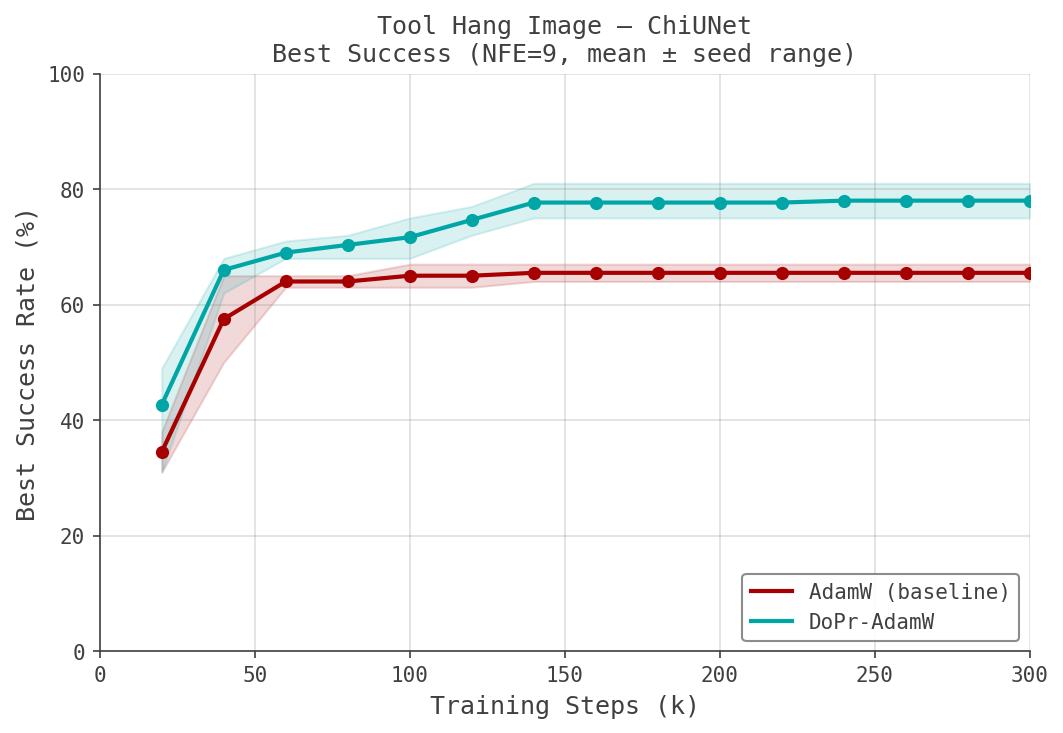}
    \caption{Tool Hang: AdamW vs \DoPe-AdamW}
\end{subfigure}
\hfill
\begin{subfigure}[t]{0.48\textwidth}
    \centering
    \includegraphics[width=\linewidth]{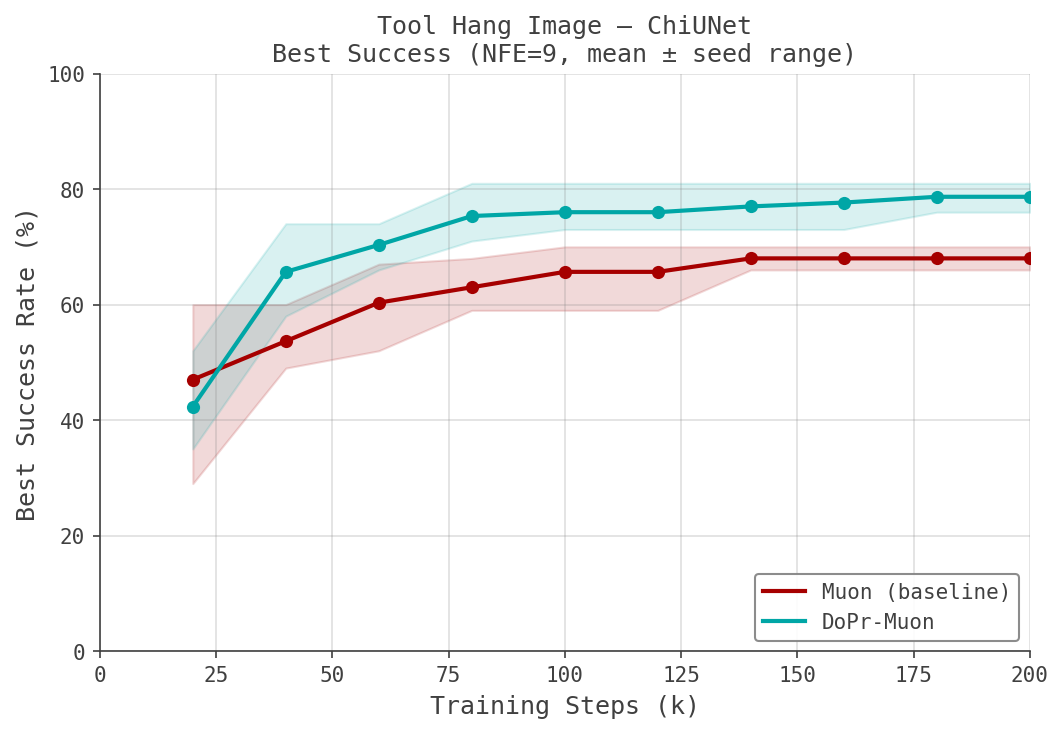}
    \caption{Tool Hang: Muon vs \DoPe-Muon}
\end{subfigure}

\vspace{0.75em}

\begin{subfigure}[t]{0.48\textwidth}
    \centering
    \includegraphics[width=\linewidth]{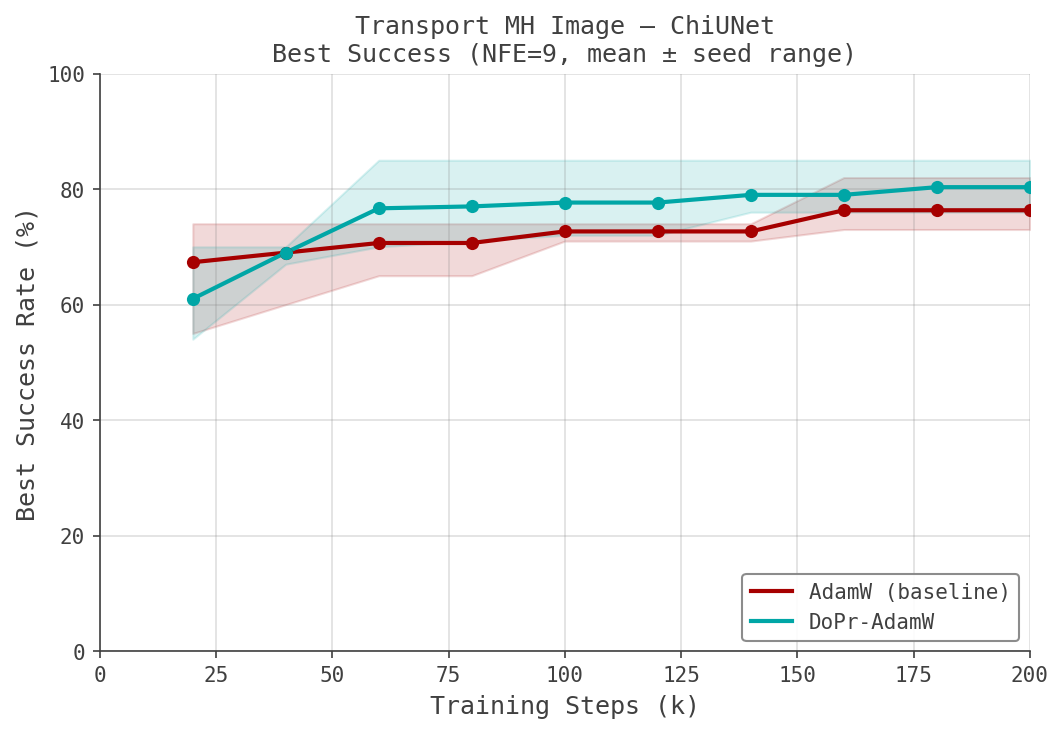}
    \caption{Transport: AdamW vs \DoPe-AdamW}
\end{subfigure}
\hfill
\begin{subfigure}[t]{0.48\textwidth}
    \centering
    \includegraphics[width=\linewidth]{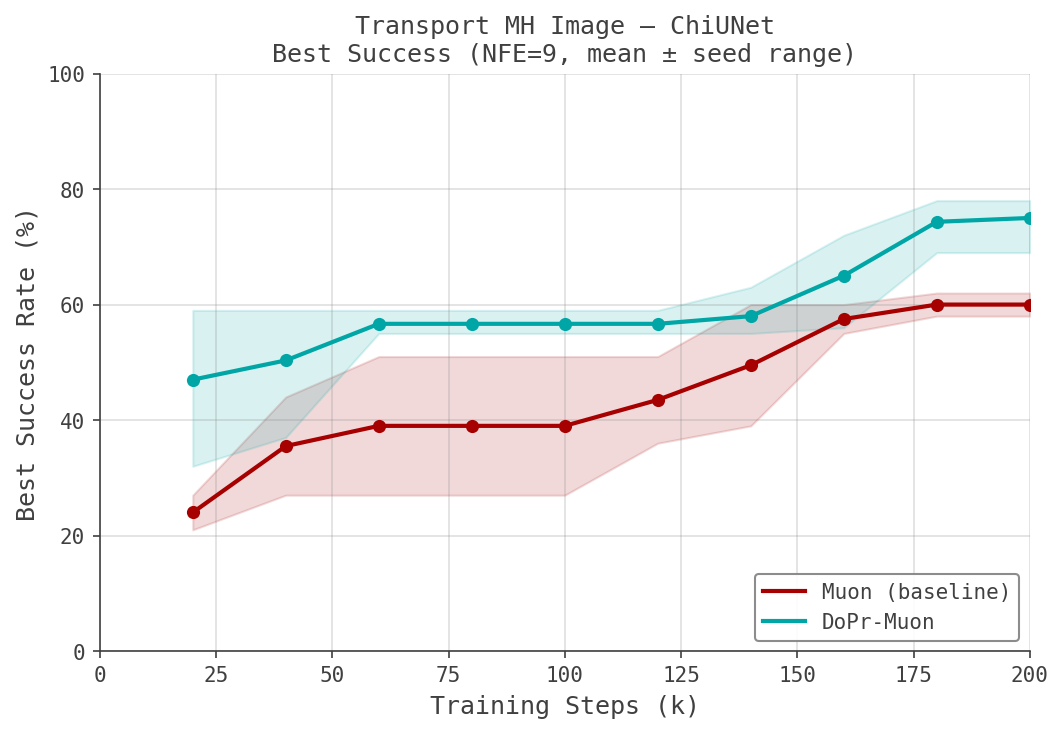}
    \caption{Transport: Muon vs \DoPe-Muon}
\end{subfigure}

\caption{Best mean success rate throughout training on \texttt{Tool-Hang} and \texttt{Transport} tasks from \texttt{robomimic}. We compare AdamW, Muon, and their \DoPe{} variants, reporting the best mean success rate over evaluation checkpoints every 20k steps.}
\label{fig:robomimic_image_based}

\end{figure*}

\begin{figure*}[t]
\centering

\begin{subfigure}[t]{0.48\textwidth}
    \centering
    \includegraphics[width=\linewidth]{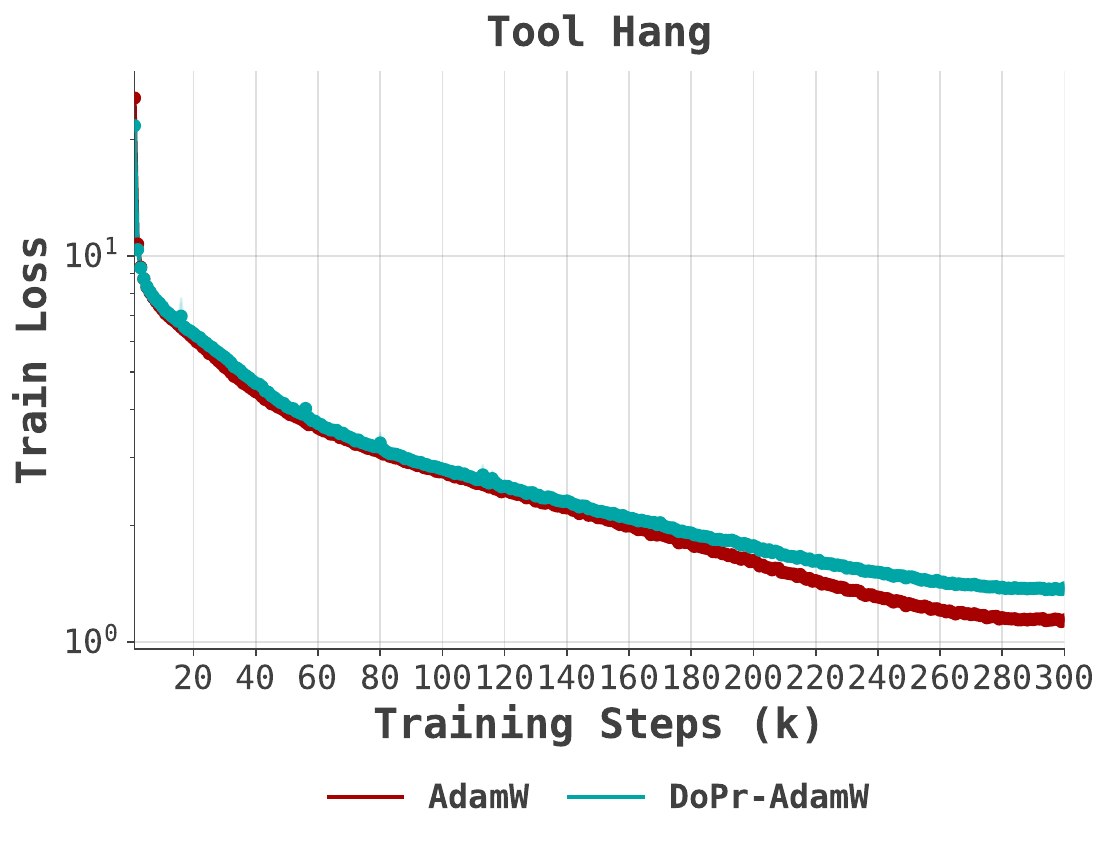}
    \caption{Tool Hang: AdamW vs \DoPe-AdamW}
\end{subfigure}
\hfill
\begin{subfigure}[t]{0.48\textwidth}
    \centering
    \includegraphics[width=\linewidth]{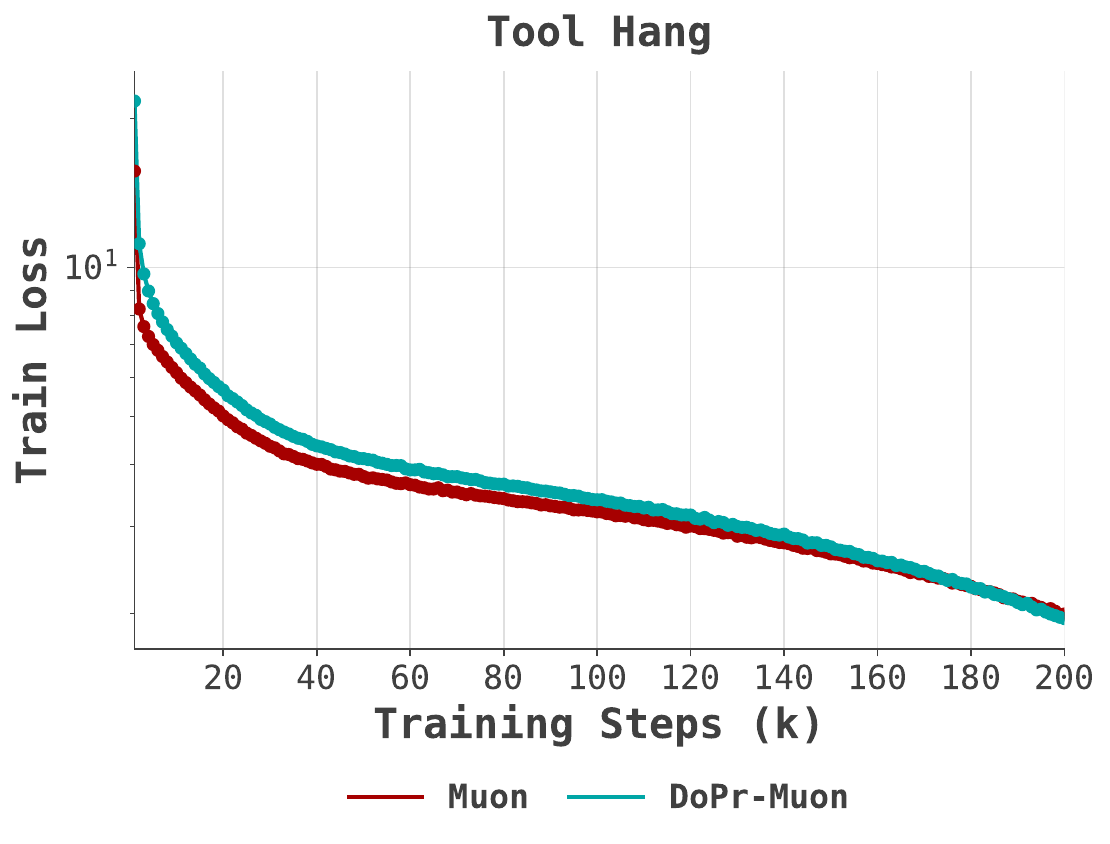}
    \caption{Tool Hang: Muon vs \DoPe-Muon}
\end{subfigure}

\vspace{0.75em}

\begin{subfigure}[t]{0.48\textwidth}
    \centering
    \includegraphics[width=\linewidth]{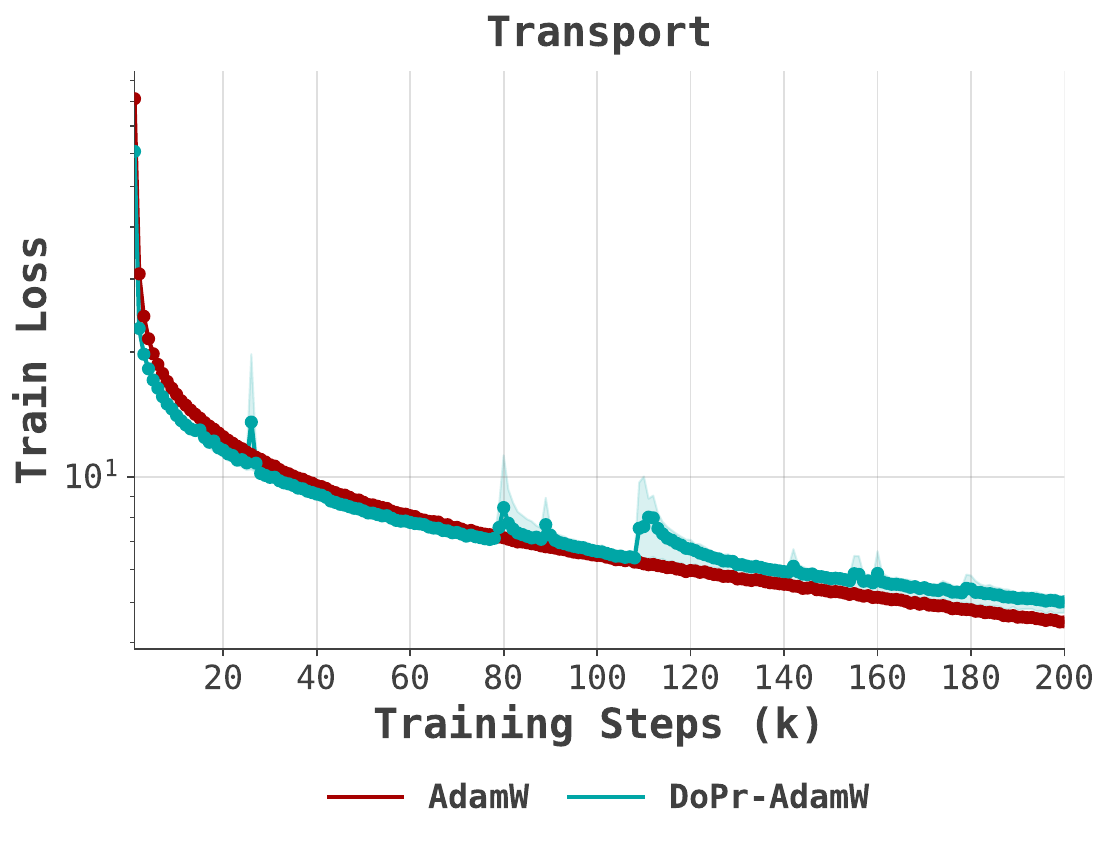}
    \caption{Transport: AdamW vs \DoPe-AdamW}
\end{subfigure}
\hfill
\begin{subfigure}[t]{0.48\textwidth}
    \centering
    \includegraphics[width=\linewidth]{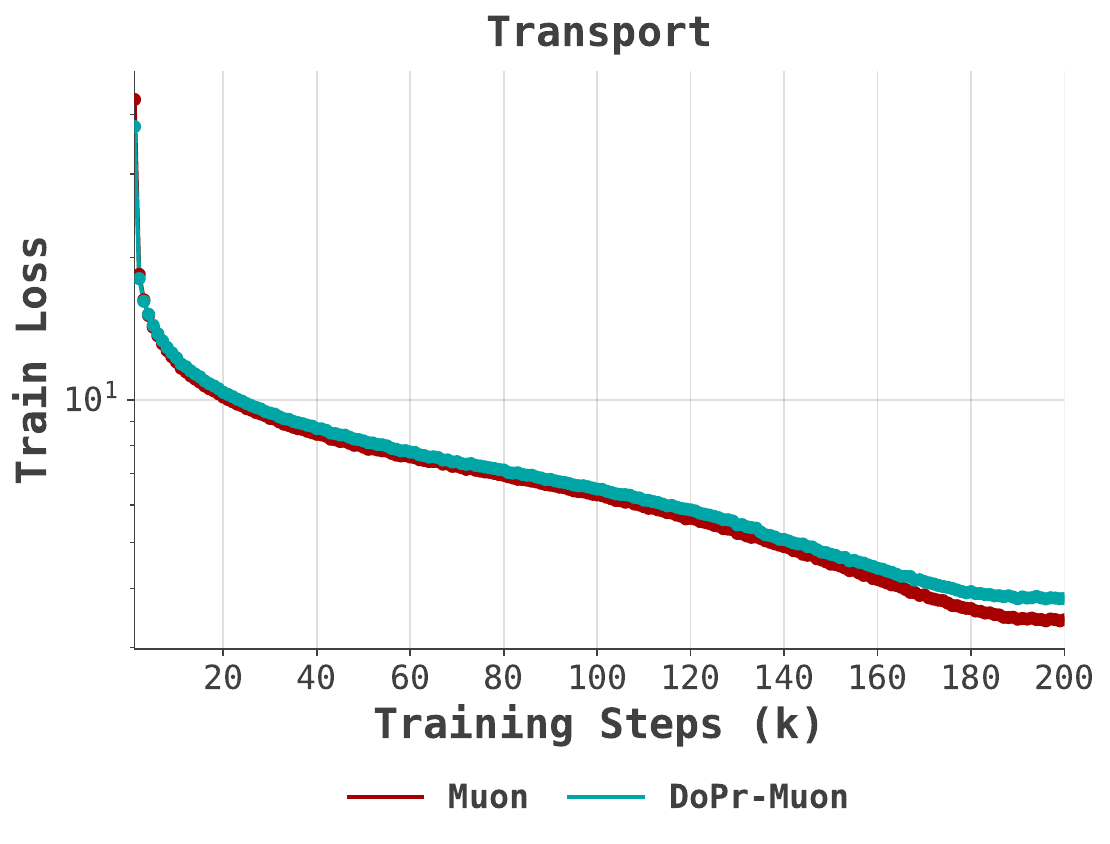}
    \caption{Transport: Muon vs \DoPe-Muon}
\end{subfigure}
\caption{Training Loss on \texttt{Tool-Hang} and \texttt{Transport} tasks from \texttt{robomimic}. We compare AdamW, Muon, and their \DoPe{} variants}
\label{fig:robomimic_image_based_train_loss}
\end{figure*}

\subsection{Supervised Fine-Tuning}
\label{app:sft_details}

This section provides full details for the supervised fine-tuning (SFT)
experiments used in \Cref{sec:sft}. We report two SFT settings. The
first is a small-scale \texttt{Llama-3.2-3B} \texttt{GSM8K} sweep used to characterize
learning-rate sensitivity, training-time sample efficiency, and the relationship
between downstream accuracy and token-level training loss. The second is
our primary large-scale math SFT experiment, which fine-tunes
\texttt{Llama-3.1-8B} with LoRA on the \texttt{OpenMathInstruct-2}
\texttt{train\_1M} split and evaluates on \texttt{GSM8K}, \texttt{GSM8K-CoT}, and \texttt{MATH-500}. 

\subsubsection{3B SFT}

\paragraph{Overview.}
This experiment uses
the \texttt{Llama-3.2-3B} base model
\cite{llama3modelcard, grattafiori2024llama} with LoRA \cite{hu2022lora} on a
100K-sample subset of OpenMathInstruct-2
\cite{toshniwal2024openmathinstruct} for one epoch.

\paragraph{Additional results}
We report the \Muon SFT sweep results in \Cref{fig:sft-3b-muon}. Consistent with the \DoPe-\AdamW comparison, \DoPe-\Muon also achieves stronger downstream \texttt{GSM8K} performance across the learning rate sweep, while often demonstrating higher final training loss than the base \Muon optimizer. This further supports our main observation that training loss alone can be a poor proxy for downstream performance, and that \DoPe improves the optimization in a way that is better aligned with downstream generalization.

\label{appdx:sft-3b-addres}
\vspace{-0.25cm}
\begin{figure}[t]
  \centering
  \setlength{\tabcolsep}{1pt}
  \renewcommand{\arraystretch}{0.05}

  \begin{tabular}{@{}cc@{}}
    \includegraphics[width=0.495\linewidth]{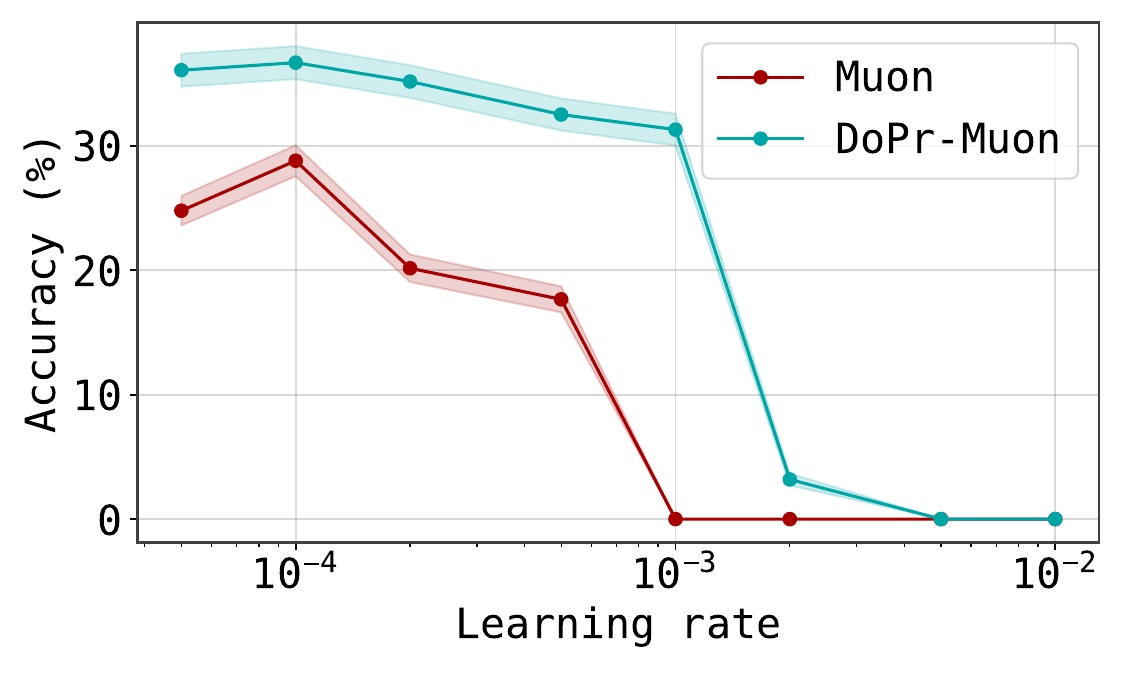}
    &
    \includegraphics[width=0.495\linewidth]{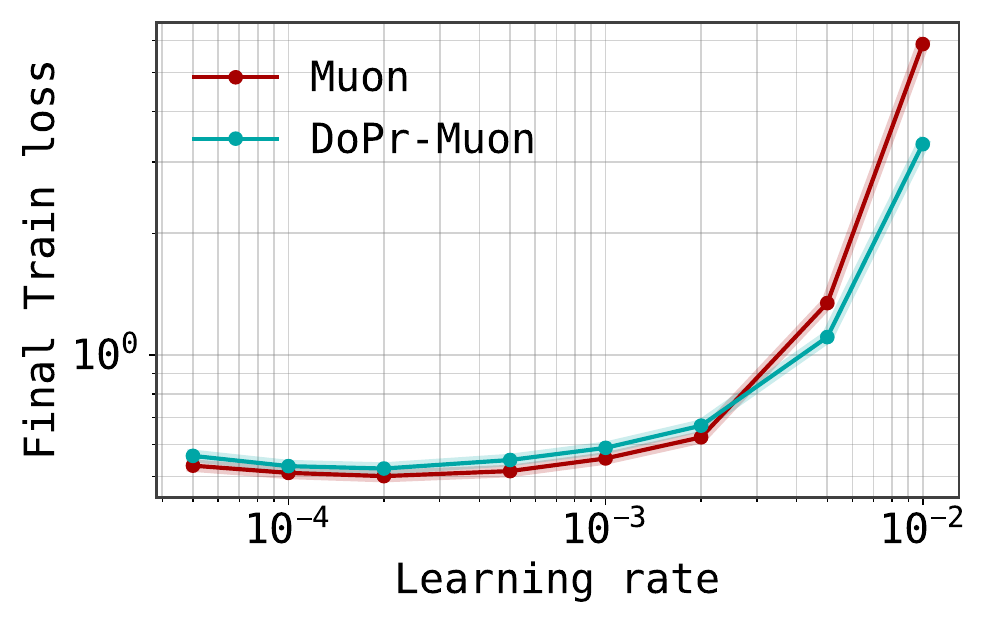}
  \end{tabular}

  \vspace{-0.15cm}
  \caption{
  \texttt{GSM8K} \textbf{3B SFT sweep}.
  Peak \texttt{GSM8K} accuracy across training steps vs.\ learning rate, comparing between \Muon and \DoPe-\Muon.
  }
  \label{fig:sft-3b-muon}
  \vspace{-0.25cm}
\end{figure}

\begin{table}[t]
  \centering
  \caption{
  Configuration for the small-scale \texttt{GSM8K} SFT sweep.
  }
  \small
  \setlength{\tabcolsep}{6pt}
  \begin{tabular}{p{0.28\linewidth} p{0.66\linewidth}}
    \toprule
    \textbf{Component} & \textbf{Setting} \\
    \midrule
    Base model & \texttt{Llama-3.2-3B} \cite{llama3modelcard, grattafiori2024llama} \\
    Tokenizer & \texttt{Llama-3.2-3B-Instruct} tokenizer \\
    Precision & bfloat16 weights and activations \\
    Attention implementation & FlashAttention-2 \cite{dao2022flashattention, dao2023flashattention2} \\
    \midrule
    Training dataset & OpenMathInstruct-2 \cite{toshniwal2024openmathinstruct} \\
    Data subset & 100K examples sampled from the 1M split \\
    Context length & 1024 tokens \\
    Packing & Enabled \\
    \midrule
    PEFT method & LoRA \cite{hu2022lora} \\
    LoRA target modules & All linear layers \\
    LoRA rank / scale & \(r=16\), \(\alpha=16\) \\
    Additional trainables & Token embedding matrix, used for chat-template special tokens \\
    Weight tying & Preserved between input embeddings and output projections \\
    \midrule
    Optimizers compared & \AdamW and \Muon, each with and without the \DoPe wrapper \\
    Training length & 1 epoch \\
    Base learning rate & \(1\times10^{-5}\) \\
    LR schedule & Cosine decay with 5\% warmup \\
    Per-device batch size & 2 \\
    Gradient accumulation & 8 steps \\
    \midrule
    Eval framework & \texttt{lm-eval} with \texttt{vLLM} backend \\
    Eval precision & bfloat16 \\
    Benchmark & \texttt{GSM8K}, 5-shot \\
    Prompting format & Chat template enabled; few-shot examples are formatted as multi-turn conversations \\
    \bottomrule
  \end{tabular}
  \label{tab:sft_small_setup}
\end{table}

\paragraph{Model and architecture.}
We fine-tune the \texttt{Llama-3.2-3B} base model
\cite{llama3modelcard, grattafiori2024llama} using the
\texttt{Llama-3.2-3B-Instruct} tokenizer. All experiments use bfloat16
precision for weights and activations, and implement attention with
FlashAttention-2.

\paragraph{Dataset and preprocessing.}
Training data are drawn from OpenMathInstruct-2
\cite{toshniwal2024openmathinstruct}, using a fixed 100K-example subset. We
tokenize sequences to a maximum context length of 1024 tokens and enable
sequence packing for efficiency.

\paragraph{Parameter-efficient fine-tuning.}
We apply LoRA \cite{hu2022lora} to all linear layers with rank \(r=16\) and
scale \(\alpha=16\). We additionally train the token embedding matrix to
accommodate chat-template special tokens, while preserving weight tying between
the input embeddings and output projection.

\paragraph{Optimization and training.}
We train for one epoch using either \AdamW or \Muon, each with and without the
\DoPe wrapper. The base learning rate is \(1\times10^{-5}\), with cosine
scheduling and 5\% warmup. We use per-device batch size 2 and gradient
accumulation over 8 steps.

\paragraph{Evaluation protocol.}
Downstream evaluation uses \texttt{lm-eval} with the \texttt{vLLM} backend in
bfloat16 precision. We evaluate on \texttt{GSM8K} with 5-shot prompting, enabling
chat-style prompting via the model's native template and formatting few-shot
exemplars as multi-turn conversations.

\subsubsection{8B SFT}
\label{appdx:sft-8b}
\paragraph{Overview.}
Our primary SFT experiment uses 
\texttt{meta-llama/Llama-3.1-8B} checkpoint
\cite{llama3modelcard, grattafiori2024llama} as the frozen base model. We train
only LoRA adapters \cite{hu2022lora}; the base model weights remain frozen
throughout training. The tokenizer and chat template are taken from
\texttt{meta-llama/Llama-3.1-8B-Instruct}. Training data are drawn from
\texttt{nvidia/OpenMathInstruct-2} \cite{toshniwal2024openmathinstruct}, split
\texttt{train\_1M}. The primary comparison in this setting is between \AdamW
and \DoPe-\AdamW, where \DoPe-\AdamW applies \AP to the
trainable LoRA modules before the AdamW update.

\definecolor{PrecondGain}{HTML}{006D2C}
\definecolor{PrecondLoss}{HTML}{8C2D04}
\definecolor{LightGain}{HTML}{E5F5E0}

\newcommand{\score}[2]{$#1{\scriptstyle\pm #2}$}
\newcommand{\bscore}[2]{$\mathbf{#1}{\scriptstyle\pm #2}$}
\newcommand{\posdelta}[1]{\textbf{\textcolor{PrecondGain}{$+#1$}}}
\newcommand{\negdelta}[1]{\textcolor{PrecondLoss}{$#1$}}
\newcommand{\zerodelta}{\textcolor{black}{$0.0$}}

\begin{table}[t]
\centering
\caption{
\textbf{8B SFT sweep.}
Final accuracy is reported as percent \(\pm\) one standard error for
\AdamW and \DoPe-\AdamW on \texttt{GSM8K}, \texttt{GSM8K-CoT}, and \texttt{MATH-500}.
\(\Delta\) denotes the improvement of \DoPe-\AdamW over \AdamW in percentage
points.
}
\label{tab:sft-serious}
\setlength{\tabcolsep}{3.2pt}
\renewcommand{\arraystretch}{1.08}
\scriptsize
\resizebox{\textwidth}{!}{%
\begin{tabular}{lccc ccc ccc}
\toprule
Learning rate
& \multicolumn{3}{c}{GSM8K}
& \multicolumn{3}{c}{GSM8K-CoT}
& \multicolumn{3}{c}{MATH-500} \\
\cmidrule(lr){2-4}\cmidrule(lr){5-7}\cmidrule(lr){8-10}
& \AdamW & \DoPe-\AdamW & \(\Delta\)
& \AdamW & \DoPe-\AdamW & \(\Delta\)
& \AdamW & \DoPe-\AdamW & \(\Delta\) \\
\midrule
\(2{\times}10^{-5}\)
& \score{71.4}{1.2} & \bscore{71.9}{1.2} & \posdelta{0.5}
& \bscore{72.2}{1.2} & \score{71.9}{1.2} & \negdelta{-0.2}
& \score{33.8}{2.1} & \score{33.8}{2.1} & \zerodelta \\
\(5{\times}10^{-5}\)
& \score{68.1}{1.3} & \bscore{70.1}{1.3} & \posdelta{2.0}
& \score{69.4}{1.3} & \bscore{71.6}{1.2} & \posdelta{2.1}
& \score{34.4}{2.1} & \bscore{35.0}{2.1} & \posdelta{0.6} \\
\(7{\times}10^{-5}\)
& \bscore{69.7}{1.3} & \score{69.5}{1.3} & \negdelta{-0.2}
& \score{69.9}{1.3} & \bscore{70.1}{1.3} & \posdelta{0.2}
& \score{29.6}{2.0} & \bscore{34.4}{2.1} & \posdelta{4.8} \\
\(1{\times}10^{-4}\)
& \bscore{68.1}{1.3} & \score{66.7}{1.3} & \negdelta{-1.4}
& \bscore{72.5}{1.2} & \score{68.5}{1.3} & \negdelta{-3.9}
& \bscore{33.0}{2.1} & \score{32.2}{2.1} & \negdelta{-0.8} \\
\(2{\times}10^{-4}\)
& \score{78.0}{1.1} & \bscore{80.1}{1.1} & \posdelta{2.1}
& \bscore{81.0}{1.1} & \score{75.7}{1.2} & \negdelta{-5.2}
& \bscore{41.6}{2.2} & \score{38.6}{2.2} & \negdelta{-3.0} \\
\rowcolor{LightGain}
\(5{\times}10^{-4}\)
& \score{66.2}{1.3} & \bscore{80.5}{1.1} & \posdelta{14.3}
& \score{51.6}{1.4} & \bscore{81.9}{1.1} & \posdelta{30.3}
& \score{46.4}{2.2} & \bscore{47.4}{2.2} & \posdelta{1.0} \\
\rowcolor{LightGain}
\(7{\times}10^{-4}\)
& \score{56.0}{1.4} & \bscore{70.4}{1.3} & \posdelta{14.4}
& \score{14.5}{1.0} & \bscore{78.4}{1.1} & \posdelta{63.9}
& \score{42.6}{2.2} & \bscore{45.8}{2.2} & \posdelta{3.2} \\
\(1{\times}10^{-3}\)
& \bscore{2.9}{0.5} & \score{2.2}{0.4} & \negdelta{-0.7}
& \bscore{2.8}{0.5} & \score{2.0}{0.4} & \negdelta{-0.8}
& \score{3.6}{0.8} & \bscore{4.6}{0.9} & \posdelta{1.0} \\
\midrule
Mean \(\Delta\) over LR grid
& \multicolumn{2}{c}{} & \posdelta{3.9}
& \multicolumn{2}{c}{} & \posdelta{10.8}
& \multicolumn{2}{c}{} & \posdelta{0.9} \\
\bottomrule
\end{tabular}%
}
\end{table}

\paragraph{Additional results.}
Results are reported in \Cref{tab:sft-serious}. Across the full learning-rate grid, \DoPe-\AdamW improves the average final accuracy over \AdamW by \(+3.9\), \(+10.8\), and \(+0.9\) percentage points on \texttt{GSM8K}, \texttt{GSM8K-CoT}, and \texttt{MATH-500}, respectively. The gains are especially pronounced in the high-learning-rate regime: at \(5\times10^{-4}\), \DoPe-\AdamW improves \texttt{GSM8K-CoT} from \(51.6\%\) to \(81.9\%\), and at \(7\times10^{-4}\), from \(14.5\%\) to \(78.4\%\). Thus, beyond improving the best observed endpoint performance, \DoPe substantially broadens the stable and effective learning-rate range for large-scale SFT. Taking the best final checkpoint over the sweep, \DoPe-\AdamW reaches \(80.5\%\) on \texttt{GSM8K}, \(81.9\%\) on \texttt{GSM8K-CoT}, and \(47.4\%\) on \texttt{MATH-500}, improving over the best \AdamW result on all three benchmarks.

\begin{table}[t]
  \centering
  \caption{
  Core configuration for the large-scale 8B SFT experiment.
  }
  \small
  \setlength{\tabcolsep}{6pt}
  \begin{tabular}{p{0.28\linewidth} p{0.66\linewidth}}
    \toprule
    \textbf{Component} & \textbf{Setting} \\
    \midrule
    Base model & \texttt{meta-llama/Llama-3.1-8B} \\
    Tokenizer / chat template & \texttt{meta-llama/Llama-3.1-8B-Instruct} \\
    Precision & bfloat16 model and compute; float32 preconditioner statistics \\
    Attention implementation & FlashAttention-2 \cite{dao2022flashattention, dao2023flashattention2} \\
    Trainer & \texttt{trl.SFTTrainer} with PEFT LoRA \\
    \midrule
    Training dataset & \texttt{nvidia/OpenMathInstruct-2}, split \texttt{train\_1M} \\
    Input fields & \texttt{problem} and \texttt{generated\_solution} \\
    Objective & Completion-only next-token loss; user-prompt tokens are masked \\
    Maximum sequence length & 4096 tokens \\
    Packing & Enabled \\
    Training length & 2 epochs \\
    Processed tokens & \(997{,}968{,}955\) tokens \\
    Optimizer steps & 482  \\
    \midrule
    Effective global batch size & 512 packed sequences \\
    World size & 8 GPUs \\
    Per-device micro-batch size & 2 \\
    Gradient accumulation & 32 steps \\
    DeepSpeed configuration & ZeRO-3 \\
    Gradient checkpointing & Enabled, reentrant checkpointing path \\
    \bottomrule
  \end{tabular}
  \label{tab:sft_8b_setup}
\end{table}

\paragraph{Training data and prompt construction.}
Each training example consists of a math problem and a generated solution from
OpenMathInstruct-2. The \texttt{problem} field is inserted into the fixed user
prompt:
\begin{quote}
{\small\ttfamily
Solve the following math problem. Make sure to put the answer (and only answer)
inside \textbackslash boxed\{\}.
}
\end{quote}
The target completion is the corresponding \texttt{generated\_solution}. 

\paragraph{Parameter-efficient fine-tuning.}
All large-scale SFT runs use LoRA \cite{hu2022lora}; full fine-tuning is not
used. LoRA adapters are applied to the standard Llama attention and MLP
projection matrices:
\[
\{
\texttt{q\_proj}, \texttt{k\_proj}, \texttt{v\_proj}, \texttt{o\_proj},
\texttt{gate\_proj}, \texttt{up\_proj}, \texttt{down\_proj}
\}.
\]
The LoRA rank is \(r=64\), the LoRA scale is \(\alpha=128\), LoRA dropout is
0.05, and no bias terms are trained.

\begin{table}[t]
  \centering
  \caption{
  LoRA, optimizer, learning-rate sweep, and preconditioner configuration for the
  large-scale 8B SFT experiment.
  }
  \small
  \setlength{\tabcolsep}{6pt}
  \begin{tabular}{p{0.30\linewidth} p{0.64\linewidth}}
    \toprule
    \textbf{Component} & \textbf{Setting} \\
    \midrule
    PEFT method & LoRA \cite{hu2022lora} \\
    LoRA target modules & \texttt{q,k,v,o,gate,up,down} projection matrices \\
    LoRA rank / scale & \(r=64\), \(\alpha=128\) \\
    LoRA dropout & 0.05 \\
    LoRA bias & \texttt{none} \\
    Trainable adapter parameters & \(167{,}772{,}160\) \\
    \midrule
    AdamW betas and epsilon & \((0.9, 0.98)\), \(\epsilon=10^{-8}\) \\
    Weight decay & 0.01 \\
    LR schedule & Constant, no warmup \\
    \bottomrule
  \end{tabular}
  \label{tab:sft_8b_lora_optim}
\end{table}

\paragraph{Evaluation protocol.}
Evaluation is performed with \texttt{lm-evaluation-harness} using the
\texttt{vLLM} backend. Fine-tuned models are evaluated by loading the frozen
\texttt{Llama-3.1-8B} base weights together with the LoRA adapter and the saved
adapter tokenizer when present. The evaluation wrapper applies the chat template
for fine-tuned adapters. Inference uses bfloat16 precision, tensor parallelism
over 8 GPUs, greedy decoding, temperature \(0.0\), and top-\(p=1.0\). For the
learning-rate sweep evaluation, we use \texttt{max\_model\_len=8193},
\texttt{max\_gen\_toks=2048}, \texttt{gpu\_memory\_utilization=0.9},
\texttt{batch\_size=auto}, and \texttt{max\_batch\_size=64}. The one-token
context margin is introduced to avoid vLLM scoring failures at exactly the model
length.

We report the following evaluation tasks:
\begin{itemize}[
    noitemsep,
    topsep=-3pt,
    parsep=0pt,
    partopsep=0pt,
    left=0pt]
    \item \texttt{gsm8k}: 5-shot direct \texttt{GSM8K} generation, exact match with the
    flexible extraction filter.
    \item \texttt{gsm8k\_cot}: 8-shot chain-of-thought \texttt{GSM8K} generation, exact
    match with the flexible extraction filter.
    \item \texttt{minerva\_math500}: 4-shot Minerva-style MATH-500 generation,
    scored with \texttt{math\_verify}.
\end{itemize}
Reported uncertainty values are the standard errors emitted by
\texttt{lm-evaluation-harness}.

\subsection{Generative Modeling}
\label{app:gen_modeling}

\begin{figure}[t]
  \centering
  \begin{subfigure}[t]{0.48\linewidth}
    \centering
    \includegraphics[width=\linewidth]{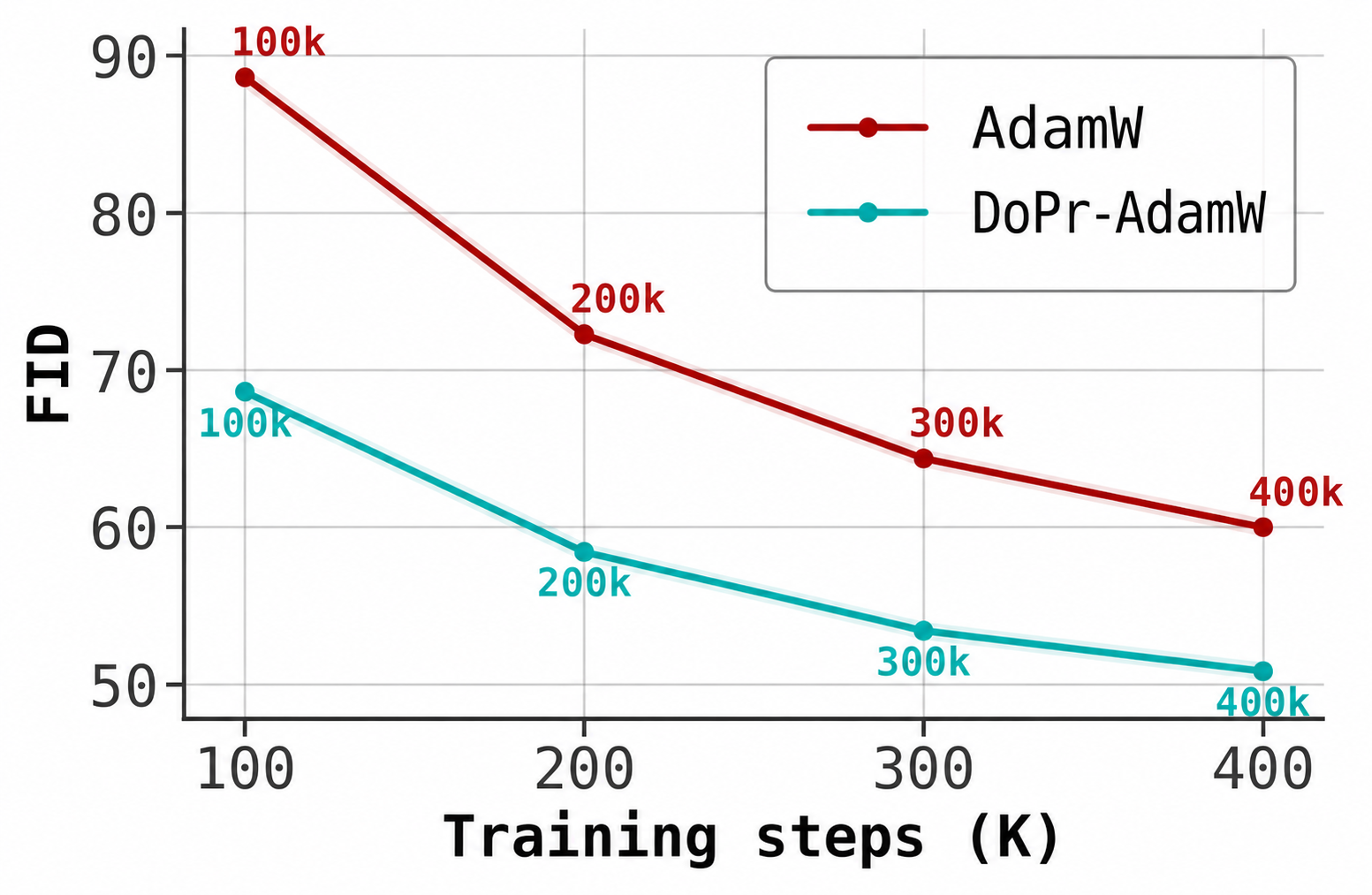}
    \label{fig:sit_imagenet256_fid_adamw}
  \end{subfigure}
  \hfill
  \begin{subfigure}[t]{0.48\linewidth}
    \centering
    \includegraphics[width=\linewidth]{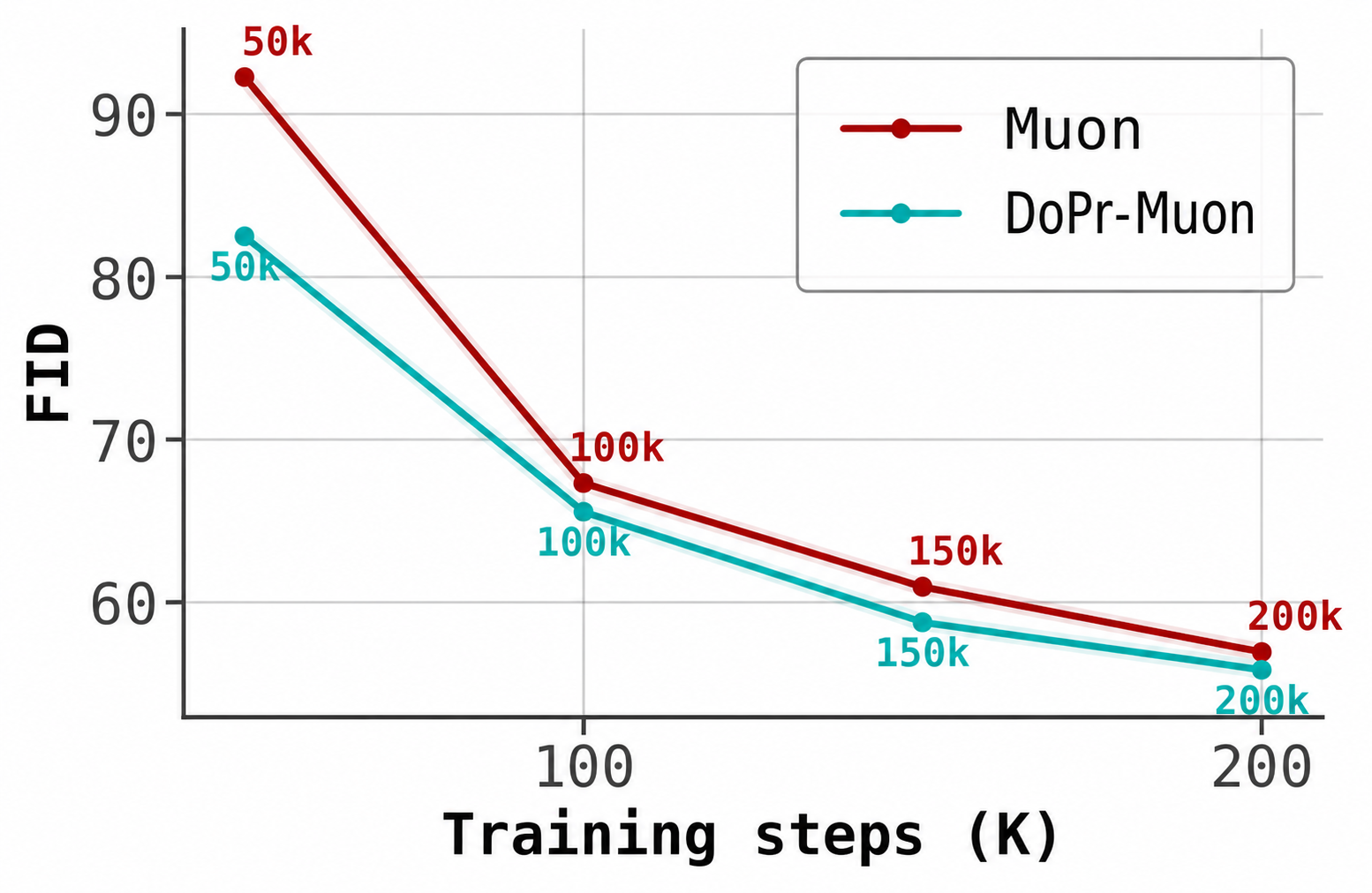}
    \label{fig:sit_imagenet256_fid_muon}
  \end{subfigure}

  \centering
  \begin{subfigure}[t]{0.48\linewidth}
    \centering
    \includegraphics[width=\linewidth]{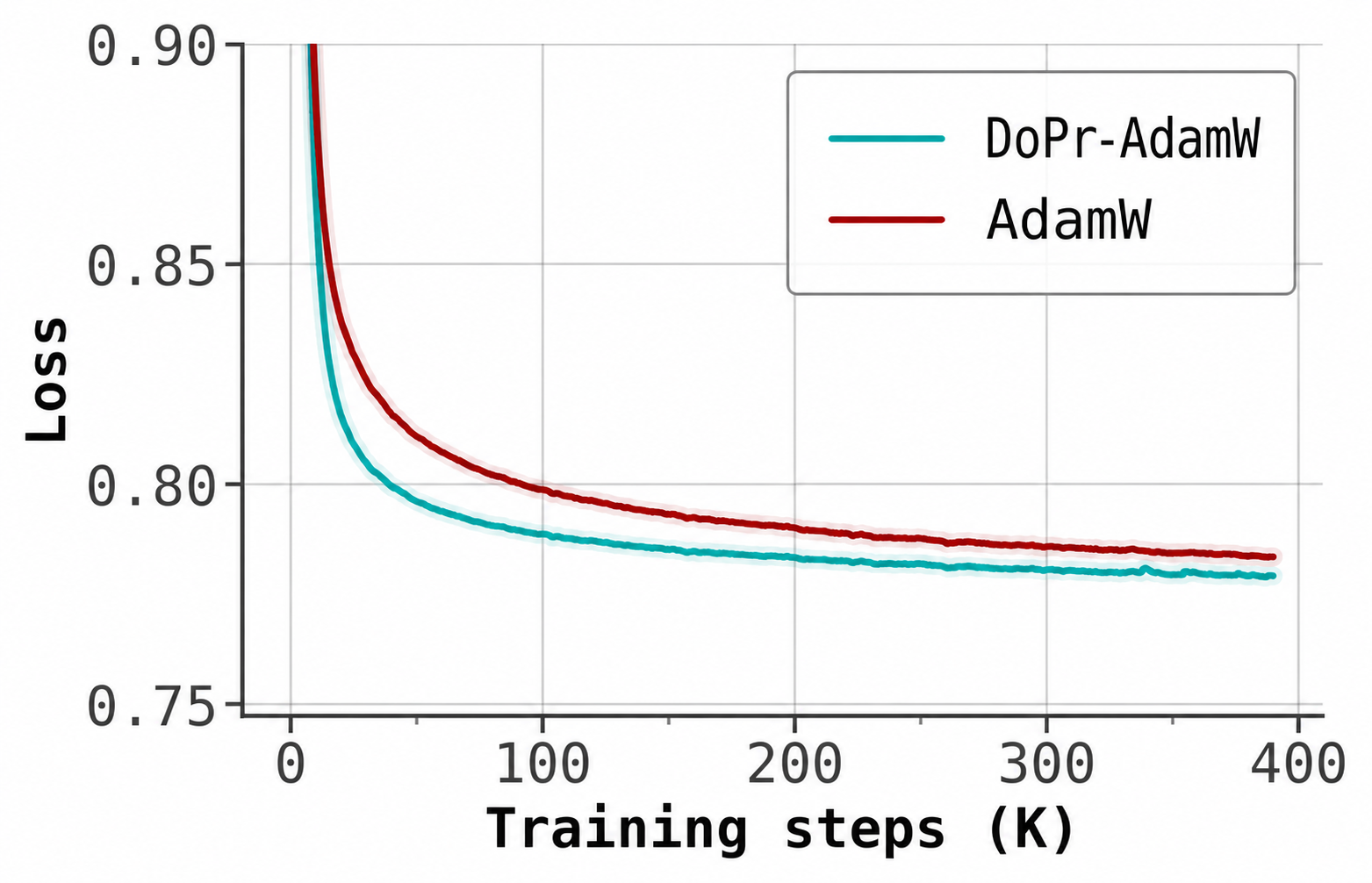}
    \label{fig:sit_imagenet256_loss_adamw}
  \end{subfigure}
  \hfill
  \begin{subfigure}[t]{0.48\linewidth}
    \centering
    \includegraphics[width=\linewidth]{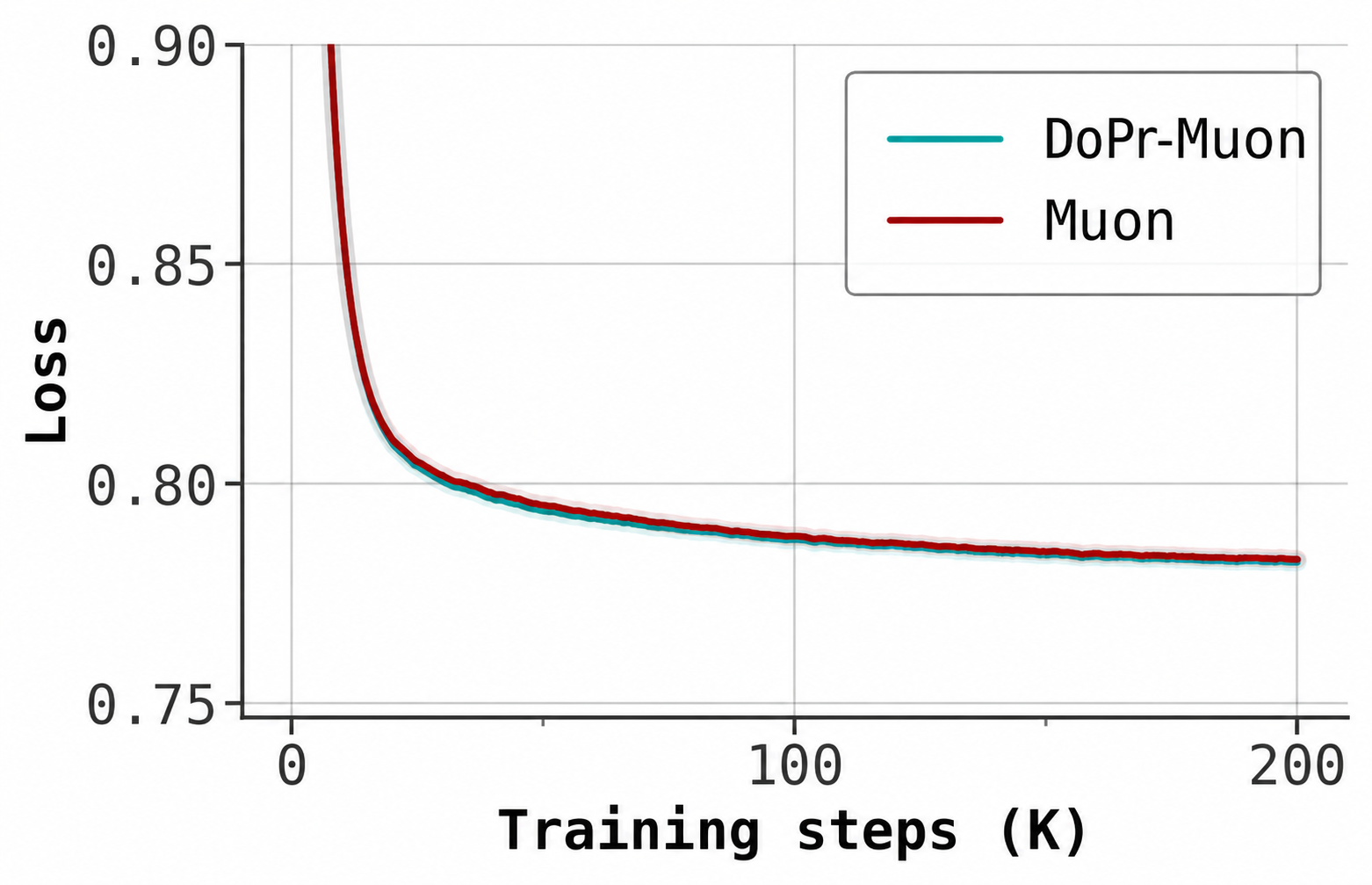}
    \label{fig:sit_imagenet256_loss_muon}
  \end{subfigure}

  \caption{\textbf{SiT-S on ImageNet256}. \textbf{Top row}: perceptual quality (FID-50K) vs.\ training steps. \textbf{Bottom row:} training loss trajectories. We see that \DoPe improves sample quality as well as loss convergence.
  }
  \label{fig:sit_imagenet256}
  \vspace{-0.3cm}
\end{figure}

\paragraph{Generative modeling as TTF.}
Generative modeling is a natural TTF setting: training is local teacher-forced velocity regression, whereas deployment rolls out the learned velocity field under its own induced trajectory. As in flow matching, the model minimizes errors on $(\xblue[t],\yblue[t])$, where $\yblue[t]=\bv(\xblue[t],t)$ is the closed-form target velocity. At generation time, however, samples are produced by iteratively integrating the learned field, e.g.\ $\frac{\rmd}{\rmd t}\zred=\vred(\xred,t)$, so velocity errors can accumulate and shift future states away from the training distribution. Consequently, per-step training loss can misalign with final sample quality metrics such as FID.

We evaluate \DoPe on generative modeling using the SiT-S model \citep{ma2024sit} trained at $256 \times 256$ resolution on ImageNet \citep{deng2009imagenet}.  As in \citet{ma2024sit}, the baseline model is trained with \AdamW, and generates using a number of function evaluations (NFE) fixed to 250. We follow our general operating guidelines and swap the optimizer to \DoPe-\AdamW, maintaining the same learning rate and all other architectural, training, and evaluation protocols as the baseline. We then repeat this using baseline \Muon and \DoPe-\Muon. Table~\ref{tab:sit_setup} summarizes the model, training, and evaluation configuration.

We plot the FID-50K scores against training steps as the ``downstream metric''; refer to \Cref{fig:sit_imagenet256}. Across both \AdamW and \Muon, \DoPe consistently accelerates convergence and improves final FID. In addition, qualitative inspection of intermediate generations in Figure~\ref{fig:sit_imagenet256_samples_20k} indicates that \DoPe induces a distinct optimization trajectory compared to the base optimizers.

Notably, as seen in \Cref{fig:sit_imagenet256}, the relative improvements in FID are somewhat proportional to \emph{improvements} in train loss convergence. We hypothesize that generative modeling is a relatively unique TTF setting, in that one has control over the data-generating dynamical system. For example, many works study ``straightening'' the transport trajectory (cf.\ standard flow matching versus Rectified Flow \citep{liu2023flow} and Optimal-Transport CFM \citep{tong2024improving}), which can enable low NFE sampling. Therefore, the severity of TTF is tied to the generative algorithm. 
We leave exploring this space for co-design to future work.

\begin{figure}[t]
  \centering
  \begin{tabular}{c c}
    \textbf{\AdamW (20K steps)} & \textbf{\DoPe-\AdamW (20K steps)} \\
    \includegraphics[width=0.45\linewidth]{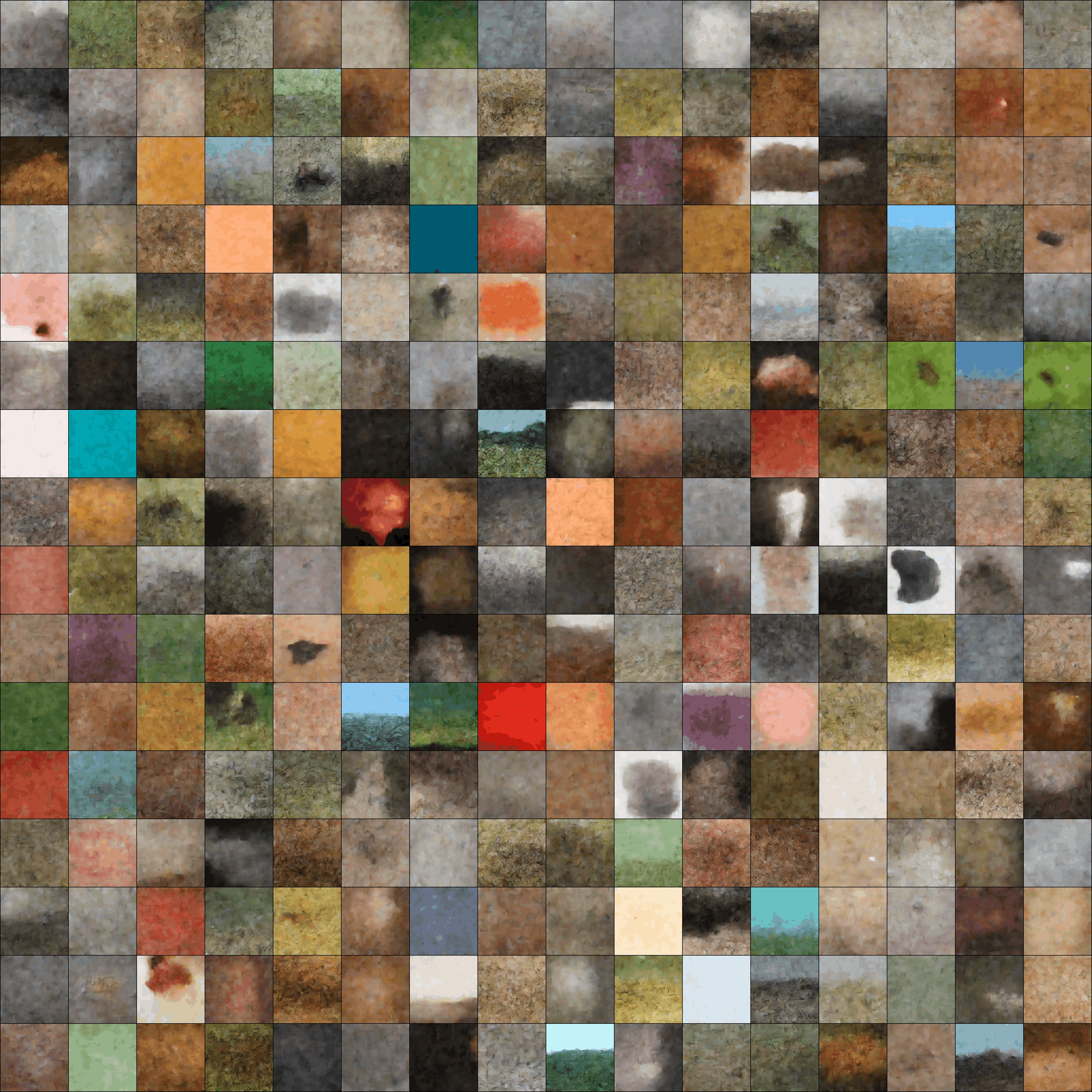} &
    \includegraphics[width=0.45\linewidth]{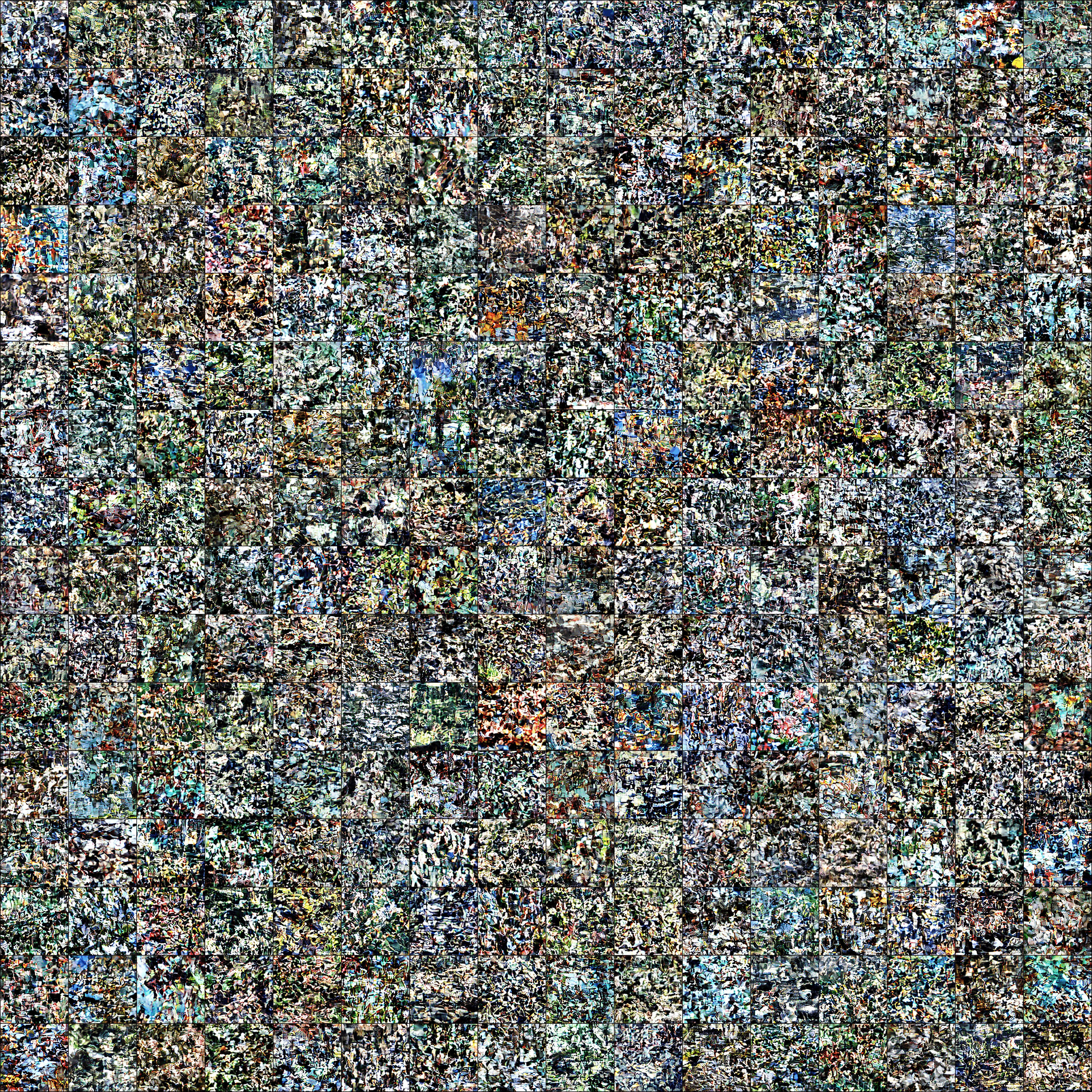}
  \end{tabular}

  \caption{
  \textbf{Generated ImageNet256 samples at 20K training steps (SiT-S, \AdamW).}
  Although both models are trained under identical settings and evaluated using the same ODE sampler,
  the samples generated by \DoPe-\AdamW exhibit qualitatively different structure and composition compared to \AdamW at the same early training stage.
  }
  \label{fig:sit_imagenet256_samples_20k}
  \vspace{-0.5cm}
\end{figure}

\begin{table}[t]
  \centering
  \caption{Experimental setup for SiT-S generative modeling on ImageNet at $256\times256$. Across all runs, we vary only the optimizer choice and whether \DoPe is enabled; all other training and evaluation settings match the standard SiT protocol.}
  \small
  \setlength{\tabcolsep}{6pt}
  \begin{tabular}{p{0.26\linewidth} p{0.68\linewidth}}
    \toprule
    \textbf{Component} & \textbf{Setting} \\
    \midrule
    Model & SiT-S \cite{ma2024sit} \\
    Dataset & ImageNet \cite{deng2009imagenet} \\
    Resolution & $256 \times 256$ \\
    Architecture / parameterization / conditioning &
    Exactly follows the original SiT implementation \cite{ma2024sit} \\
    \midrule
    Optimizers compared &
    \AdamW; \Muon; each \emph{with} and \emph{without} the \DoPe wrapper \\
    Training steps &
    \AdamW: 400K steps; \Muon: 200K steps \\
    Learning rate &
    \AdamW: matches SiT defaults \cite{ma2024sit}; \Muon: $10^{-3}$ (per prior recommendations) \\
    Other training settings &
    Held fixed across runs (data preprocessing, batch size, LR schedule, regularization) \\
    \midrule
    Sampler & ODE-based sampler from SiT \cite{ma2024sit} \\
    NFE & 250 \\
    Metric & FID-50K (unless otherwise specified) \\
    FID reference set & ImageNet training set \cite{deng2009imagenet} \\
    \bottomrule
  \end{tabular}
  \label{tab:sit_setup}
\end{table}

\section{Additional Experiments and Setup}\label{appdx:additional exps}

\subsection{Moment Grafting Experiments}

\paragraph{Datasets.} 
The ``original'' dataset used is CIFAR10 \cite{cifar10dataset}. We split this into a 0.8:0.2 split of train/validation data. Then, we generate a ``transformed'' dataset by applying a random spatial rotation of CIFAR by $90^\circ$, and after normalizing all coordinates to be between $[-1, 1]$, we apply $\tanh$ elementwise. We note this induces a nonlinear transform. 
\paragraph{Training.}
Next, we train $3$ models, each a 3 layer MLP with layer dimensions $(3072, 256)$, $(256, 256)$, and $(256, 10)$, and SiLU activations. All models are initialized identically. One model is trained on the original CIFAR dataset \cite{cifar10dataset}, called ``Original." Another model is trained on the transformed CIFAR data, denoted ``Transformed." We also coordinate training so that at each step, the batches the Transformed model sees corresponds exactly to the transformed version of the batches the Original model sees. The last model is called ``Grafted." All models are trained simultaneously.

The Grafted model is trained as follows: during training, the activations inputted into each layer of the Transformed model are recorded, and their means/covariances calculated. Then, the Grafted model takes as input the Original data batch. Then, before passing the data through each layer's weights, an affine transform is applied to the incoming activations to match their mean and covariance to the cached means/covariances of the Transformed model's forward pass.  

Concretely, let $\boldsymbol{\mu}_{t}, \bSigma_{t}$ be the mean and covariance of transformed batch, $\bX_{t}\in \R^{3072\times b}$, where $b$ is the batch size. When the corresponding batch $\bX$ is passed into the Grafted model, let $\bX$ have mean $\boldsymbol{\mu}$ and covariance $\bSigma$. Then, before we pass $\bX$ through the layer, we apply the following transform to $\bX$: 
\[\bX_{new}=\bSigma_{t}^{1/2}\bSigma^{-1/2}(\bX-\boldsymbol\mu\mathbf{1}^\top) + \boldsymbol\mu_{t}\mathbf{1}^\top\]
We can note that as $\bX-\boldsymbol\mu \mathbf{1}^\top$ has mean $0$, then $\bX_{new}$ will have mean $\boldsymbol{\mu}_{t}$. Similarly, we can find the covariance as: 
\[(\bX_{new}-\boldsymbol\mu_{t}\mathbf{1}^\top)(\bX_{new}-\boldsymbol\mu_{t}\mathbf{1}^\top)^\top=\bSigma_{t}^{1/2}\bSigma^{-1/2}(\bX-\boldsymbol\mu\mathbf{1}^\top)(\bX-\boldsymbol\mu\mathbf{1}^\top)^\top(\bSigma^{-1/2})^\top(\bSigma_{t}^{1/2})^\top\]
Noting that covariances matrices are symmetric: 
\[=\bSigma_{t}^{1/2}\bSigma^{-1/2}\bSigma(\bSigma^{-1/2})^\top(\bSigma_{t}^{1/2})^\top=\bSigma_{t}^{1/2}\bSigma_{t}^{1/2}=\bSigma_{t}\]
Therefore, this transform does give $\bX_{new}$ a mean of $\boldsymbol\mu_{t}$, and covariance of $\bSigma_{t}$. However, since a nonlinear transform was used to create the transformed data, this affine transform \emph{will not be an exact recovery of the transformed data}, only matching the mean and covariance. We repeat this process for the hidden layer and output layer, with the respective activations for each.

In these experiments, we use a batch size of $2048$, and $\mu$P scaling. The base learning rate was $10^{-3}$ for \SGD, $10^{-5}$ for \AdamW, and $10^{-4}$ for \Muon. Otherwise, the default PyTorch parameters are used for \AdamW and \Muon. To compute square roots and inverse square roots of the covariance matrices, symmetric eigenvalue decomposition was used, with a damping constant of $10^{-6}\cdot \trace(\bA)$, where $\bA$ is the matrix to be decomposed. 

\paragraph{Evaluation.}
Finally, the models are validated on the Transformed data after each training step, and the validation losses plotted. The results are displayed in \Cref{fig:Moment graft}. 

Despite not being trained on the Transformed data, the Grafted model's validation performance on the transformed data matches much more closely with the Transformed model than the Original model (trained on the original data) across optimizers, suggesting that when trained on data with similar means and covariances, models often learn similarly, even if the data itself does not perfectly match. While not displayed, we further note that increasing batch size and decreasing learning rate both help to improve how closely the Grafted model fits the Transformed model in validation loss, as they reduce drift arising from mean/covariance estimation noise and higher-order mismatches from the \emph{non-linear} data transform. Notably, we remark that since there is no corrective mechanism to port the Grafted model's weights onto the Transformed model's ``ground truth'' weights, the drift will monotonically compound over the training horizon.

The Grafted model matching the Transformed model in validation demonstrates that even ignoring the higher-order nonlinear aspects of the forward pass, the mean and covariance of a forward pass are sufficient to \emph{graft} much of the downstream performance of a model that is truly trained on the transformed data onto another model, despite the latter only seeing the original data in its forward/backward pass. This lends credence to the premise of \AP, where the working hypothesis is that ``error directions'' (as captured by the mean and covariance of each layer's activations) strongly bias the optimizer's downstream behavior. Grafting demonstrates that, in the hypothetical case where we \emph{knew} the feature directions of the ``downstream'' model, we can graft its bias onto the training on the training distribution. However, since we typically cannot know these features ahead of time, \AP --- which essentially treats all feature directions equally --- is the ``worst-case'' optimal approach.

\begin{figure}[ht]
  \begin{center}    
  \centerline{\includegraphics[width=\columnwidth]{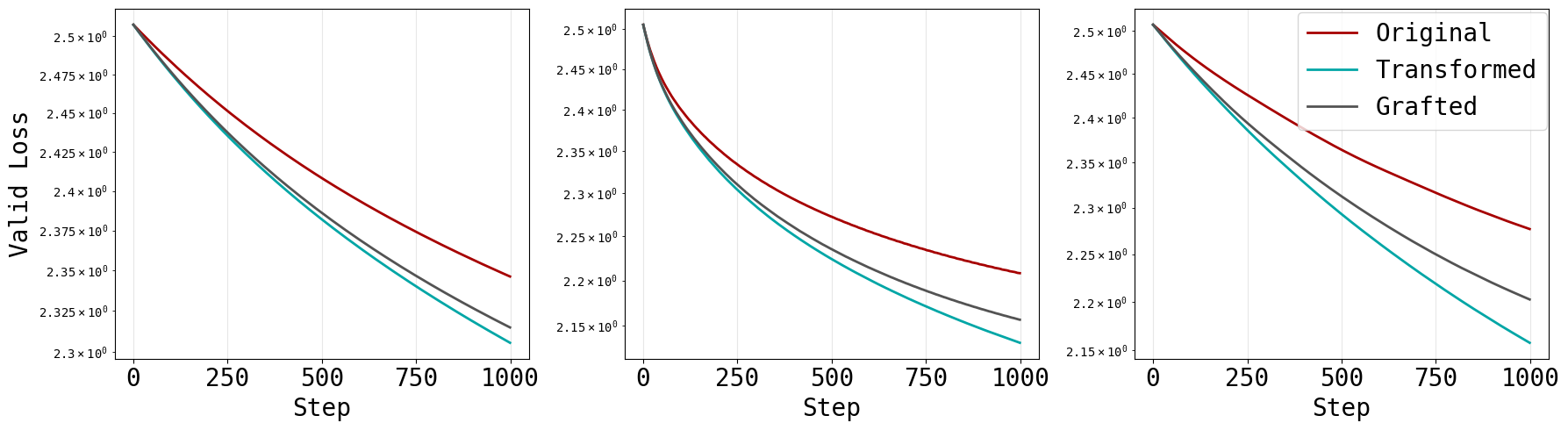}}
    \caption{
      We train 3 MLPs on CIFAR, and a nonlinear invertible transform of CIFAR. We plot the validation loss for the 3 models for each of \AdamW (left), \SGD (middle), and \Muon (right). Across all $3$, it can be seen that the Grafted model's trajectory is closer to the Transformed model than the Original model. 
    }
    \label{fig:Moment graft}
  \end{center}
  \vspace{-0.5cm}
\end{figure}

\subsection{Stable Ranks of Optimizer and Model Quantities}\label{appdx:stable ranks}

Following prior work that hypothesizes the superior training efficiency of \Muon (compared to, e.g., \SGD) due to ratios of certain model/optimizer quantities, we plot these relevant quantities using the terminal weights of the small-scale 3B SFT run in \Cref{sec:sft} and \Cref{app:sft_details}. In particular, we follow the predictions in \citet{davis2025spectral} and compare the ratio of the ``nuclear rank'' of the layer-wise gradients and the stable rank of the batch activation covariances of a forward pass, see \Cref{fig:stable_rank_G}. However, we note that \DoPe-\Muon does not orthogonalize the raw gradient, rather, it takes the \AP gradient as input. Following the hypothesis in \citet{davis2025spectral}, we also compare the nuclear rank of the \AP gradient to the stable rank of the activation covariances (\Cref{fig:stable_rank_GSigmaInv}).

\begin{figure}[t]
    \centering
    \begin{subfigure}{0.48\linewidth}
        \centering
        \includegraphics[width=\linewidth]{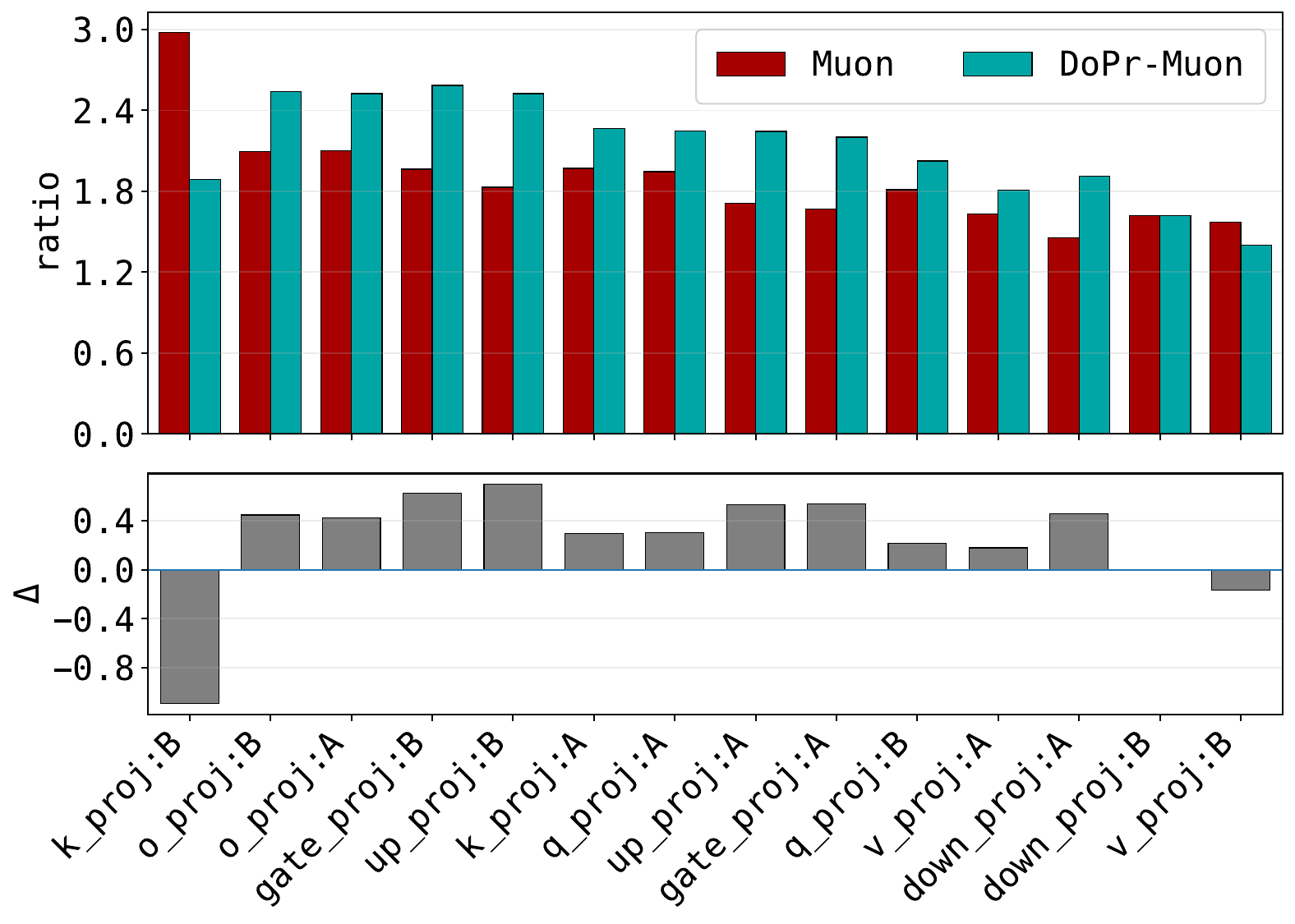}
        \caption{}
        \label{fig:stable_rank_G}
    \end{subfigure}
    \hfill
    \begin{subfigure}{0.48\linewidth}
        \centering
        \includegraphics[width=\linewidth]{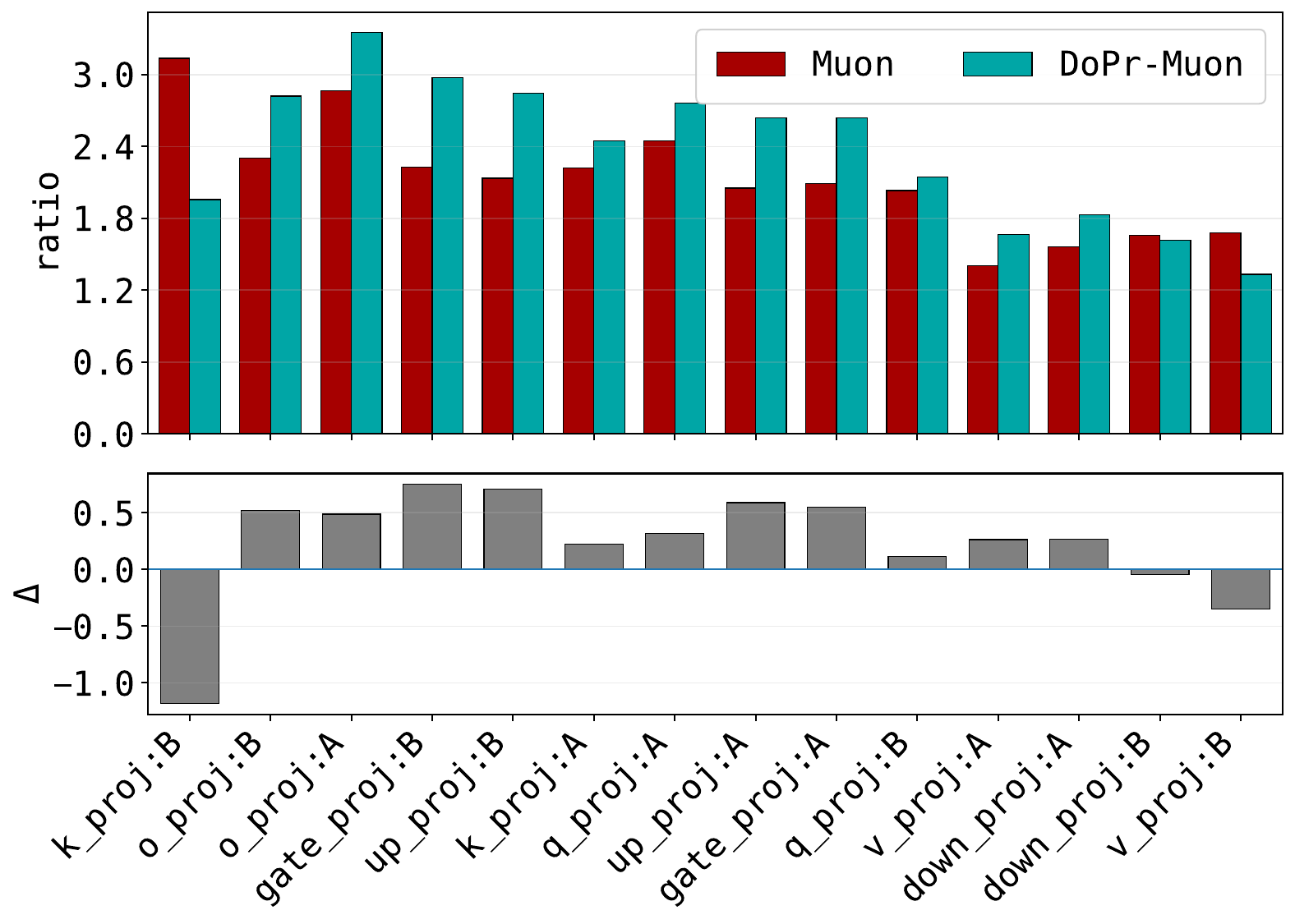}
        \caption{}
        \label{fig:stable_rank_GSigmaInv}
    \end{subfigure}

    \begin{subfigure}{0.48\linewidth}
        \centering
        \includegraphics[width=\linewidth]{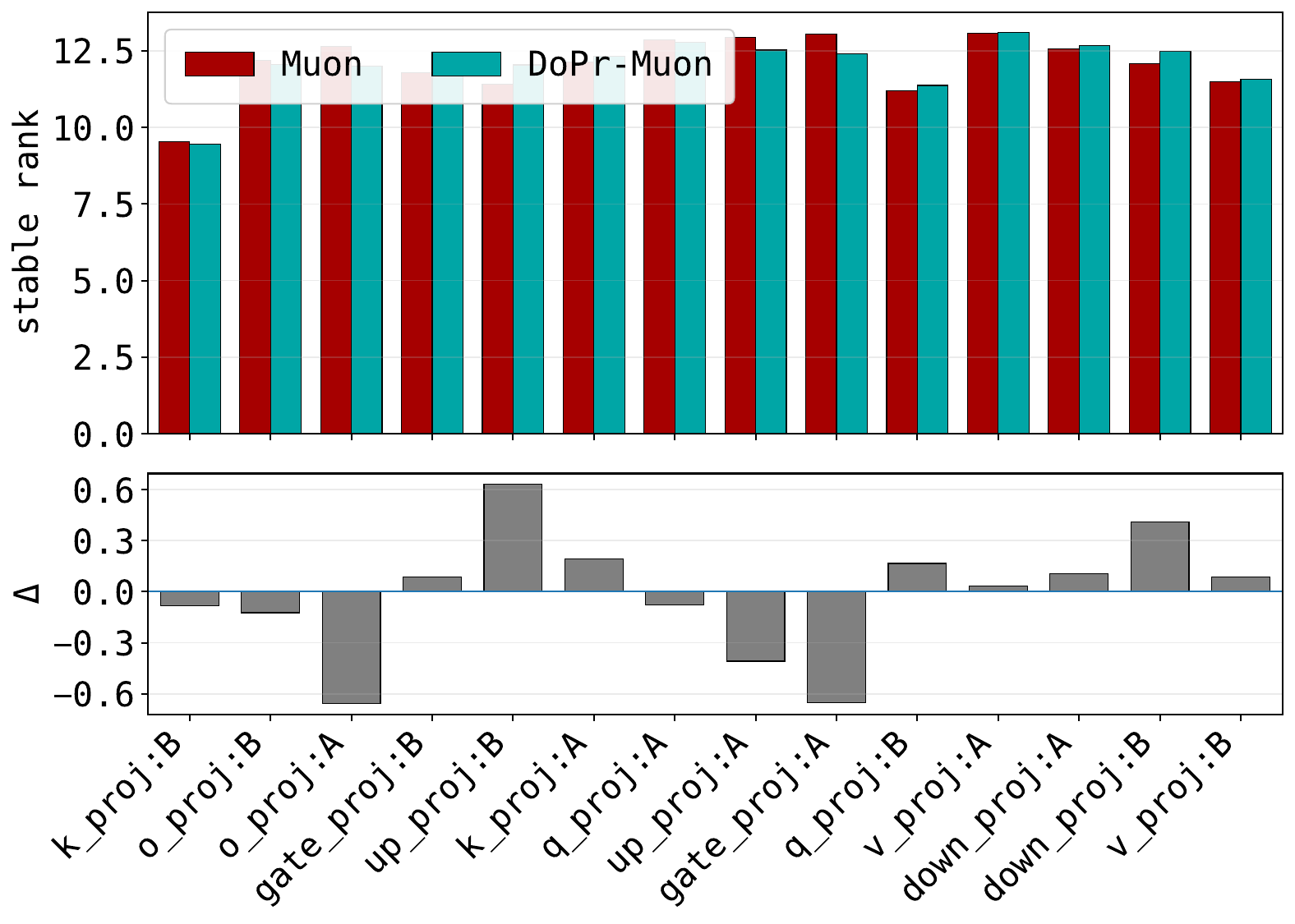}
        \caption{}
        \label{fig:stable_rank_W}
    \end{subfigure}
    \hfill
    \begin{subfigure}{0.48\linewidth}
        \centering
        \includegraphics[width=\linewidth]{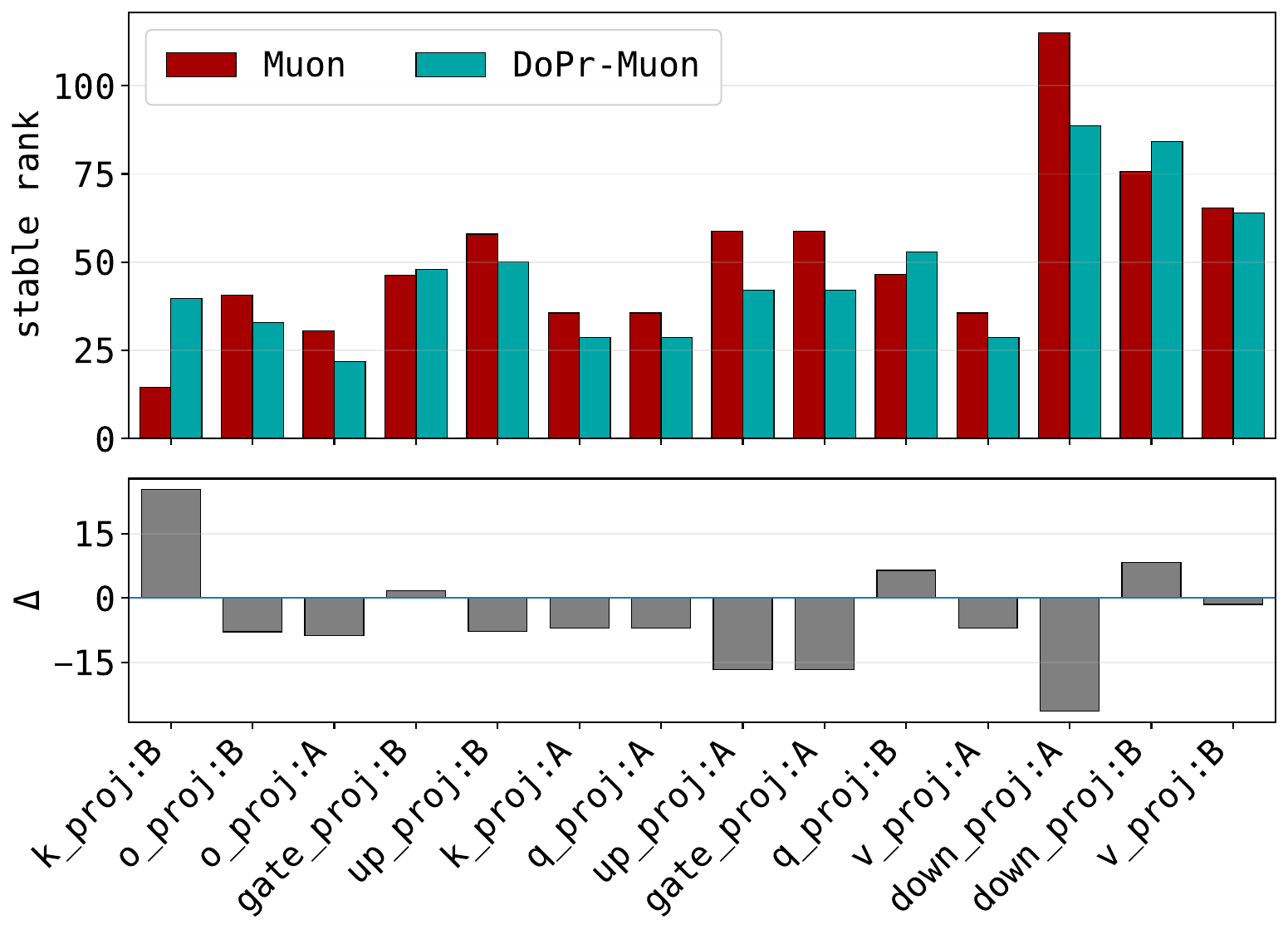}
        \caption{}
        \label{fig:stable_rank_Sigma}
    \end{subfigure}

    \caption{Metric values aggregated by layer group for \Muon and \DoPe-\Muon. The bottom panel shows the mean difference (\DoPe-\Muon minus \Muon).}
    \label{fig:stable_rank}
\end{figure}

On the other hand, we plot the standalone stable ranks of the weight matrices themselves (\Cref{fig:stable_rank_W}) and the batch activation covariances (\Cref{fig:stable_rank_Sigma}). In this case, we do not see a consistent trend across layers, which supports our working hypothesis that our intervention on the optimizer is unrelated to regularizing the optimizer dynamics in \emph{weight geometry}, but rather targets the \emph{error geometry}, captured by the correspondence between the (\AP) gradient and the activation.

\end{document}